\documentclass{article}
\usepackage[utf8]{inputenc}
\usepackage[title]{appendix}    
\usepackage{amscd}     
\usepackage{amsmath}   
\usepackage{amssymb}   
\usepackage{amsthm}    
\usepackage{dsfont}
\usepackage{hyperref}  
\usepackage{bookmark}  
\usepackage{mathtools} 
\usepackage{verbatim}  
\usepackage{graphicx}  
\usepackage[dvipsnames]{xcolor}
\usepackage{soul}
\usepackage[skip=2pt,font={small, it}]{caption} 
\usepackage{algorithm} 
\usepackage{algorithmicx}
\usepackage{booktabs}
\usepackage{multirow}
\usepackage{algpseudocode}
\usepackage[normalem]{ulem}
\usepackage{geometry}
\usepackage{bm}
\usepackage{subcaption}
\usepackage{algorithm} 
\usepackage{algorithmicx}
\usepackage{multirow}
\usepackage{enumitem}
\usepackage{algpseudocode}
\usepackage{lineno}
\usepackage[colorinlistoftodos]{todonotes}
\usepackage{tikz}
\usetikzlibrary{calc,shapes, positioning}
\usetikzlibrary{arrows.meta, quotes}

\newtheorem{definition}{Definition}

\newtheorem{theorem}{Theorem}[section]
\newtheorem{lemma}[theorem]{Lemma}

\theoremstyle{remark}
\newtheorem{remark}[theorem]{Remark}
\newcommand{\RR}{\mathbb{R}}

\newcommand{\OO}{\mathcal{O}}

\newcommand{\T}{\mathcal{T}}
\newcommand{\D}{\mathcal{D}}
\newcommand{\eps}{\varepsilon}

\newcommand{\wtilde}[1]{\widetilde{#1}}

\graphicspath{{'./figures/'}}
\DeclareMathOperator*{\argmin}{\arg\!\min}
\usepackage{authblk}

\title{Spectral Top-Down Recovery of Latent Tree Models}
\date{}
\author[1]{Yariv Aizenbud \footnote{YA and AJ contributed equally to this work}}
\author[1]{Ariel Jaffe$^*$}
\author[5]{Meng Wang}
\author[1]{Amber Hu}
\author[1]{Noah Amsel}
\author[2]{Boaz Nadler}
\author[3]{Joseph T. Chang}
\author[1,4,5]{Yuval Kluger}

\affil[1]{\small Program in Applied Mathematics, Yale University, New Haven, CT 06511}
\affil[2]{Department of Computer Science, Weizmann Institute of Science, Rehovot, 76100, Israel}
\affil[3]{Department of Statistics, Yale University, New Haven, CT 06520, USA}
\affil[4]{Interdepartmental Program in Computational Biology and Bioinformatics, Yale University, New Haven, CT 06511}
\affil[5]{Department of Pathology, Yale University New Haven, CT 06511}

\begin{document}

\maketitle
\begin{abstract}
Modeling the distribution of high dimensional data by a latent tree graphical model is a prevalent approach in multiple scientific domains. 
A common task is to infer the underlying tree structure, given only observations of its terminal nodes. 
Many algorithms for tree recovery are computationally intensive, which limits their applicability to trees of moderate size.
For large trees, a common approach, termed \textit{divide-and-conquer}, is to recover the tree structure in two steps. First, recover the structure separately of multiple, possibly random subsets of the terminal nodes. Second, merge the resulting subtrees to form a full tree. Here, we 
develop Spectral Top-Down Recovery (STDR), a deterministic divide-and-conquer approach to infer large latent tree models. 
Unlike previous methods, STDR partitions the terminal nodes in a non random way, based on the Fiedler vector of a suitable Laplacian matrix related to the observed nodes. We prove that under certain conditions, this partitioning is consistent with the tree structure. This, in turn, leads to a significantly simpler merging procedure of the small subtrees.
We prove that STDR is statistically consistent and bound the number of samples required to accurately recover the tree with high probability. Using simulated data from several common tree models in phylogenetics,
we demonstrate that STDR has a significant advantage in terms of runtime, with improved or similar accuracy.
\end{abstract}


\section{Introduction}\label{sec:introduction}
Learning the structure of latent tree graphical models is a common task in machine learning 
\cite{anandkumar2012learning, choi2011learning,harmeling2010greedy, mourad2013survey,zhang2008latent} and computational biology \cite{jones2020inference,jones2013evolutionary}.
A canonical application is phylogenetics, where the task is to infer the evolutionary tree that describes the relationship between a group of biological species based on their nucleotide or protein sequences \cite{felsenstein2004inferring,nei2000molecular,semple2003phylogenetics}.
Depending on the  application, 
the number of observed nodes ranges from a dozen and up to tens of thousands. 

In latent tree graphical models, every node is associated with a random variable. 
A key assumption is that the given data corresponds to the terminal nodes of a tree, while the set of unobserved internal nodes determines its distribution.
In phylogenetics, the terminal nodes are existing organisms, while the non-terminal nodes correspond to their extinct ancestors. 
Given a set of nucleotide or amino acid sequences as in Figure \ref{fig:model_tree}, the task is to recover the structure of the tree, which
describes how the observed organisms evolved from their ancestors.

Many algorithms have been developed for recovering latent trees. Distance-based methods, including the classic neighbor joining (NJ) \cite{saitou1987neighbor} and UPGMA \cite{sokal1958statistical}, recover the tree based on a distance measure between all pairs of terminal nodes. 
These methods are computationally efficient and thus applicable to large trees \cite{tamura2004prospects}. They also have 
statistical guarantees for accurate recovery \cite{atteson1999performance,mihaescu2009neighbor}. 
Since the distance measure does not encapsulate all the information available from the sequences, distance-based methods may perform poorly when the amount of data is limited \cite{yang2012molecular}.
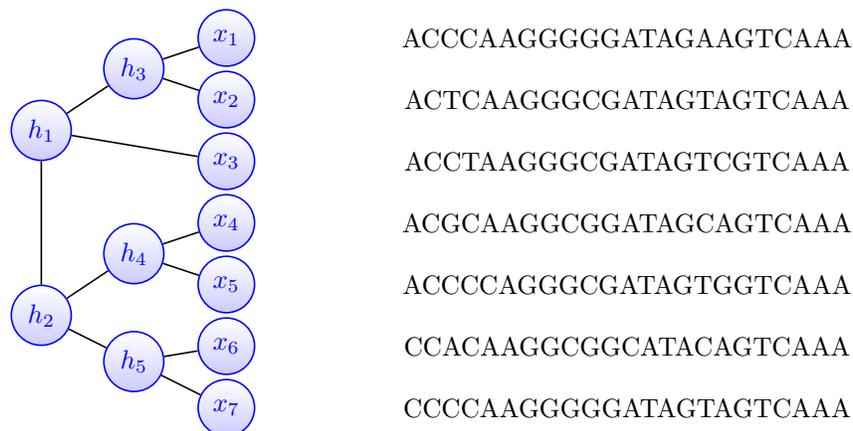
\begin{figure}[t]
    \centering
		\begin{tikzpicture}[-latex ,auto ,node distance =4 cm and 5cm ,on grid ,
		semithick ,  scale = 0.82,
		state/.style ={ circle ,top color =white , bottom color = blue!20 ,
			draw,blue , text=blue , minimum width = 0.75 cm}]	
		
		\node[state] (h2) at (1.5,4.5) {$h_1$};
		\node[state] (h3) at (1.5,1.5) {$h_2$};
		\node[state] (h4) at (3,5.5) {$h_3$};
		\node[state] (h5) at (3,2.5) {$h_4$};
		\node[state] (h6) at (3,0.75) {$h_5$};
		\node[state] (x1) at (4.5,6) {$x_1$};
		\node[state] (x2) at (4.5,5) {$x_2$};
		\node[state] (x3) at (4.5,4) {$x_3$};
		\node[state] (x4) at (4.5,3) {$x_4$};
		\node[state] (x5) at (4.5,2) {$x_5$};
		\node[state] (x6) at (4.5,1) {$x_6$};
		\node[state] (x7) at (4.5,0) {$x_7$};
		\path[-] (h2) edge node [above =0.15 cm,left = 0.15cm] {}(h3);
		\path[-] (h2) edge node [above =0.15 cm,left = 0.15cm] {}(h4);
		\path[-] (h3) edge node [above =0.15 cm,left = 0.15cm] {}(h5);
		\path[-] (h3) edge node [above =0.15 cm,left = 0.15cm] {}(h6);
		\path[-] (h4) edge node [above =0.15 cm,left = 0.15cm] {}(x1);
		\path[-] (h4) edge node [above =0.15 cm,left = 0.15cm] {}(x2);		
		\path[-] (h2) edge node [above =0.15 cm,left = 0.15cm] {}(x3);		
		\path[-] (h5) edge node [above =0.15 cm,left = 0.15cm] {}(x4);		
		\path[-] (h5) edge node [above =0.15 cm,left = 0.15cm] {}(x5);			
		\path[-] (h6) edge node [above =0.15 cm,left = 0.15cm] {}(x6);
		\path[-] (h6) edge node [above =0.15 cm,left = 0.15cm] {}(x7);				
		\node at (11,0) {CCCCAAGGGGGATAGTAGTCAAA};
		\node at (11,1) {CCACAAGGCGGCATACAGTCAAA};
		\node at (11,2) {ACCCCAGGGCGATAGTGGTCAAA};
		\node at (11,3) {ACGCAAGGCGGATAGCAGTCAAA};
		\node at (11,4) {ACCTAAGGGCGATAGTCGTCAAA};
		\node at (11,5) {ACTCAAGGGCGATAGTAGTCAAA};
		\node at (11,6) {ACCCAAGGGGGATAGAAGTCAAA};
    					
	\end{tikzpicture}
	
	\caption{A tree with $m=7$ observed nodes. The data consists of a sequence of characters at every terminal node.} 
	\label{fig:model_tree}
\end{figure}

A different approach for  tree recovery is based on spectral properties of the input data  \cite{allman2007molecular,eriksson2005tree}. Several methods work top-down, repeatedly applying spectral partitioning to the terminal nodes until each partition contains a single node \cite{matsui2020graph,zhang2011phylogeny}. However, there is no theoretical guarantee that the partitions match the structure of the tree. 
Of direct relevance to this manuscript is the recently proposed spectral neighbor joining (SNJ) \cite{jaffe2020spectral}, which consistently recovers the tree based on a spectral criterion.
Similarly to NJ, SNJ is a bottom-up method, which iteratively merges subsets of nodes to recover the tree.

Perhaps one of the most accurate approaches for tree recovery is to search for the topology that maximizes the likelihood of the observed data \cite{felsenstein2004inferring}.
Since computing the likelihood for every possible topology is intractable, many methods apply a local search to iteratively increase the likelihood function ~\cite{guindon2003simple,price2010fasttree,stamatakis2014raxml,zhou2018evaluating}. Though there is no guarantee that such a process will converge to the global maximum of the likelihood function, in many settings the resulting tree is more accurate than the one obtained by distance-based methods. The main disadvantage of likelihood-based algorithms is their slow runtime, which limits their applicability to trees of moderate size.  

With the dramatic increase in the sizes of measured datasets, there is a pressing need to develop fast tree recovery algorithms, able to handle trees with tens of thousands of nodes~\cite{tamura2004prospects,sanderson2003challenge}.
For example, the recently developed GESTALT method combines scRNA-seq readouts with CRISPR/Cas9 induced mutations to perform lineage tracing on tens of thousands of cells. \cite{quinn2021single,simeonov2021single}. For the multispecies coalescent model, recent works recover multiple gene trees, where each tree is composed of thousands of genes \cite{mirarab2015astral}. Recently, many works recovered the evolutionary history of the SARS-COV-2 virus, with over ten thousand variants \cite{morel2021phylogenetic}.   

Tree recovery problems with thousands of terminal nodes pose a significant computational challenge, as even distance-based methods may prove to be too slow.
To improve the scalability of slow but accurate methods such as maximum likelihood, a common framework known as \textit{divide-and-conquer} is to recover the tree by a two-step process \cite{molloy2019treemerge,warnow2018supertree}: 
(i) infer the tree structure independently for a large number of small possibly random subsets of terminal nodes;
(ii) compute the full tree by merging the small trees obtained in step (i). 
In \textit{supertree} methods, the small subsets of terminal nodes in step (i) overlap. Their merging step requires optimizing  a non-convex objective, which is computationally hard \cite{jiang2001polynomial,hillis1996molecular}. Thus, most supertree methods circumvent global optimization problems by iterative approaches for step (ii) \cite{strimmer1996quartet,warnow2018supertree}.
Recently, several methods were derived to merge subtrees with disjoint terminal nodes \cite{molloy2019treemerge, molloy2019statistically}. 
To apply these algorithms in a divide-and-conquer pipeline, 
the terminal nodes are partitioned according to an initial tree estimate computed by NJ.  
Despite these works, the problem of reconstructing large trees from limited amount of data is not yet fully resolved. In particular, there is still a need for fast and scalable approaches that also have strong recovery guarantees.

\paragraph{Contributions and outline} 
In this work we develop 
Spectral Top-Down Recovery (STDR), a scalable divide-and-conquer  approach backed by theoretical guarantees to recover large trees.
In contrast to previous methods, the partitioning of the terminal nodes in step (i) is deterministic.  Importantly, we prove that under mild assumptions the partitions are consistent with the unobserved tree structure. The importance of this consistency is that it simplifies considerably the merging process in step (ii) of the algorithm. 
Since  STDR is recursive, it is  instructive to replace the standard divide-and-conquer two step outline, with the following recursive description. 

\begin{itemize}
    \item[(i)] Partitioning: split the terminal nodes into two subsets.
    
    \item[(ii)] Recursive reconstruction: infer the latent tree of each subset. 
    When the partition size falls below a given threshold $\tau$, the tree is recovered by a  user-specified algorithm. Above this threshold, the reconstruction is done by recursively applying STDR to each subset. 
    
    \item[(iii)] Merging: reconstruct the full tree by merging the two small trees.
\end{itemize}
Each of the above three steps is explained in detail in Section~\ref{sec:algorithm}.
In step (i) we apply spectral partitioning to  
a weighted complete graph, with nodes that correspond to the terminal nodes of the tree and  weights based on a similarity measure described in Section \ref{sec:similarity_matrix}. 
In Section \ref{sec:consistency_partition} we prove that given an accurate estimate of these similarities, step (i) is consistent in the sense that the resulting subsets belong to two disjoint subtrees. For this proof, we derive a novel relation between latent tree models and a classic result from spectral graph theory known as Fiedler's theorem of nodal domains \cite{fiedler1975property}. This theorem is important in various learning tasks such as clustering data \cite{von2007tutorial}, graph partitioning \cite{ding2001min}, and low dimensional embeddings \cite{jaffe2020spectralb}.
To the best of our knowledge, this is the first guarantee for spectral partitioning in the setting of latent tree models. 

The output of step (ii) is the inner structure of two disjoint subtrees.
The task in step (iii) is to merge them into the full tree. 
 In Section \ref{sec:tree_merging}, we show that this task is equivalent to finding the root of an unrooted tree, given a reference set of one or more sequences, also known as an \textit{outgroup}. We derive a novel spectral-based method to find the root and prove its statistical consistency in Section \ref{sec:consistency_merging}. This approach is of independent interest, as finding the root of a tree is a common challenge in phylogenetics \cite{barriel1998rooting,boykin2010comparison,kinene2016rooting}.
Finite sample guarantees for the Jukes-Cantor model of evolution are derived in Section \ref{sec:finite_sample}. 

In Section \ref{sec:experiments} we compare the accuracy and runtime of various methods when applied to recover the full tree directly versus when used as subroutines in step (ii) of STDR. 
For example, Figure \ref{fig:kingman_performance} shows the results of recovering simulated trees with $2000$ terminal nodes generated according to the coalescent model \cite{wakeley2009coalescent}. 
As one baseline, we applied RAxML \cite{stamatakis2014raxml}, one of the most popular maximum likelihood software packages in phylogenetics.
With 8,000 samples, RAxML took over $5 \frac12$ hours to complete. 
In contrast, STDR with RAxML as subroutine and a threshold $\tau=128$ 
took approximately $21$ minutes, more than an order of magnitude faster. Importantly, in this setting, the trees recovered via STDR have similar accuracy to those obtained by applying RAxML directly. 
These and other simulation results illustrate the potential benefit of STDR in recovering large trees.

\section{Problem setup}\label{sec:setup}
	Let $\T$ be an unrooted binary tree with $m$ terminal nodes.
	We assume that each node of the tree has an associated discrete random variable
	over the alphabet $\{1,\ldots,\ell\}$. We denote by 
	$\bm x = (x_1,\ldots,x_m)$ the vector of the random variables at the $m$ observed terminal nodes of the tree, and by $\bm h = (h_1,\ldots,h_{m-2})$ the random variables at the non-terminal nodes. 
	We assume that these random variables form a Markov random field on $\T$.
    This means that given the values of its neighbors, the random variable at a node is statistically independent of the rest of the tree~\cite{chang1996full}. 
    An edge $e(h_i,h_j)$ connecting a pair of adjacent nodes $(h_i,h_j)$ is equipped with two
    transition matrices of size $\ell \times \ell$,
    \begin{equation}\label{eq:transition_matrix}
	P(h_i|h_j)_{ba} = \Pr[h_i=b | h_j=a], \qquad P(h_j|h_i)_{ba} = \Pr[h_j=b | h_i=a].
	\end{equation}
	Note that every pair of adjacent nodes may in general have different transition matrices. 
	
	Our observed data is a matrix $X=[\bm x^{(1)},\ldots,\bm x^{(n)}] \in \{1,\ldots ,\ell\}^{m \times n}$, where $\bm x^{(j)}$ are random i.i.d. realizations of $\bm x = (x_1,\ldots,x_m)$. Each row in the matrix is a sequence of length $n$ that corresponds to a terminal node in the tree, see illustration in Figure \ref{fig:model_tree}. 
	For example, in phylogenetics, each row in the matrix corresponds to a different species, while each column corresponds to a different location in a DNA sequence, see \cite{durbin1998biological} and references therein. 
	Figure \ref{fig:model_tree} shows an example with $m=7$ terminal nodes and $n=23$ observations. The support of each node is the DNA alphabet $A,C,G,T$, so $\ell=4$. 

	Given the matrix $X$,
	the task at hand is to recover the structure of the hidden tree $\T$. We assume that for every pair of adjacent nodes $(h_i,h_j)$, the corresponding $\ell \times \ell$ stochastic matrices $P(h_i|h_j)$ and $P(h_j|h_i)$ defined in \eqref{eq:transition_matrix}  are full rank, with determinants that satisfy
	\begin{equation} 
		\label{eq:assumption_1}
		0 < \delta < \det(P(h_i|h_j)),\det(P(h_j|h_i))< \xi < 1 .
	\end{equation}
	Eq. \eqref{eq:assumption_1} implies that the transition matrices are invertible and are not permutation matrices. This assumption is necessary for the tree's topology to be identifiable, see Proposition 3.1 in \cite{chang1996full} and  \cite{mossel2005learning}.
Next, to describe our approach we present several definitions related to unrooted trees, following the terminology of \cite{wilkinson2007clades}. 
\begin{definition}[clan]\label{def:clan}
    A clan is a subset of nodes in $\T$ that is connected to the rest of the tree by a single edge. 
\end{definition}
\begin{definition}[the root of a clan]\label{def:clan_root} A non-terminal node $h$ is termed the root of a clan $C$ if $h \in C$ and it is connected to the edge that separates $C$ from the rest of the tree.
\end{definition}
For example, in Figure \ref{fig:model_tree} $h_4$ and $h_5$ are the root nodes of the clans $C_1 = \{x_6,x_7,h_5\}$ and $C_2 = \{x_4,x_5,x_6,x_7,h_2,h_4,h_5\}$, respectively.
In our work, we will sometimes refer to the clans by their terminal nodes only (e.g. $\{x_6,x_7\}$ and $\{x_4,x_5,x_6,x_7\}$ for $C_1$ and $C_2$). 
\begin{definition}[adjacent clans]\label{def:adj_clans}
Let $C_1$ and $C_2$ be two disjoint subsets of terminal nodes that form two clans. If the union $C_1 \cup C_2$ forms a clan, then $C_1$ and $C_2$ are adjacent clans.
\end{definition}
Two disjoint clans whose respective root nodes share a common neighboring node are adjacent clans. For example, in Figure \ref{fig:model_tree} the clans $C_1 = \{x_4,x_5\}$ and $C_2 = \{x_6,x_7\}$ are adjacent. Their respective root nodes $h_4$ and $h_5$ are adjacent to $h_2$. This observation is important for the merging step of STDR.

\section{A spectral top-down approach for tree reconstruction}\label{sec:algorithm}
Here we present the three steps of the Spectral Top-Down Recovery (STDR) algorithm, as outlined in the introduction. Pseudocode for the method appears in Algorithm~\ref{alg:description}. We begin with the definition and properties of the similarity matrix and similarity graph.


\begin{algorithm}[t]
	\caption{STDR: Spectral Top-Down  Recovery}
	\label{alg:description}
	\begin{algorithmic}[1]
		\Statex {\bfseries Input:}\begin{tabular}[t]{ll}
            $X \in \{1,\ldots, \ell\}^{m \times n}$ & A matrix containing sequences from $m$ terminal nodes \\
            $\tau \in \mathbb{N}$ & Partition size threshold \\
            Alg & An algorithm for recovering small tree structures
		\end{tabular}
		\Statex {\bfseries Output:}\begin{tabular}[t]{ll}
				$\T$ & Estimated tree\\
		\end{tabular}
		\If {number of terminal nodes $m \leq \tau$  }
		    \State \Return Alg(X) \Comment{Recover small tree structures by a user defined algorithm}
		\EndIf
		\State Compute the similarity matrix $S$ from $X$ via Eq. \eqref{eq:similarity}
		\State Compute the Fiedler vector $v$ of $S$ 
		\Statex \(\triangleright\) \underline{Partitioning step}\medskip
		\State Partition the terminal nodes into two subsets $C_1$ and $C_2$ by thresholding $v$ via Eq. \eqref{eq:spectral_partition}
		\label{alg:partition_step}
		\Statex \(\triangleright\) \underline{Recursive reconstruction step}\medskip
		\State \label{algstep:stdr1}$\T_1 = \mathrm{STDR}(X(C_1, :), \tau, \mathrm{Alg})$ 
		\State  \label{algstep:stdr2} $\T_2 = \mathrm{STDR}(X(C_2, :), \tau, \mathrm{Alg})$ 
		\Statex \(\triangleright\) \underline{Merging step}\medskip
		\State Compute $u$, the first left singular vector of $S(C_1,C_2)$
		\For {all edges $e$ in $\T_1$}
		    \State Compute the edge score $d(e)$ from $u$ via Eq. \eqref{eq:edge_score}
		\EndFor
		\State Insert a root node for $\T_1$ into the edge $e_1 =\argmin_{e \in \T_1} d(e) $
		\State Compute $v$, the first right singular vector of $S(C_1, C_2)$
		\For {all edges $e$ in $\T_2$}
		    \State Compute the edge score $d(e)$ from $v$ via Eq. \eqref{eq:edge_score}
		\EndFor
		\State Insert a root node for $\T_2$ into the edge $e_2 =\argmin_{e \in \T_2} d(e) $
		\State Connect the roots of $\T_1$ and $\T_2$ to construct the merged tree $\T$
		\State \Return $\T$
	\end{algorithmic}
\end{algorithm}

\subsection{The pairwise similarity matrix and similarity graph}\label{sec:similarity_matrix}
	
	Similar to Eq. \eqref{eq:transition_matrix}, we define the $\ell \times \ell$ transition matrix for every pair $h_i,h_j$ of (not necessarily adjacent) nodes by
	\[
	P(h_i|h_j)_{ba} = \Pr[h_i=b|h_j=a].
	\]
	Note that due to the Markov assumption, the transition matrix is multiplicative along the edges of the tree. For example in Figure \ref{fig:model_tree},
	$P(x_1|x_2) = P(x_1|h_3)P(h_3|x_2)$.
	In \cite{jaffe2020spectral}, a similarity function between a pair of nodes $h_i$ and $h_j$ was defined as follows:
	\begin{equation}\label{eq:adjacent_similarity}
	S(h_i,h_j) = \sqrt{\det(P(h_i|h_j))\det(P(h_j|h_i))}.
	\end{equation}
	Similar to the transition matrix, the similarity is multiplicative along the edges of the tree and is bounded by $ \delta \leq S(h_i,h_j) \leq \xi$. Thus, it exhibits an exponential decay along the tree.
	For any two ordered sets of terminal or non-terminal nodes $A = \{a_1,\ldots a_r\}$ and $B= \{b_1,\ldots b_s\}$, we denote by $S(A,B)$ a matrix of size $r\times s$, where 
	\[
	S(A,B)_{ij} = S(a_i,b_j)~~~~\mbox{ for all}~1\leq i \leq r \mbox{ and } 1\leq j\leq s.
	\]
	To simplify notation, for the case where $A$ and $B$ are both equal to the full set of terminal nodes, we denote the similarity matrix by $S$:
	\begin{equation}\label{eq:similarity}
	S = S(\bm x,\bm x) ~~~~\mbox{where}~ \bm x = \{x_1,\ldots, x_m\}.
	\end{equation}
	where by definition, $S_{ii}=1$ $\forall(i)$. 
	The matrix $S$ is the adjacency matrix of the following graph.
\begin{definition}[Similarity graph]\label{def:similarity_graph}
The similarity graph $G$ is a complete graph whose vertices are the terminal nodes of $\T$. The weight assigned to every edge $e(x_i,x_j)$ is the similarity $S(x_i,x_j)$. 
\end{definition}
The relation between the spectral properties of $G$ and the topology of $\T$ forms the theoretical basis of our approach. 
	The following result from \cite[Lemma 3.1]{jaffe2020spectral} shows how the spectral structure of the similarity matrix $S$ relates to the structure of the underlying tree. 
	
	\begin{lemma}\label{lem:affinity_spectral}
        Let $A$ and $B$ be 
        a partition of the terminal nodes of an unrooted binary tree $\T$.
        The matrix $S(A,B)$ is rank-one if and only if $A$ and $B$ are clans of $\T$.
	\end{lemma}
    
     Lemma \ref{lem:affinity_spectral} implies that given the exact similarity matrix $S$, one can determine if a subset $A$ of terminal nodes is a clan in $\T$ by 
     computing the rank of $S(A,A^c)$, where $A^c = \bm x \setminus A$.
In practice, the exact similarity matrix $S$ is unknown. Yet, as shown in \cite{jaffe2020spectral}, a sufficiently accurate estimate $\hat{S}$, which in general is full rank, still allows to determine if a subset is a clan.


\subsection{Tree partitioning via spectral clustering}\label{sec:tree_splitting}
The aim of step (i) of STDR is to partition the terminal nodes into two clans of $\T$.
Our approach is based on the similarity graph $G$ of Definition \ref{def:similarity_graph}. One possible way to partition the graph is by the min-cut criteria. Given the exact similarity, this approach is guaranteed to yield two clans, see Lemma \ref{lem:min_cut} in the appendix. 
Though the min-cut problem can be solved efficiently \cite{von2007tutorial}, it often leads to unbalanced partitions of the graph, with the smaller one containing $1$ or $2$ terminal nodes. Since one goal is to reduce the runtime of the reconstruction algorithm in step (ii), we would like to avoid imbalanced partitions. 
To this end, we propose to partition the terminal nodes via a spectral approach based on the Fiedler vector.
\begin{definition}[Graph Laplacian and Fiedler vector]\label{def:laplacian}
The Laplacian matrix of a graph $G$ with a symmetric weight matrix $W$ is given by $
L_G = D-W,$
where $D$ is a diagonal matrix with $D_{ii} = \sum_j W(x_i,x_j)$.
The Fiedler vector is the eigenvector of $L_G$ that corresponds to the second smallest eigenvalue.
\end{definition}

In the STDR algorithm, we use the Fiedler vector $v$ of the similarity graph $G$ to partition the terminal nodes into two subsets $C_1$ and $C_2$ (Algorithm \ref{alg:description}, line \ref{alg:partition_step}), as follows:
\begin{equation}\label{eq:spectral_partition}
C_1 = \{i; v(i)  \geq 0\}, \qquad C_2 = \{i; v(i)  < 0\}.
\end{equation}
Importantly, 
in Section \ref{sec:consistency_partition} we prove that partitioning the nodes of $G$ via Eq. \eqref{eq:spectral_partition} yields two clans of the underlying tree $\T$. 
To illustrate this point, we created a tree graphical model from a symmetric binary tree with $m=128$ nodes, see Figure \ref{fig:full_tree-tree}. The transition matrices between adjacent nodes are all identical and were chosen according to the HKY model \cite{hasegawa1985dating}. We used this model to generate a dataset of nucleotide sequences of length $n=1,000$. Figure \ref{fig:full_tree-eigs} shows the Fiedler vector of the similarity graph estimated from the dataset. Here, the Fiedler vector exhibits a single dominant gap, and partitioning the terminal nodes by Eq. \eqref{eq:spectral_partition} yields two sets  $C_1$ and $C_2$ which are indeed clans of $\T$.
A similar example is shown in the appendix for a tree generated according to the coalescent model. 
In Section \ref{sec:finite_sample_partition} we derive an expression for the number of samples required to obtain two clans with high probability.


\begin{figure}[t]
    \centering
    \begin{subfigure}[b]{0.45\textwidth}
        \includegraphics[width = 0.8\textwidth]{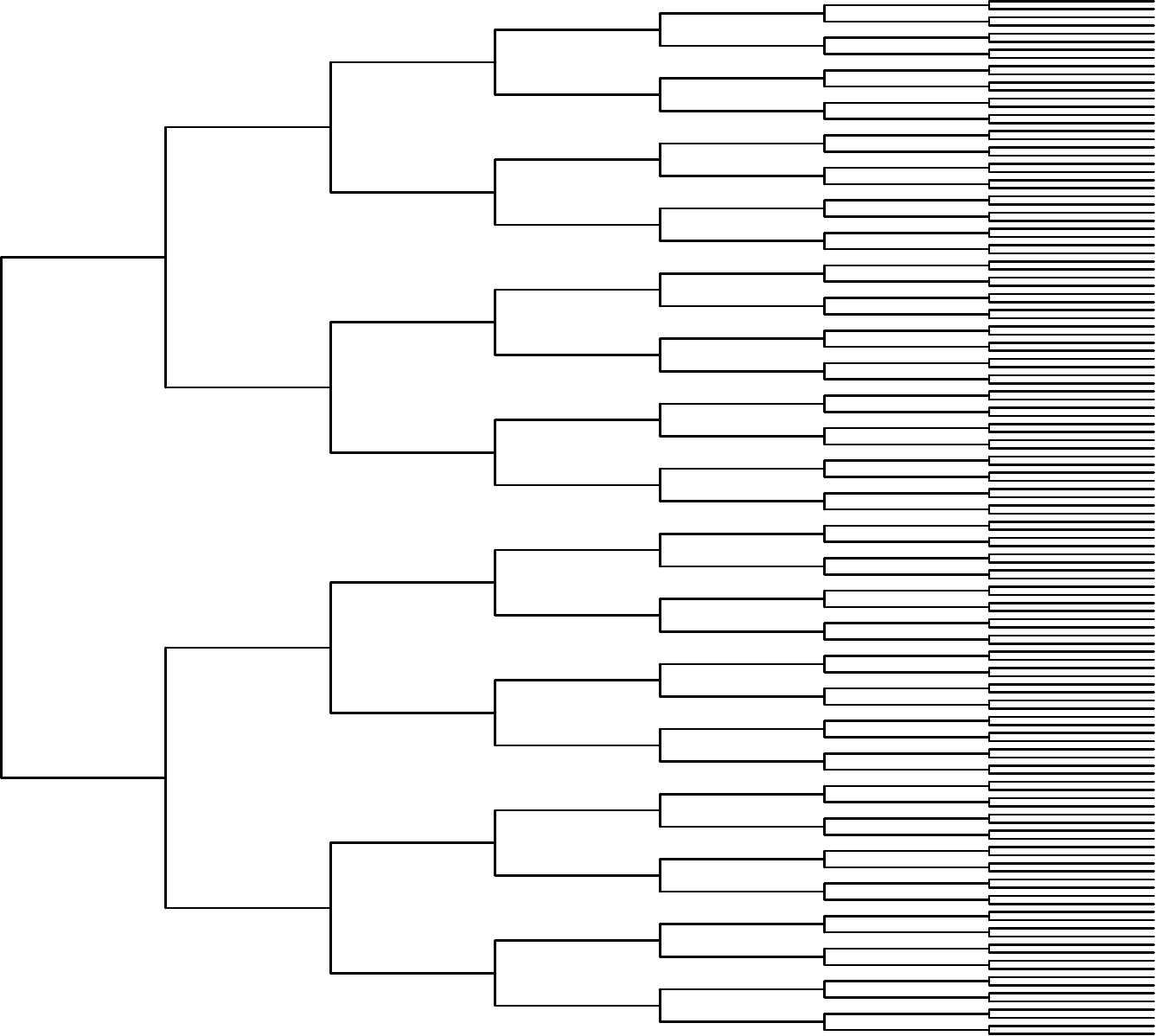}
        \caption{Illustration of a symmetric binary tree.}
        \label{fig:full_tree-tree}
 \end{subfigure}
    \begin{subfigure}[b]{0.53\textwidth}
        \includegraphics[width=\textwidth]{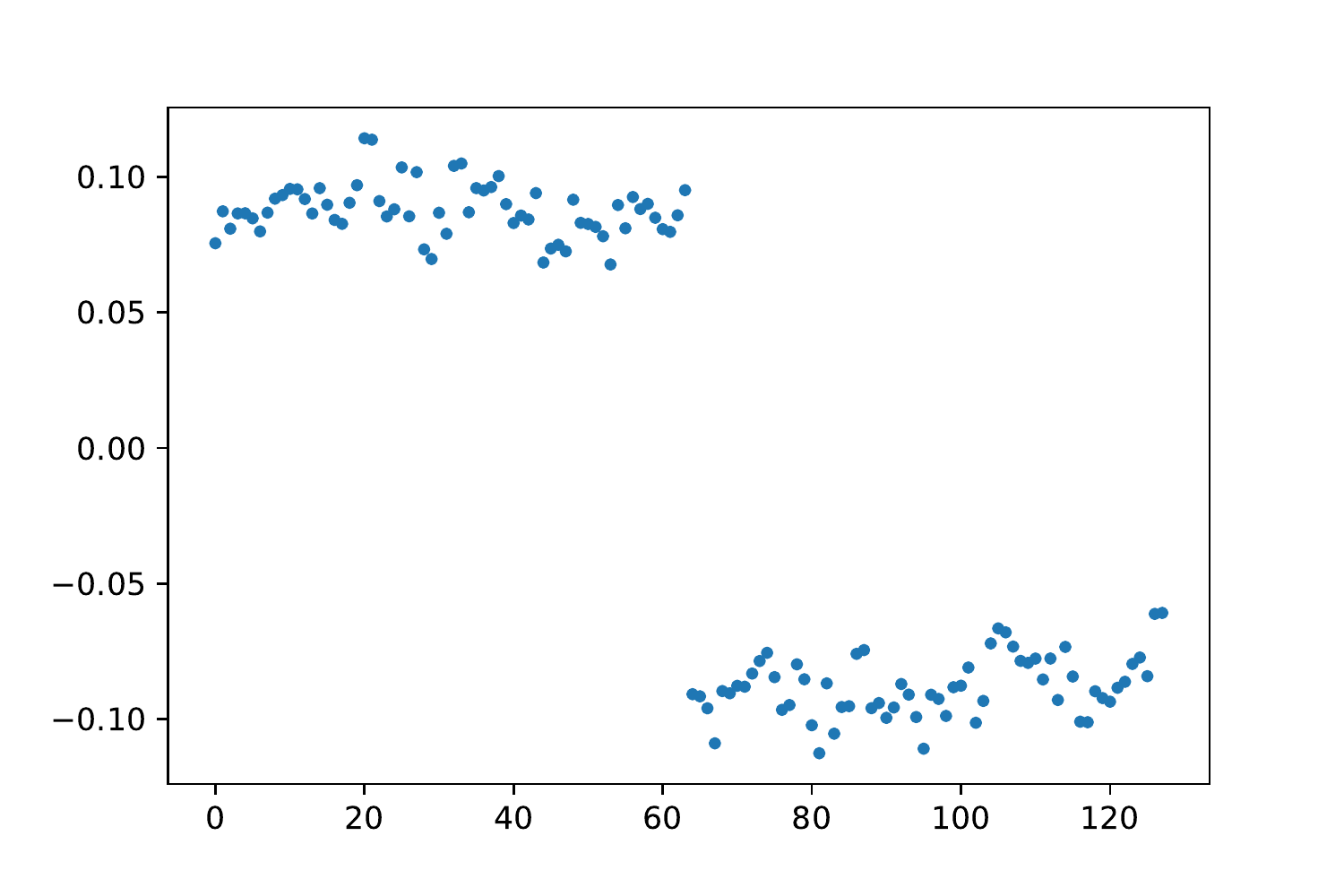}
        \caption{Fielder vector of binary symmetric tree}
        \label{fig:full_tree-eigs}
    \end{subfigure}
    \caption{Symmetric binary tree with 128 terminal nodes. The data consists of sequences of length $n=1000$ over the $\ell=4$ characters of the DNA alphabet, generated according to the HKY  model. }
    \label{fig:full_tree}
\end{figure}


\subsection{Recursive Reconstruction Step}
Step (i) of STDR outputs two sets of terminal nodes $C_1$ and $C_2$. Under certain conditions defined in Section \ref{sec:consistency_partition}, these are guaranteed to be two clans in the tree $\T$. 
The next task is to construct trees $\T_1$ and $\T_2$ that describe their latent internal structure. 
If $|C_1| > \tau$, then $\T_1$  is recovered by recursively reapplying the three steps of STDR to $C_1$.  
When $|C_1| \leq \tau$, the input is small enough that we consider it tractable to use a direct method for tree reconstruction, even a slow one like maximum likelihood. 


\subsection{Merging disjoint subtrees}\label{sec:tree_merging}

The output of step (ii) of STDR consists of the internal \textit{unrooted} tree structures $\T_1$ and $\T_2$ of two subsets of terminal nodes $C_1$ and $C_2$. 
Assuming steps (i) and (ii) were successful, then $C_1$ and $C_2$ are adjacent clans, and $\T_1$ and $\T_2$ are indeed their correct internal structure.
The remaining challenge in step (iii) is to recover the full tree $\T$ by correctly merging $\T_1$ and $\T_2$.

Since $\T_1$ and $\T_2$ are unrooted binary trees, to merge them it is necessary to add a root node to each of them. Adding a connecting edge between the two root nodes yields a binary unrooted tree and completes the merging process, see Figure \ref{fig:merge_trees} for an illustration. 
To add a root node to a subtree, we select one of its edges to be
the ``placeholder edge" (illustrated in red in Figure \ref{fig:merge_treesa}). 
Subsequently, the placeholder edge is  replaced with two edges connected to the root node. Importantly, as shown in Figure~\ref{fig:tree_clans}, changing the placeholder edge in either $\T_1$ or $\T_2$ yields a merged tree with a different topology.  

Thus, merging $\T_1$ and $\T_2$ reduces to the task of identifying the correct ``placeholder edge".
Here, we derive a novel spectral method for finding these edges.
To the best of our knowledge, our approach for merging subtrees is new and may be of independent interest for other applications, such as rooting unrooted trees \cite{kinene2016rooting,boykin2010comparison,barriel1998rooting}.  In the following lemma, whose proof is in Appendix \ref{appendix:sec_algorithm}, we describe a property of the placeholder edge that motivates our approach. 
 
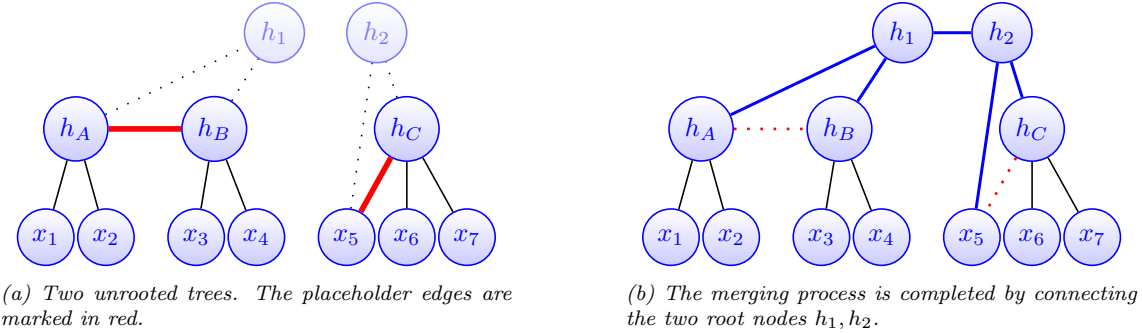
\begin{figure}[t]
    \begin{subfigure}[b]{0.45\textwidth}
      \centering
      \begin {tikzpicture}[-latex ,auto ,node distance =4 cm and 5cm ,on grid ,
		semithick ,scale=0.80,
		state/.style ={ circle ,top color =white , bottom color = blue!20 ,
			draw,blue , text=blue , minimum width = 0.75 cm}]	
		
		\node[state,  opacity =0.5] (h1) at (3.8,3.4) {$h_1$};
		
		\node[state] (hA) at (0.5,1.8) {$h_A$};
		\node[state] (hB) at (2.8,1.8) {$h_B$};
		
		\node[state] (T1) at (0,0) {$x_1$};
		\node[state] (T2) at (1,0) {$x_2$};
		\node[state] (T3) at (2.5,0) {$x_3$};
		\node[state] (T4) at (3.5,0) {$x_4$};
		
		\path[-] (hA) edge node [above =0.15 cm,left = 0.15cm] {}(T1);
		\path[-] (hA) edge node [above =0.15 cm,left = 0.15cm] {}(T2);
		\path[-] (hB) edge node [above =0.15 cm,left = 0.15cm] {}(T3);
		\path[-] (hB) edge node [above =0.15 cm,left = 0.15cm] {}(T4);
		
		\path[-, color = red, line width = 0.075 cm] (hA) edge node [above =0.15 cm,left = 0.15cm] {}(hB);
		
		\path[-,loosely dotted] (hA) edge node [above =0.15 cm,left = 0.15cm] {}(h1);
		\path[-, loosely dotted] (hB) edge node [above =0.15 cm,left = 0.15cm] {}(h1);
		\node[state,  opacity =0.5] (h2) at (5.5,3.4) {$h_2$};
		
		\node[state] (hC) at (6,1.8) {$h_C$};
		
		\node[state] (T5) at (5,0) {$x_5$};
		\node[state] (T6) at (6,0) {$x_6$};
		\node[state] (T7) at (7,0) {$x_7$};
	
		\path[-, color = red, line width = 0.075 cm] (hC) edge node [above =0.15 cm,left = 0.15cm] {}(T5);
		\path[-] (hC) edge node [above =0.15 cm,left = 0.15cm] {}(T6);	
		\path[-] (hC) edge node [above =0.15 cm,left = 0.15cm] {}(T7);
		
		\path[-, loosely dotted] (hC) edge node [above =0.15 cm,left = 0.15cm] {}(h2);
		\path[-, loosely dotted] (T5) edge node [above =0.15 cm,left = 0.15cm] {}(h2);
		\end{tikzpicture}
      \caption{Two unrooted trees. The placeholder edges are marked in red.}
      \label{fig:merge_treesa}
    \end{subfigure}\hfill
    \begin{subfigure}[b]{0.45\textwidth}
      \centering
      \begin {tikzpicture}[-latex ,auto ,node distance =4 cm and 5cm ,on grid ,
		semithick ,  scale = 0.8,
		state/.style ={ circle ,top color =white , bottom color = blue!20 ,
			draw,blue , text=blue , minimum width = 0.75 cm}]	
		
		\node[state] (h1) at (3.85,3.4) {$h_1$};
		
		\node[state] (hA) at (0.5,1.8) {$h_A$};
		\node[state] (hB) at (2.8,1.8) {$h_B$};
		
		\node[state] (T1) at (0,0) {$x_1$};
		\node[state] (T2) at (1,0) {$x_2$};
		\node[state] (T3) at (2.5,0) {$x_3$};
		\node[state] (T4) at (3.5,0) {$x_4$};
		
		\path[-] (hA) edge node [above =0.15 cm,left = 0.15cm] {}(T1);
		\path[-] (hA) edge node [above =0.15 cm,left = 0.15cm] {}(T2);
		\path[-] (hB) edge node [above =0.15 cm,left = 0.15cm] {}(T3);
		\path[-] (hB) edge node [above =0.15 cm,left = 0.15cm] {}(T4);
		
		\path[draw,loosely dotted,-, color = red, line width = 0.035 cm] (hA) edge node [above =0.15 cm,left = 0.15cm] {}(hB);
		
		\path[-, line width = 0.04 cm, color = blue] (hA) edge node [above =0.15 cm,left = 0.15cm] {}(h1);
		\path[-, line width = 0.04 cm, color = blue] (hB) edge node [above =0.15 cm,left = 0.15cm] {}(h1);
		\node[state] (h2) at (5.5,3.4) {$h_2$};
		
		\node[state] (hC) at (6,1.8) {$h_C$};
		
		\node[state] (T5) at (5,0) {$x_5$};
		\node[state] (T6) at (6,0) {$x_6$};
		\node[state] (T7) at (7,0) {$x_7$};
	
		\path[draw,loosely dotted,-, color = red, line width = 0.035 cm] (hC) edge node [above =0.15 cm,left = 0.15cm] {}(T5);
		\path[-] (hC) edge node [above =0.15 cm,left = 0.15cm] {}(T6);	
		\path[-] (hC) edge node [above =0.15 cm,left = 0.15cm] {}(T7);
		
		\path[-, line width = 0.04 cm, color = blue] (hC) edge node [above =0.15 cm,left = 0.15cm] {}(h2);
		\path[-, line width = 0.04 cm, color = blue] (T5) edge node [above =0.15 cm,left = 0.15cm] {}(h2);
		\path[-, line width = 0.04 cm, color = blue] (h1) edge node [above =0.15 cm,left = 0.15cm] {}(h2);
		
		\end{tikzpicture}
      \caption{The merging process is completed by connecting the two root nodes $h_1,h_2$.
      }
      \label{fig:merge_treesb}
    \end{subfigure}
	\caption{Merging example} 
	\label{fig:merge_trees}
\end{figure}
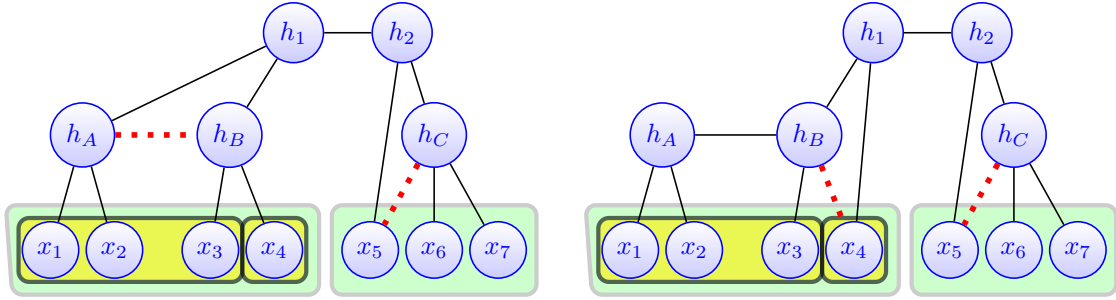
\begin{figure}
    \begin{subfigure}[b]{0.49\textwidth}
      \centering
      \begin {tikzpicture}[-latex ,auto ,node distance =4 cm and 5cm ,on grid ,
		semithick ,scale=0.85,
		state/.style ={ circle ,top color =white , bottom color = blue!20 ,
			draw,blue , text=blue , minimum width = 0.75 cm}]	
		
		\draw [ultra thick, draw=black, fill=green, opacity=0.2, rounded corners]
      (4.4,-0.7) -- (4.4,0.7) -- (7.6,0.7) -- (7.6,-0.7) -- cycle;
       
      	\draw [ultra thick, draw=black, fill=green, opacity=0.2, rounded corners]
      (-0.6,-0.7) -- (-0.7,0.7) -- (4.2,0.7) -- (4.2,-0.7) -- cycle;
       
		\draw [ultra thick, draw=black, fill=yellow, opacity=0.6, rounded corners]
      (-0.5,-0.5) -- (-0.5,0.5) -- (3,0.5) -- (3,-0.5) -- cycle;
		\draw [ultra thick, draw=black, fill=yellow, opacity=0.6, rounded corners]
      (3,-0.5) -- (3,0.5) -- (4,0.5) -- (4,-0.5) -- cycle;
       

		\node[state] (h1) at (3.8,3.4) {$h_1$};
		
		\node[state] (hA) at (0.5,1.8) {$h_A$};
		\node[state] (hB) at (2.8,1.8) {$h_B$};
		
		\node[state] (T1) at (0,0) {$x_1$};
		\node[state] (T2) at (1,0) {$x_2$};
		\node[state] (T3) at (2.5,0) {$x_3$};
		\node[state] (T4) at (3.5,0) {$x_4$};
		
		\path[-] (hA) edge node [above =0.15 cm,left = 0.15cm] {}(T1);
		\path[-] (hA) edge node [above =0.15 cm,left = 0.15cm] {}(T2);
		\path[-] (hB) edge node [above =0.15 cm,left = 0.15cm] {}(T3);
		\path[-] (hB) edge node [above =0.15 cm,left = 0.15cm] {}(T4);
		
		\path[-, color = red,loosely dotted, line width = 0.075 cm] (hA) edge node [above =0.15 cm,left = 0.15cm] {}(hB);
		
		\path[-] (hA) edge node [above =0.15 cm,left = 0.15cm] {}(h1);
		\path[-] (hB) edge node [above =0.15 cm,left = 0.15cm] {}(h1);
		\node[state] (h2) at (5.5,3.4) {$h_2$};
		
		\node[state] (hC) at (6,1.8) {$h_C$};
		
		\node[state] (T5) at (5,0) {$x_5$};
		\node[state] (T6) at (6,0) {$x_6$};
		\node[state] (T7) at (7,0) {$x_7$};
	
		\path[-, color = red,loosely dotted, line width = 0.075 cm] (hC) edge node [above =0.15 cm,left = 0.15cm] {}(T5);
		\path[-] (hC) edge node [above =0.15 cm,left = 0.15cm] {}(T6);	
		\path[-] (hC) edge node [above =0.15 cm,left = 0.15cm] {}(T7);
		
		\path[-] (hC) edge node [above =0.15 cm,left = 0.15cm] {}(h2);
		\path[-] (T5) edge node [above =0.15 cm,left = 0.15cm] {}(h2);
		
		\path[-] (h1) edge node [above =0.15 cm,left = 0.15cm] {}(h2);
		
		\end{tikzpicture}
      \caption{Placeholder edges set to $e(h_A,h_B)$ and $e(h_C,x_5)$.}
      \label{fig:tree_clans_option1}
    \end{subfigure}\hfill
    \begin{subfigure}[b]{0.49\textwidth}
      \centering
      \begin {tikzpicture}[-latex ,auto ,node distance =4 cm and 5cm ,on grid ,
		semithick ,scale=0.85,
		state/.style ={ circle ,top color =white , bottom color = blue!20 ,
			draw,blue , text=blue , minimum width = 0.75 cm}]	
	
		\draw [ultra thick, draw=black, fill=green, opacity=0.2, rounded corners]
      (4.4,-0.7) -- (4.4,0.7) -- (7.6,0.7) -- (7.6,-0.7) -- cycle;
       
      	\draw [ultra thick, draw=black, fill=green, opacity=0.2, rounded corners]
      (-0.6,-0.7) -- (-0.7,0.7) -- (4.2,0.7) -- (4.2,-0.7) -- cycle;
       
		\draw [ultra thick, draw=black, fill=yellow, opacity=0.6, rounded corners]
      (-0.5,-0.5) -- (-0.5,0.5) -- (3,0.5) -- (3,-0.5) -- cycle;
		\draw [ultra thick, draw=black, fill=yellow, opacity=0.6, rounded corners]
      (3,-0.5) -- (3,0.5) -- (4,0.5) -- (4,-0.5) -- cycle;
       
		\node[state] (h1) at (3.8,3.4) {$h_1$};
		
		\node[state] (hA) at (0.5,1.8) {$h_A$};
		\node[state] (hB) at (2.8,1.8) {$h_B$};
		
		\node[state] (T1) at (0,0) {$x_1$};
		\node[state] (T2) at (1,0) {$x_2$};
		\node[state] (T3) at (2.5,0) {$x_3$};
		\node[state] (T4) at (3.5,0) {$x_4$};
		
		\path[-] (hA) edge node [above =0.15 cm,left = 0.15cm] {}(T1);
		\path[-] (hA) edge node [above =0.15 cm,left = 0.15cm] {}(T2);
		\path[-] (hB) edge node [above =0.15 cm,left = 0.15cm] {}(T3);
		\path[-, color = red,loosely dotted, line width = 0.075 cm] (hB) edge node [above =0.15 cm,left = 0.15cm] {}(T4);
		
		\path[-] (hA) edge node [above =0.15 cm,left = 0.15cm] {}(hB);
		
		\path[-] (T4) edge node [above =0.15 cm,left = 0.15cm] {}(h1);
		\path[-] (hB) edge node [above =0.15 cm,left = 0.15cm] {}(h1);
		\node[state] (h2) at (5.5,3.4) {$h_2$};
		
		\node[state] (hC) at (6,1.8) {$h_C$};
		
		\node[state] (T5) at (5,0) {$x_5$};
		\node[state] (T6) at (6,0) {$x_6$};
		\node[state] (T7) at (7,0) {$x_7$};
	
		\path[-, color = red,loosely dotted, line width = 0.075 cm] (hC) edge node [above =0.15 cm,left = 0.15cm] {}(T5);
		\path[-] (hC) edge node [above =0.15 cm,left = 0.15cm] {}(T6);	
		\path[-] (hC) edge node [above =0.15 cm,left = 0.15cm] {}(T7);
		
		\path[-] (hC) edge node [above =0.15 cm,left = 0.15cm] {}(h2);
		\path[-] (T5) edge node [above =0.15 cm,left = 0.15cm] {}(h2);
		
		\path[-] (h1) edge node [above =0.15 cm,left = 0.15cm] {}(h2);
		\end{tikzpicture}
      \caption{Placeholder edges set to $e(h_B,x_4)$ and $e(h_C,x_5)$.}
      \label{fig:tree_clans_option2}
    \end{subfigure}
	\caption{Different choices of placeholder edges result in a different merged trees.} 
	\label{fig:tree_clans}
\end{figure}

\begin{lemma}\label{lem:correct_placeholder}
Let $C_1$ be a set of terminal nodes that forms a clan in $\T$, and let $\T_1$ be the internal structure of $C_1$. An edge $e \in \T_1$ is the correct placeholder edge if and only if it partitions $C_1$ into two sets  $A(e),B(e)$, such that both form clans in $\T$.
\end{lemma}
Lemma  \ref{lem:correct_placeholder} is illustrated in Figure \ref{fig:merge_treesa}. The edge $e(h_A, h_B)$ divides the left subtree into the clans $\{x_1, x_2\}$ and $\{x_3, x_4\}$. These subsets also form clans in the full tree depicted in Figure \ref{fig:merge_treesb}.
Next, using Lemma \ref{lem:correct_placeholder}, we derive a spectral characterization of the correct placeholder edge. 
Recall that by Lemma \ref{lem:affinity_spectral}, the matrix $S(C_1,C_2) \in \RR^{|C_1|\times |C_2|}$, 
is rank one. Thus, 
\begin{equation}\label{eq:clans_svd}
    S(C_1,C_2) = u \sigma v^T, \quad \mbox{ where }\quad\|v\|=\|u\|=1, \quad \mbox{ and } \quad \sigma>0.
\end{equation}
Given a placeholder edge $e$ and its corresponding partition of terminal nodes $A(e)$ and $B(e)$, we denote by $u_{A(e)},u_{B(e)}$ the entries of $u$ that correspond to $A(e)$ and $B(e)$, respectively. 
The following lemma, proven in Appendix \ref{appendix:sec_algorithm}, characterizes the correct placeholder edge in terms of  $u_{A(e)}$ and $u_{B(e)}$.


\begin{lemma}\label{lem:merge_trees}
An edge $e$ is the correct placeholder edge of $\T_1$ if and only if there exists a constant $\alpha$ such that
\begin{equation}\label{eq:correct_placeholder}
S(A(e),B(e)) = u_{A(e)} \alpha u_{B(e)}^T.
\end{equation}
\end{lemma}
In practice we can only compute an estimate of $S$. Motivated by Lemma \ref{lem:merge_trees}, we propose to determine the placeholder edge $e^\ast$ by minimizing the following score function,
\begin{equation}\label{eq:edge_score}
e^\ast = \argmin_e d(e) = \argmin_e \frac{1}{\|S(A(e),B(e))\|_F} \min_{\alpha} \|S(A(e),B(e)) - u_{A(e)}\alpha u_{B(e)}^T \|_F.
\end{equation}
The normalizing factor $\|S(A(e),B(e))\|_F$ is added since the size of $S(A(e),B(e))$ changes for every edge $e$. Note that given the exact matrix $S$, at the correct placeholder edge $d(e^\ast)=0$.
In Section \ref{sec:finite_sample_mearging} we derive an expression for the number of samples required to obtain the correct  placeholder edge by Eq. \eqref{eq:edge_score} with high probability.

\subsection{Computational complexity}
We analyze the complexity of each step of STDR separately.
We assume that the similarity or distance matrix are given.
To simplify the analysis, we assume a balanced binary tree, and that the partition steps gave $m/\tau$ subsets of size $\tau$ each. 
We denote by $B(k)$ the complexity of recovering the topology of a tree with $k$ terminal nodes by the given subroutine $\text{Alg}$.
\begin{enumerate}
    \item Given the  similarity matrix, partitioning a set of $k$ terminal nodes is $\OO(k^2)$, due to the computation of the Fiedler vector of the positive semi-definite  Laplacian matrix  \cite[Chapter 2]{stewart2001matrix}.
    \item The complexity of merging two subtrees with $k$ terminal nodes each is composed of two parts: 
    (i) compute the leading singular vector of the matrix $S(C_1,C_2) \in \RR^{k \times k}$, which takes $\OO(k^2)$ operations;
    (ii) compute the score for every edge as in Eq. \eqref{eq:edge_score}. 
    The number of operations required for the least square operation in the numerator of Eq. \eqref{eq:edge_score}, as well as computing the Frobenius norms in the numerator and denominator is proportional to the number of elements in $S(A(e),B(e))$. 
    Thus, the total complexity of computing the score for all edges in $\T_1$ (and similarly $\T_2$) is  $O(\sum_{e\in \T_1} |A(e)||B(e)|)$. For a balanced tree, this term is  equal to
    \[
    \frac{k^2}{4} + \sum_{i=2}^{\log k} \underbrace{2^i}_{\substack{\text{Number of partitions} \\\text{of size } k/2^i} } \underbrace{\frac{k}{2^i}}_{|A(e)|}\underbrace{\Big(k-\frac{k}{2^i}\Big)}_{|B(e)|}= 
    \OO(k^2\log k).
    \]
    We remark that if the two trees are highly imbalanced, the complexity may increase up to $\OO(k^3)$.
\end{enumerate}

Let $T(m)$ be the complexity of the partitioning and merging operations of STDR, excluding the complexity of the subroutine algorithm that recovers the structure of small trees. We have that
\[
T(m) =  \underbrace{\OO(m^2)}_{\text{partitioning}} + 2T(m/2) + \underbrace{\OO(m^2 \log m)}_{\text{merging}} = 2T(m/2) + \OO(m^2 \log m).
\]
By the Master theorem \cite{bentley1980general},
\begin{equation}\label{eq:complexity_T}
    T(m) = \OO(m^2 \log m).
\end{equation}
Thus, the total complexity of STDR is 
\[
\OO(m^2 \log m + (m/\tau)B(\tau)).
\]
For example, 
the complexity of NJ is $B(\tau)=\OO(\tau^3)$. Thus, the complexity of STDR+NJ is $\OO(m^2\log m + m\tau^2)$, which for $\tau = \OO(1)$ improves upon the $\OO(m^3)$ complexity of running NJ to recover the full tree.
In the simulation section, we show that STDR+NJ outperforms NJ in accuracy while being about an order of magnitude faster. 


An important property of the STDR algorithm, in terns of actual runtime, is that it is embarrassingly parallel. 
Specifically, steps \ref{algstep:stdr1} and \ref{algstep:stdr2} in Algorithm \ref{alg:description} can be
executed in two independent processes.
This may result in up to $k$ parallel processes, where $k$ is the number of partitions.


\section{Correct tree recovery of STDR}
In this section we consider the population setting where the similarity matrix $S$ is known. 
In this setting we prove that STDR correctly recovers the underlying tree. We do so by analyzing the
partitioning step (i) and the  merging step (iii) of STDR. 
Our key results are Theorem   \ref{thm:consistenct_partitioning},  which states that step (i) is guaranteed to yield disjoint clans, and Theorem \ref{thm:merge_consistency}, which states that given accurate trees for two clans, step (iii) recovers the exact structure of the full tree.
Combining these two results directly yields the following theorem establishing the correctness of STDR in the population setting. 


\begin{theorem}\label{thm:consistency}
Given an exact similarity matrix $S$, and assuming that the subroutine Alg correctly recovers the internal structure of its input, STDR recovers the exact latent tree $\T$.
\end{theorem}

\label{sec:consistency}

\subsection{Consistency of the partition step} \label{sec:consistency_partition}
The following theorem proves that given the exact similarity matrix, partitioning the terminal nodes of the tree by thresholding the Fiedler vector as described in Section \ref{sec:similarity_matrix} yields two  adjacent clans.

\begin{theorem}\label{thm:consistenct_partitioning}
Let $G$ be the similarity graph of a binary tree $\T$. Denote by $v$ the Fiedler vector of $G$ and
by $\{C_1,C_2\}$ 
a partition of the terminal nodes according to the sign pattern of $v$ as in Eq. \eqref{eq:spectral_partition}.
Then $C_1,C_2$ are adjacent clans in $\T$.
\end{theorem}

Before proving Theorem \ref{thm:consistenct_partitioning}, we would like to put its novelty into the context of related results. In A result similar in nature to Theorem \ref{thm:consistenct_partitioning} was proved for  hierarchical block models (HBM) \cite{balakrishnan2011noise}
where the underlying block structure of a given connectivity matrix  is recovered by recursive partitioning according to its Fiedler vector. 
The statistical guaranty, however, is derived by making additional assumptions on the structure of the tree as well as its parameters. Theorem \ref{thm:consistenct_partitioning}, in contrast, is true for any tree structure and parameters. 
A different distance-based approach for tree partitioning was derived in chapter 4 of \cite{griffing2012connections}. This approach is  guaranteed to yield two clans, but only given the exact distance matrix between terminal nodes. 
In Appendix \ref{appendix:griffing} we show empirically that our similarity based approach is more robust than the distance based approach, specifically in cases where the number of samples is limited.


For the proof of Theorem \ref{thm:consistenct_partitioning}, we present several preliminaries on graphs.
First, we define the Schur complement of a matrix, which plays an important role in graph theory~\cite{dorfler2012kron}. 
\begin{definition}[Schur complement]\label{def:schur}
Let $A,B,C$ and $D$ be matrices of dimensions $p \times p , p \times q, q \times p$ and $q \times q$, respectively. Assume $D$ is invertible and consider the matrix 
\[
M = \left[\begin{matrix} A & B \\ C & D \end{matrix}\right],
\]
of size $(p + q) \times (p + q)$. The Schur complement of $M$ with respect to $D$ is the $p \times p$ matrix
\[
M/D = A - BD^{-1}C.
\]
\end{definition}
Let $H$ be a graph with a set of nodes $V$ and Laplacian matrix $L$. We denote by $L_R$ the principal sub-matrix of $L$ that corresponds to a subset of nodes $R \subset V$. The Schur complement of $L$ with respect to $L_R$ yields the Laplacian of a different graph, with $|V-R|$ nodes  \cite{crabtree1966applications, dorfler2012kron}. We denote this Laplacian matrix by $L_{H/R}$.
The rows and columns of $L_{H/R}$ correspond to vertices of $H$ that are not in $R$.
When the graph is a tree $\T$, and $R$ is the set of its non-terminal nodes, then $L_{\T/R}$ is the Laplacian of a complete graph $G$ whose nodes are the terminal nodes of $\T$. 

Equipped with these definitions, we proceed to the proof of Theorem~\ref{thm:consistenct_partitioning}. The
proof consists of two parts, that correspond to Theorem \ref{thm:schur_complement_thm} and Lemma \ref{lem:graph_construction}. Theorem \ref{thm:schur_complement_thm}, which is a rephrase of Theorem 3.3 of \cite{stone2009fiedler}, shows that one can partition the terminal nodes of a tree $\T$ into two clans via the Fiedler vector of $L_{\T/R}$, where $R$ is the set of all internal nodes. 
\begin{theorem}[\cite{stone2009fiedler},Theorem 3.3\footnote{For clarity, we rephrased the theorem from \cite{stone2009fiedler} according to our terminology.}]\label{thm:schur_complement_thm}
 Let $\T$ be a tree with a node set $V$ and a subset of non terminal nodes $R\subset V$. We denote by $L_\T$ the Laplacian of $\T$ and by $L_{\T/R}$ the Laplacian of a graph $G$ obtained by Schur complement of $L_\T$ with respect to $R$. 
Let $v$ be the Fiedler vector of $G$, and $C_1,C_2$ the following partition of the terminal nodes,
\[
C_1 = \{i\in V \setminus R ;\; v(i) \leq 0\},  \qquad C_2 = \{j\in V\setminus R ; \; v(j) > 0\}.
\]
Then $C_1$ and $C_2$ are adjacent clans in $\T$.
\end{theorem}

Theorem \ref{thm:schur_complement_thm}, however, is not directly applicable to our setting, since computing $L_{\T/R}$ requires knowledge of the unknown similarities between all  nodes of $\T$, including its unobserved nodes.
Here, we derive Lemma \ref{lem:graph_construction} that shows that 
for any tree $\T$, there is a \textit{twin tree} $\widetilde \T$ with the same topology, such that $L_{\widetilde \T/R} = L_G$. This result, proven in appendix \ref{appendix:proof_of_consistensy_lemma},  provides the critical missing link required for inference of the latent tree from the similarity matrix, which can be estimated from observed data. 


\begin{lemma}
\label{lem:graph_construction}
Let $\T$ be a tree with a set of non-terminal nodes $R$.  Let $G$ be the similarity graph of $\T$. Then there is a tree $\widetilde{\T}$ with the same topology as $\T$ but different edge weights, such that 
\[
 L_G = L_{\widetilde{\T}/R}.
\]
\end{lemma}
Combining Lemma \ref{lem:graph_construction} with Theorem \ref{thm:schur_complement_thm} yields the following proof 
 of Theorem \ref{thm:consistenct_partitioning}.

\begin{proof}[Proof of Theorem \ref{thm:consistenct_partitioning}]
Let $L_G$ be the Laplacian matrix of the similarity graph $G$. By Lemma \ref{lem:graph_construction} there is a tree $\widetilde \T$ with the same topology as $\T$ such that
$
L_G = L_{\widetilde T/R}.
$
By Theorem \ref{thm:schur_complement_thm}, partitioning the terminal nodes of $\widetilde \T$ according to the sign pattern of the Fiedler vector of $L_{\widetilde \T/R}$ yields adjacent clans in $\widetilde \T$. Since $L_G = L_{\widetilde \T/R}$ and $\widetilde \T$ has the same topology as $\T$, it follows that 
partitioning the terminal nodes of $\T$ according to the Fiedler vector of $L_G$ yields adjacent clans in $\T$.
\end{proof}

\subsection{Correctness of the merging step} \label{sec:consistency_merging}
Step (iii) of STDR merges the two subtrees, $\T_1$ and $\T_2$, that were constructed from the two disjoint subsets of terminal nodes $C_1$ and $C_2$. As described in Section \ref{sec:algorithm}, this step is done by finding for each tree its placeholder edge as  the edge with the smallest score $d(e)$, Eq. \eqref{eq:edge_score}. 
Here, we prove that this merging step is correct,  under the following two assumptions on its input (the output of steps (i) and (ii)):
the two subtrees $\T_1,\T_2$ correspond to adjacent clans in $\T$ and their internal structure was recovered correctly.

\begin{theorem}\label{thm:merge_consistency}
Let $C_1$ and $C_2$ be the terminal nodes of two adjacent clans that partition a tree $\T$. Let $\T_1$ and $\T_2$ be the internal structures of these clans. Then given the exact similarity matrix $S(C_1,C_2)$, minimizing the criterion in Eq. \eqref{eq:edge_score} yields the correct placeholder edge.
\end{theorem}
\begin{proof}
By Lemma \ref{lem:merge_trees}, for the correct placeholder edge $e^*$ there exists an $\alpha \in \RR$ such that
\[
S(A(e^*),B(e^*)) = u_{A(e^*)} \alpha u_{B(e^*)}^T.
\]
Hence $d(e^*) = 0$. If $e$ is an incorrect placeholder edge, then again according to Lemma \ref{lem:merge_trees} there is no constant $\alpha$ that satisfies the equation, and hence $d(e) > 0$ which implies $e^\ast = \argmin d(e)$.
\end{proof}

\section{Finite sample guarantees for STDR}\label{sec:finite_sample}

In practice, the true similarity matrix $S$ is unknown, and an estimate $\hat{S}$ is computed from a sequence data of length $n$.
In this section we show that STDR is still able to correctly recover the tree provided that $\hat{S}$ is sufficiently close to $S$. Specifically, in sections \ref{sec:finite_sample_partition} and \ref{sec:finite_sample_mearging}, we derive  lower bounds on the number of samples required for the partitioning step
and the merging step to succeed with high probability.  In Section \ref{sec:finite_sample_comperison} we compare these results to the guarantees available for other tree recovery algorithms. 

For simplicity, in the finite sample  analysis, we assume the Jukes-Cantor (JC) model of sequence evolution, where each transition matrix is parameterized by a single mutation rate $\theta(i,j)$:
\begin{equation}
P(h_i|h_j)_{ba} = P[h_i = b|h_j=a] = 
\begin{cases}\label{eq:jc_model}
        1-\theta(i,j) & a = b \\
        \theta(i,j)/(\ell-1) & a \neq b.
\end{cases}
\end{equation}
According to this model, the similarity between adjacent nodes defined in Eq. \eqref{eq:similarity} simplifies to
\[
S(h_i,h_j) = \left(1-\frac{\ell}{\ell-1}\theta(i,j) \right)^{\ell-1}.
\]
By Eq. \eqref{eq:assumption_1} the similarity is strictly positive and hence $\theta(i,j) < (\ell-1)/\ell$.
We remark that our analysis can be extended, under minor additional assumptions to  more general models of evolution as in \cite[Lemma 4.8]{jaffe2020spectral}.
We present results for the top level of the tree partitioning and merging.
Following the proof, we show in Remark \ref{rem:one_step_to_many} that the same guarantees hold for multiple partitions and merging steps. 



\subsection{Finite sample guarantees for the partitioning step}\label{sec:finite_sample_partition}

We compute the number of samples $n$ required for the partitioning step to yield two clans with high probability. To this end, we require that in the population setting, the entries of the Fiedler vector are bounded away from zero.
To that end, we assume that the similarity matrix $S$ satisfies the hierarchical constant block model (CBM) addressed in \cite{balakrishnan2011noise}. We assume there is a hierarchy of partitions $A_C$ and $B_C$, such that for each partition there is a (different) constant $c$ such that 
\begin{align}\label{eq:hbm}
S(x,y) &= c \qquad \forall (x,y) \in A_C \times B_C, \notag \\
S(x,y) &> c \qquad \forall (x,y) \in A_C \times A_C \quad  \text{and} \quad \forall(x,y) \in B_C \times B_C.
\end{align}
In phylogenetics, this assumption is satisfied in the molecular clock model \cite{kumar2005molecular}, where the probability of mutation between adjacent nodes is determined by two factors: (i) the edge length between them and (ii) a mutation rate matrix that is constant throughout the tree. The structure of the rate matrix is determined by the choice of  evolutionary model, such as  Jukes-Cantor. In addition, the path length between all terminal nodes and the root is constant. 
This implies that for every ancestor $h$ (internal node) the similarity between the terminal nodes $A_C$ on the left of $h$ and the nodes $B_C$ on the right of $h$ is constant as in Eq. \eqref{eq:hbm}. 
For the hierarchy of partitions, we denote by $\eta$ the maximum over all partitions $C$ of the ratio between the size of left and right parts $A_C,B_C$. 
\begin{equation}\label{eq:max_ratio}
\eta = \max_C \{|A_C|/|B_C|,|B_C|/|A_C|\}.
\end{equation}
This factor serves as a measure for the \textit{balancedness} of the tree. In addition, we denote by $r(\T)$ the diameter of $\T$, which is the maximal distance between pairs of terminal nodes,
\begin{equation}\label{eq:diam}
r(\T) = \max_{i,j} (-\log S(x_i,x_j)).
\end{equation}


Finally, we denote by $h(\T)$ the depth of $\T$ as defined in \cite{erdHos1999few}:
\begin{definition}\label{def:depth}
Let $\T_1,\T_2$ be two rooted subtrees with respective roots $h_1,h_2$ obtained by removing an edge $e(h_1,h_2)$ from $\T$. Let $d_1(e),d_2(e)$ be the distances $\log S(h_1,x_i) $ and $\log S(h_2,x_j)$ from $h_1,h_2$ to  the \textit{closest} leaves $x_i$ and $x_j$ in $\T_1,\T_2$, respectively. Then 
\begin{equation}\label{eq:depth}
h(\T) = \max_e \max \{d_1(e),d_2(e)\}.
\end{equation}
\end{definition}
Note that $h(\T) < r(\T)$ as the maximal distance between terminal nodes is larger than any distance between a pair of terminal and non terminal nodes. The following theorem bounds the number of samples $n$ by the properties of the tree defined in Eqs. \eqref{eq:max_ratio},\eqref{eq:depth} and \eqref{eq:depth}.   
\begin{theorem}\label{thm:finite_split}
Let $\T$ be a Jukes-Cantor evolutionary tree with $m$ terminal nodes, with a similarity matrix $S$ that satisfies the assumptions made for the CBM. If the number of samples $n$ satisfies
\[
n \geq 4\ln \left(\frac{2m^2}{\epsilon}\right)\eta \ell^2 m(\sqrt{m}+1)^2e^{2r(t)}\max\left\{1,\frac{(1+\eta)^2}{(e^{r(\T)-h(\T)}-1)^2}\right\},
\]
then STDR partitions the terminal nodes into two clans with probability at least $1-\epsilon$.
\end{theorem}

To prove the theorem, we derive a bound on the error that the partitioning step can tolerate in the estimate $\hat S$.


\begin{lemma}\label{lem:davis_khan_for_trees_mc}
	Assume a tree with $m$ terminal nodes generated according to the molecular clock model. If the estimate $\hat S$ of its similarity matrix satisfies
	\begin{equation}
		\|S-\hat S\| \leq 
    \frac{\sqrt{m}e^{-r(\T)}}{\sqrt{\eta} 2^{3/2}(\sqrt{m}+1)}\min\left\{1,\frac{1}{1+\eta}\big(e^{r(\T)-h(\T)}-1\big)\right\},    
	\end{equation}	 
	then STDR correctly partitions the terminal nodes into two clans.
\end{lemma}

In our proof, we use the following lemma  regarding the spectrum of the Laplacian. This Lemma is a reformulation of lemma 7 from \cite{balakrishnan2011noise}
that addresses the spectrum of the CBM.
\begin{lemma}\label{lem:symmetric_spectrum}
Consider a tree with $m$ terminal nodes generated according to the molecular clock model. Let $L$ be the Laplacian of its similarity graph. The first second and third smallest eigenvalues of $L$ satisfy
\[
\lambda_1 = 0, \qquad \lambda_2 = m e^{-r(\T)}, \qquad \lambda_3 \geq \frac{m}{1+\eta}\big (\eta e^{-r(\T)}+ e^{-h(\T)}\big).
\]
The elements $v_2(i)$ of the eigenvector that corresponds to $\lambda_2$ satisfy $|v_2(i)| \geq \sqrt{\frac{1}{m\eta}}$.
\end{lemma}

\begin{proof}[Proof of Lemma \ref{lem:davis_khan_for_trees_mc}]
Let $L$ and $\hat L$ be two symmetric matrices and let $v_i$ and $\hat v_i$ be their $i$-th eigenvectors, respectively. A variant of the Davis-Kahan theorem for perturbation of eigenvectors (see Theorem 2 of \cite{yu2015useful}) gives 
\begin{equation}\label{eq:davis_kahan}
\|v_i-\hat v_i\| \leq 2^{3/2} \frac{\|L-\hat L\|}{\gamma_i}.    
\end{equation}
where $\gamma_i =  \min\{|\lambda_i-\lambda_{i+1}|,|\lambda_i-\lambda_{i-1}|\} $ is the  eigengap. We apply the theorem to the Laplacian matrix $L=D-S$ (see Definition \ref{def:laplacian}), and its Fiedler vector $v_2$. 
The spectral norm $\|L-\hat L\|$ can be bounded by,
\begin{align}\label{eq:bound_error_laplacian}
\|L-\hat L\| &\leq \|D-\hat D\|+\|S-\hat S\| = \max_i \Big|\sum_k (S_{ik}-\hat S_{ik}) \Big| + \|S-\hat S\|  \notag   \\
&\leq  \max_i \sum_k  |S_{ik}-\hat S_{ik}| + \|S-\hat S\|   
\leq (\sqrt{m}+1)\|S-\hat S\|.
\end{align}
Substituting \eqref{eq:bound_error_laplacian} into \eqref{eq:davis_kahan} yields
\begin{equation}\label{eq:davis_kahan_S}
\|v_2 - \hat v_2\| \leq 2^{3/2} (\sqrt{m}+1) \frac{\|S-\hat S\|}{\gamma_2}.    
\end{equation}
From Lemma \ref{lem:symmetric_spectrum}  it follows that the spectral gap $\gamma_2$ is bounded by,
\begin{equation}\label{eq:spectral_gap}
\gamma_2 = \min(\lambda_2 - \lambda_1, \lambda_3 - \lambda_2) \geq  
m e^{-r(\T)} \min\left\{1,\frac{1}{1+\eta}\big(e^{r(\T)-h(\T)}-1\big)\right\}. 
\end{equation}
Combining Eqs. \eqref{eq:davis_kahan_S} and \eqref{eq:spectral_gap} proves that if
\begin{equation}\label{eq:sufficient_similarity}
\|S-\hat S\| \leq 
\frac{\sqrt{m}e^{-r(\T)}}{\sqrt{\eta} 2^{3/2}(\sqrt{m}+1)}\min\left\{1,\frac{1}{1+\eta}\big(e^{r(\T)-h(\T)}-1\big)\right\}
\end{equation}
then $\|v_2- \hat v_2\|<1/\sqrt{\eta m}$, which implies $\|v_2- \hat v_2\|_\infty<1/\sqrt{\eta m}$.
Thus, by Lemma \ref{lem:symmetric_spectrum} $\text{sign}(v_i)=\text{sign}(\hat v_i)$ for each $i \in [m]$. Hence, partitioning the terminal nodes according to $\text{sign}(\hat v_2)$ or $\text{sign}(v_2)$ yield the same result. As we proved in Theorem \ref{thm:consistenct_partitioning}, the resulting subsets are clans of the tree.
\end{proof}

Next, we prove Theorem \ref{thm:finite_split} under the additional assumption of the Jukes-Cantor model. The theorem is proved
by combining Lemma \ref{lem:davis_khan_for_trees_mc} with a concentration bound on the similarity matrix estimate $\hat S$, derived in \cite{jaffe2020spectral}.

\begin{proof}[Proof of Theorem \ref{thm:finite_split}]
From Lemma 4.7 of \cite{jaffe2020spectral}, under the JC model of evolution, 
\[
P\left(\|\hat S - S\| \leq t\right) \geq 1-2m^2 \mbox{exp}\left(-\frac{2nt^2}{\ell^2m^2}\right).
\]
Setting $t$ to the right hand side of \eqref{eq:sufficient_similarity} yields that if
\[
n \geq 4\ln \left(\frac{2m^2}{\epsilon}\right)\eta \ell^2 m(\sqrt{m}+1)^2e^{2r(t)}\max\left\{1,\frac{(1+\eta)^2}{(e^{r(\T)-h(\T)}-1)^2}\right\},
\]
the requirements of Lemma \ref{lem:davis_khan_for_trees_mc} are satisfied with probability at  least $1 - \eps$, which concludes the proof.
\end{proof}

\subsection{Merging step of STDR} \label{sec:finite_sample_mearging}
We derive finite sample bounds for the merging step of STDR. In contrast to the partitioning step, the guarantees for the merging step, presented in the following theorem, hold for any tree topology.

\begin{theorem}\label{thm:finite_merging}
Let $\T$ be a tree with $m$ terminal nodes, which consists of two subtrees $\T_1,\T_2$ with terminal nodes $C_1$ and $C_2$, respectively. Let $\{A,B\}$ be the partition of $C_1$ induced by the correct placeholder edge $e^\ast$, and let $\D = \min\{\|S(A,B)\|_F, \|S(C_1,C_2)\|_F\}$.
For any $\eps>0$, if the number of samples $n$ satisfies
\begin{equation}\label{eq:merging_sample_bound}
n \geq 8\ell^2m^3\left(\frac{2}{\D}+\frac{2.5}{\D^2} +  \frac{1+10\sqrt{2}}{\D^3}  \right)^2\left(\frac{ \xi^4}{\delta^6 (1-\xi^2)^2}\right)\log\left(\frac{2m^2}{\eps}\right),
\end{equation}
then STDR finds the correct placeholder edge in $\T_1$
with probability at least $1-\eps$.
\end{theorem}
From Eq. \eqref{eq:merging_sample_bound}, if $\D\gg 1$ then the required number of samples is  $\widetilde O(m^3/\D^2)$. Assuming that the lower bound on the similarity between adjacent nodes $\delta$ is close to $1$, the value of $\D$ depends mainly on the size of the two submatrices $S(A,B)$ and $S(C_1,C_2)$. 
This analysis has important implications on the choice of the smallest partition $\tau$ in Algorithm \ref{alg:description}.
The number of samples in Eq. \eqref{eq:merging_sample_bound} is $\widetilde O(m)$ if $A,B$ and $C_2$ are of size $O(m)$, but is $\widetilde O(m^3)$ if $A$ and $B$ or $C_2$ are of size $O(1)$.  Thus on the one hand, reducing $\tau$ results in smaller subsets of terminal nodes, which improves the runtime of the reconstruction step of STDR. On the other hand it may affect the accuracy of the merging step. Figure \ref{fig:threshold} shows both runtime and accuracy of STDR as a function of the threshold parameter $\tau$, when applying STDR with RAxML or SNJ as its subroutine. The data consists of $n=1000$ samples generated from a binary symmetric tree with $m=2048$ terminal nodes. The accuracy of the algorithm degrades for small values of $\tau$ while the runtime improves by approximately half an order of magnitude.

Our proof of Theorem \ref{thm:finite_merging} consists of three steps: (i) In Lemma  \ref{thm:merge_incorrect_lower_bound} we derive a lower bound on the score $d(e)$ of an edge $e$ that is not the correct placeholder edge. (ii) Lemma \ref{lem:guarantee_s} provides a sufficient condition on the accuracy of the similarity matrix estimate $\hat S$ that guarantees the merging step will yield the correct placeholder edge. (iii) For the JC model, we derive an expression for the number of samples required for the condition in Lemma \ref{lem:guarantee_s} to hold with high probability. 

\paragraph{Step 1: A lower bound on the score $d(e)$ for incorrect edges}
In Section \ref{sec:consistency_merging}, we showed that $d(e)=0$ if and only if $e$ is the correct placeholder edge. Here, for the exact similarity matrix $S$ we derive a lower bound on $d(e)$, if $e$ is an incorrect placeholder edge in $\T_1$. 
\begin{lemma}\label{thm:merge_incorrect_lower_bound}
Let $\T$ be a tree that consists of two subtrees $\T_1,\T_2$, and let $e \in \T_1$ be an edge that is not the correct placeholder edge. Then,
\[
d(e) \geq 
    \begin{cases}
        \frac{(\sqrt{2}\delta)^{\log m} \delta^2 (1-\xi^2)}{2\sqrt{m} \xi^2} & \delta^2 \leq 0.5 \\
        \frac{\delta^3 (1-\xi^2)}{\sqrt{2m} \xi^2} & \delta^2 > 0.5.
    \end{cases}
\]
\end{lemma}

For the proof of Lemma  \ref{thm:merge_incorrect_lower_bound}, we introduce new notation, illustrated in Figure~\ref{fig:incorrect_placeholder_edge}. The sets of terminal nodes of $\T_1$ and $\T_2$ are denoted by $C_1$ and $C_2$, respectively. 
We denote by $e^\ast \in \T_1$  the correct placeholder edge, and by $e \in \T_1$ an arbitrary incorrect placeholder edge. The edge $e$ splits the terminal nodes of $\T_1$ into $A$ and $B$ and has endpoints $h_A$ and $h_B$.  
We denote by $h_0,\ldots,h_{N}$ the non terminal nodes on the path between the root node of $\T_1$, denoted $h_0$, and $h_A=h_N$.  
We partition the terminal nodes in $A$ to $N+1$ subsets $A_0,\ldots,A_N$ according to $h_0,\ldots,h_N$ as
follows: Every node in $A$ is assigned to the closest non terminal node on the path between $h_0,\ldots,h_N$.
In the proof of Lemma  \ref{thm:merge_incorrect_lower_bound}, we use the following auxiliary lemma, proven in appendix \ref{appendix:proof_of_incorrect_merge_lemma}.
\begin{lemma}\label{lem:merge_lower_bound_single_pair}
Let $R_i= S(h_0,h_i)^2$. For any $1\leq i\leq N-1$ and $1 \leq k \leq (N-i)$ we have
\begin{equation}\label{eq:merge_lemma_two_terms}
\begingroup 
\setlength\arraycolsep{1pt}
     \min_\beta \frac{(1-\beta {R_i})^2 \|S(A_i,B)\|^2_F + (1-\beta {R}_{i+k})^2 \|S(A_{i+k},B)\|^2_F }{\|S(A_{i},B)\|^2_F + \|S(A_{i+k},B)\|^2_F} \geq \begin{cases}
        \frac{(2\delta^2)^{\log m} \delta^{2(k+1)} (1-\xi^2)^2}{4m \xi^4} & \delta^2 \leq 0.5 \\
        \frac{\delta^{2(k+2)} (1-\xi^2)^2}{2m \xi^4} & \delta^2 > 0.5
    \end{cases}
    \endgroup
\end{equation}
\end{lemma}

\begin{proof}[Proof of Lemma  \ref{thm:merge_incorrect_lower_bound} ]
The proof consists of the following steps: (i) we rewrite the score $d(e)$ defined in Eq. \eqref{eq:edge_score} in terms of $\|S(A_0,B)\|_F,\ldots \|S(A_N,B)\|_F$. The new expression is given in  Eq. \eqref{eq:merge_score_general}. (ii) In Eq. \eqref{eq:merge_lower_bound_pairs} we derive a lower bound on $d(e)$ in terms of two consecutive terms $\|S(A_i,B)\|_F$ and $\|S(A_{i+1},B)\|_F$.  (iii) In Lemma \ref{lem:merge_lower_bound_single_pair} we combine Eq. \eqref{eq:merge_lower_bound_pairs} with a bound on $\|S(A_i,B)\|_F$ and $\|S(A_{i+1},B)\|_F$ to conclude the proof.  
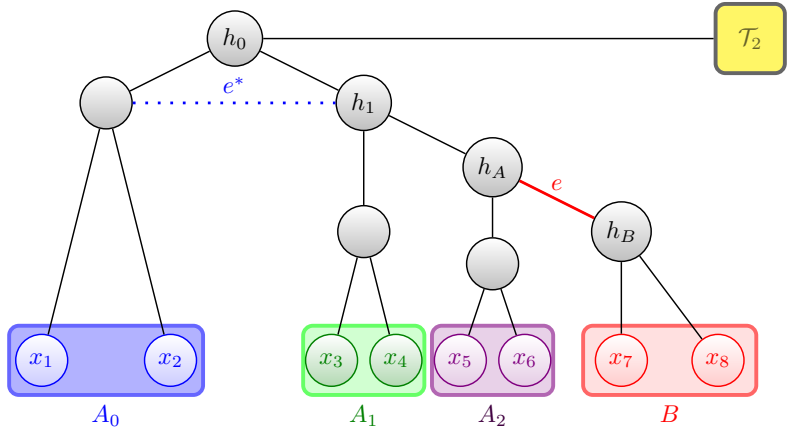
\begin{figure}[t]
  \centering
  \resizebox{0.7\textwidth}{!}
  {
      \begin {tikzpicture}[-latex ,auto ,node distance =4 cm and 5cm ,on grid ,
    	semithick ,scale=0.93,
    	state/.style ={ circle ,top color =white, bottom color=lightgray,
    		draw,black , text=black , minimum width = 0.75 cm}]	
    	
    	\node[rectangle, ultra thick, draw=blue, minimum width=2.8cm, minimum height=1cm, fill=blue!50, opacity=0.6, rounded corners, label=below:\textcolor{blue}{$A_0$}] (A1) at (2,1) {};
    	
    	\node[rectangle, ultra thick, draw=green, minimum width=1.75cm, minimum height=1cm, fill=green!30, opacity=0.6, rounded corners, label=below:\textcolor{green!50!black}{$A_1$}] (A2) at (6,1) {};
    	
    	\node[rectangle, ultra thick, draw=violet, minimum width=1.75cm, minimum height=1cm, fill=violet!30, opacity=0.6, rounded corners, label=below:\textcolor{violet!50!black}{$A_2$}] (A2) at (8,1) {};
    	
    	\node[rectangle, ultra thick, draw=red, minimum width=2.5cm, minimum height=1cm, fill=red!20, opacity=0.6, rounded corners, label=below:\textcolor{red}{$B$}] (B) at (10.75,1) {};
    		
    	\node[state] (h0) at (4,6) {$h_0$};
    	\node[state] (hAs) at (2,5) {};
    	\node[state] (h1) at (6,5) {$h_1$};
    	\node[state] (hA) at (8,4) {$h_A$};
    	\node[state] (hB) at (10,3) {$h_B$};
    	
    	\node[state] (h2) at (6,3) {};
    	\node[state] (h3) at (8,2.5) {};
    	
    	
    	\node[state, bottom color = blue!20, draw, blue, text=blue] (x1) at (1,1) {$x_1$};
    	\node[state, bottom color = blue!20, draw, blue, text=blue] (x2) at (3,1) {$x_2$};
    	
    	\node[state, bottom color = green!50!black!30, draw, green!50!black, text=green!50!black] (x3) at (5.5,1) {$x_3$};
    	\node[state, bottom color = green!50!black!30, draw, green!50!black, text=green!50!black] (x4) at (6.5,1) {$x_4$};
    	
    	\node[state, bottom color = violet!20, draw, violet, text=violet] (x5) at (7.5,1) {$x_5$};
    	\node[state, bottom color = violet!20, draw, violet, text=violet] (x6) at (8.5,1) {$x_6$};
    	
    	\node[state, bottom color = red!20, draw, red, text=red] (x7) at (10,1) {$x_7$};
    	\node[state, bottom color = red!20, draw, red, text=red] (x8) at (11.5,1) {$x_{8}$};
    	
    	\path[-] (h0) edge node [above =0.15 cm,left = 0.15cm] {}(hAs);
    	\path[-] (h0) edge node [above =0.15 cm,left = 0.15cm] {}(h1);
    	\path[-] (h1) edge node [above =0.15 cm,left = 0.15cm] {}(hA);
    	\path[-] (hA) edge node [above =0.15 cm,left = 0.15cm] {}(hB);
    	
    	\path[-] (h1) edge node [above =0.15 cm,left = 0.15cm] {}(h2);
    	\path[-] (hA) edge node [above =0.15 cm,left = 0.15cm] {}(h3);
    	\path[-] (hAs) edge node [above =0.15 cm,left = 0.15cm] {}(x1);
    	\path[-] (hAs) edge node [above =0.15 cm,left = 0.15cm] {}(x2);

    
    	\path[-] (h2) edge node [above =0.15 cm,left = 0.15cm] {}(x3);	
    	\path[-] (h2) edge node [above =0.15 cm,left = 0.15cm] {}(x4);
    	\path[-] (h3) edge node [above =0.15 cm,left = 0.15cm] {}(x5);	
    	\path[-] (h3) edge node [above =0.15 cm,left = 0.15cm] {}(x6);
    	
    	\path[-] (hB) edge node [above =0.15 cm,left = 0.15cm] {}(x7);
    	\path[-] (hB) edge node [above =0.15 cm,left = 0.15cm] {}(x8);
    	
    	\path[-, line width = 0.04 cm, color = red] (hA) edge node [above] {$e$}(hB);
    	\path[draw,loosely dotted,-, line width = 0.04 cm, color = blue] (hAs) edge node [above] {$e^\ast$} (h1);
    	
    	\node[rectangle, ultra thick, draw=black, minimum width=1cm, minimum height=1cm, fill=yellow, opacity=0.6, rounded corners] (T2) at (12,6) {$\T_2$};
    	\path[-] (h0) edge node [above =0.15 cm,left = 0.15cm] {}(T2);
    	
    	\end{tikzpicture}
	}
  \caption{Bounding the score $d(e)$ for an incorrect placeholder edge in $\T_1$.  The correct placeholder edge $e^\ast \in \T_1$ is marked by a dotted blue line. The incorrect placeholder edge $e$, which partitions the terminal node to subsets $A(e)$ and $B(e)$, is marked by a thick red line. The two non-terminal nodes on the path between the correct and incorrect edges are denoted by $h_1,h_2=h_A$, and the root node of $C_1$ is denoted by $h_0$. The subset of terminal nodes closest to $h_i$ is denoted by $A_i$.}
  \label{fig:incorrect_placeholder_edge}
\end{figure}

First, we express the numerator of $d(e)$ in Eq. \eqref{eq:edge_score} in terms of $S(A_0,B),\ldots, S(A_N,B)$.
Since $h_0$ separates $C_1$ and $C_2$, by the multiplicative property of the similarity we have, 
\[
S(C_1,C_2) = S(C_1,h_0)S(h_0,C_2) = u \sigma v^T  \qquad \|u\| = \|v\| = 1.
\]
Let $\bar \beta$ be the proportionality constant between $u$ and $S(C_1,h_0)$ such that $u = \bar \beta S(C_1,h_0)$. Recall that  $u_A,u_B$ in Eq. \eqref{eq:edge_score} are  the entries in $u$ that correspond to $A$ and $B$, respectively. Partitioning $u$ into $u_A$ and $u_B$ and partitioning $S(C_1,h_0)$ into $S(A,h_0)$ and $S(B,h_0$) gives
\[
u_A =  \bar{\beta}S(A,h_0)   \qquad u_B =  \bar{\beta} S(B,h_0).
\]
It follows that 
\[
    S(A,B) - u_A \alpha u_B^T = S(A,B) - \alpha \bar{\beta}^2 S(A,h_0)S(h_0,B) = S(A,B) - \beta S(A,h_0)S(h_0,B),
\]
where $\beta = \bar{\beta}^2\alpha$. 
We split $S(A,B)$ into the submatrices $S(A_0,B), S(A_1,B), ..., S(A_N,B)$. Similarly, we split $S(A,h_0)$ into the components $S(A_0,h_0), S(A_1,h_0), ..., S(A_N,h_0)$. This gives 
\begin{equation}\label{eq:S_AB-beta}
    S(A,B) - \beta S(A,h_0)S(h_0,B) = \begin{bmatrix}
        S(A_0,B)\\
        S(A_1,B)\\
        \vdots\\
        S(A_N,B)
    \end{bmatrix} - \beta 
    \begin{bmatrix}
        S(A_0,h_0)\\
        S(A_1,h_0)\\
        \vdots\\
        S(A_N,h_0)
    \end{bmatrix}S(h_0,B).
\end{equation}
Let $R_i = S(h_0,h_i)^2$. We show that the matrix $S(A_i,h_0)S(h_0,B)$, which appears on the right side of Eq. \eqref{eq:S_AB-beta}, is proportional to $S(A_i,B)$ with the proportionality constant $R_i$.
\begin{align}\label{eq:sim_a2_b}
    R_i S(A_i,B) &= S(h_0,h_i)^2 S(A_i,h_i) S(h_i,B) \notag \\ &=S(A_i,h_i) S(h_i,h_0) ~  S(h_0,h_i) S(h_i,B) =  S(A_i,h_0)S(h_0,B). 
\end{align}
Inserting \eqref{eq:sim_a2_b} into \eqref{eq:S_AB-beta} gives
\begin{equation*}
    \begin{bmatrix}
        S(A_0,B)\\
        S(A_1,B)\\
        \vdots\\
        S(A_N,B)
    \end{bmatrix} - \beta 
    \begin{bmatrix}
        {R_0} S(A_0,B)\\
        {R_1} S(A_1,B)\\
        \vdots\\
        {R_N} S(A_N,B)
    \end{bmatrix} = 
    \begin{bmatrix}
        (1-\beta {R_0})S(A_0,B)\\
        (1-\beta {R_1}) S(A_1,B)\\
        \vdots\\
        (1-\beta {R_N}) S(A_N,B)
    \end{bmatrix}.
\end{equation*}
Thus, the score in Eq.~\eqref{eq:edge_score} is equivalent to
\begin{equation}\label{eq:merge_score_deriv}
    d^2(e) = 
    \frac{1}{\|S(A,B)\|^2_F} \min_{\beta} \sum_{i=0}^N (1-\beta {R_i})^2 \|S(A_i,B)\|^2_F.
\end{equation}
Since $\|S(A,B)\|^2_F = \sum_{i=0}^N \|S(A_i,B)\|^2_F$, we can rewrite Eq. \eqref{eq:merge_score_deriv} as follows,
\begin{equation}\label{eq:merge_score_general}
    d^2(e) = \min_\beta \frac{\sum_{i=0}^N (1-\beta {R_i})^2 \|S(A_i,B)\|^2_F}{\sum_{i=0}^N \|S(A_i,B)\|^2_F}.
\end{equation}
Next, the following lemma, proven in Appendix \ref{appendix:proof_of_incorrect_merge_lemma}, bounds the ratio of two sums. 
\begin{lemma}\label{lem:ratio_sum_inequality}
For two series of positive numbers $a_i,b_i > 0$ we have
\[
\frac{\sum a_i}{\sum b_i} \geq \min_{i\neq j; |i-j|\leq 2} \frac{a_i+a_j}{b_i+b_j}.
\]
\end{lemma}
Applying the lemma to Eq. \eqref{eq:merge_score_general} yields
\begin{equation}\label{eq:merge_lower_bound_pairs}
d^2(e) \geq \min_{0 \leq i \leq N-1; k\in \{1,2\}} \min_{\beta} \frac{(1-\beta {R_i})^2 \|S(A_i,B)\|^2_F + (1-\beta {R}_{i+k})^2 \|S(A_{i+k},B)\|^2_F }{\|S(A_{i},B)\|^2_F + \|S(A_{i+k},B)\|^2_F}.
\end{equation}
Combining Eq. \eqref{eq:merge_lower_bound_pairs} and Lemma~\ref{lem:merge_lower_bound_single_pair} gives
\[
d^2(e) \geq 
    \begin{cases}
        \frac{(2\delta^2)^{\log m} \delta^6 (1-\xi^2)^2}{4m \xi^4} & \delta^2 \leq 0.5 \\[10pt]
        \frac{\delta^8 (1-\xi^2)^2}{2m \xi^4} & \delta^2 > 0.5,
    \end{cases}
\]
which concludes the proof of Lemma \ref{thm:merge_incorrect_lower_bound}, and with it, Step 1 in the proof of Theorem \ref{thm:finite_merging}.
\end{proof}

\paragraph{Step 2: A sufficient condition on the estimate $\hat S$}

Lemma  \ref{thm:merge_incorrect_lower_bound} shows that there is a gap between the score of the correct placeholder edge and the scores of all other edges in $\T_1$. 
In the following lemma we show that if $\hat S$ is sufficiently close to $S$ the gap is preserved and STDR selects the correct placeholder edge.
For simplicity, we address only the case $\delta^2 \geq 0.5$. 

\begin{lemma}\label{lem:guarantee_s}
Let $\D = \min\{\|S(A,B)\|_F, \|S(C_1,C_2)\|_F\}$.
If the similarity matrix estimate $\hat S$ satisfies
\begin{equation}\label{eq:S_diff_orig_bound}
\|S-\hat S\|_F \leq \frac{1}{2}\left(\frac{2}{\D}+\frac{2.5}{\D^2} +  \frac{1+10\sqrt{2}}{\D^3}  \right)^{-1}
        \frac{\delta^3 (1-\xi^2)}{\sqrt{2m} \xi^2}, 
\end{equation}
then STDR selects the correct placeholder edge.
\end{lemma}
In our proof, we use the following auxiliary lemma, proven in Appendix \ref{appendix:proof_of_incorrect_merge_lemma}.
\begin{lemma}\label{lem:d_err_bound}
Let $d(e),\hat d(e)$ be the exact and estimated score functions. If $\|S- \hat S\|_F \leq \D/2$,
then 
\[
|d(e) - \hat d(e)| \leq \|S- \hat S\|_F \left(\frac{2}{\D}+\frac{2.5}{\D^2} +  \frac{1+10\sqrt{2}}{\D^3}  \right).
\]
\end{lemma}
\begin{proof}[Proof of Lemma \ref{lem:guarantee_s}]
Suppose $e^* \in \T_1$ is the correct placeholder edge and $e' \neq e^*$ is a different edge in $\T_1$. 
By Lemma  \ref{thm:merge_incorrect_lower_bound} 
\begin{equation*}
d(e') \geq \frac{\delta^3(1-\xi^2)}{\sqrt{2m}\xi^2},  \end{equation*}
while for the correct edge $d(e^\ast)=0$. It follows from the triangle inequality that if
\begin{equation}\label{eq:guarantee_d}
|d(e)-\hat d(e)| \leq \frac12 \frac{\delta^3(1-\xi^2)}{\sqrt{2m}\xi^2},
\end{equation}
for all edges $e$, then $\hat d(e^\ast)< \hat d(e')$.
Since $\delta \leq \xi <1$ and $m \geq 2$, 
\[
\frac{1}{2}\left(\frac{2}{\D}+\frac{2.5}{\D^2} +  \frac{1+10\sqrt{2}}{\D^3}  \right)^{-1}
 \frac{\delta^3 (1-\xi^2)}{\sqrt{2m} \xi^2} \leq \frac{\D}{2}.
\]
Thus, if the estimate $\hat S$ satisfies  Eq. \eqref{eq:S_diff_orig_bound}, then $\|S-\hat S\| \leq \D/2$ and the condition for Lemma \ref{lem:d_err_bound} holds. Combining the lemma with Eq. \eqref{eq:guarantee_d} concludes the proof.  
\end{proof}

\paragraph{Step 3: Finite sample guarantees} We are now ready to prove Theorem \ref{thm:finite_merging}, which bounds the number of samples required to compute, with high probability, a sufficiently accurate estimate $\hat S$, as determined in Lemma \ref{lem:guarantee_s}.

\begin{proof}[Proof of Theorem \ref{thm:finite_merging}]
The following concentration bound for $\hat S$ was derived in Lemma 4.7 of  \cite{jaffe2020spectral},
\[
\Pr(\|\hat S-S\|_F\leq t) \geq 1- 2m^2 \exp \left(-\frac{2nt^2}{\ell^2m^2}\right).
\]
We note that in \cite{jaffe2020spectral}, this bound was presented for the spectral norm, but the proof holds for the Frobenius norm as well.
Suppose that $\Pr(\|\hat S-S\|_F\leq t)>1-\eps$, Namely
\begin{equation} \label{eq:n_lower_bound}
    n\geq \frac{\ell^2m^2}{2t^2}\log\left(\frac{2m^2}{\eps}\right).
\end{equation}
By Lemma \ref{lem:guarantee_s}, a sufficient condition for STDR to select the correct placeholder edge is
\begin{equation}\label{eq:S-S_hat_bound_req}
\|S-\hat S\|_F \leq \frac{1}{2}\left(\frac{2}{\D}+\frac{2.5}{\D^2} +  \frac{1+10\sqrt{2}}{\D^3}  \right)^{-1}
        \frac{ \delta^3 (1-\xi^2)}{\sqrt{2m} \xi^2}.
\end{equation}
Setting $t$ to the right hand side of Eq. \eqref{eq:S-S_hat_bound_req} and substituting into Eq. \eqref{eq:n_lower_bound}, we have that if
\begin{align*}
n &\geq \frac{\ell^2m^2}{2}2^2\left(\frac{2}{\D}+\frac{2.5}{\D^2} +  \frac{1+10\sqrt{2}}{\D^3}  \right)^2\left(\frac{2m \xi^4}{\delta^6 (1-\xi^2)^2}\right)\log\left(\frac{2m^2}{\eps}\right) \notag 
\\
& = 8\ell^2m^3\left(\frac{2}{\D}+\frac{2.5}{\D^2} +  \frac{1+10\sqrt{2}}{\D^3}  \right)^2\left(\frac{ \xi^4}{\delta^6 (1-\xi^2)^2}\right)\log\left(\frac{2m^2}{\eps}\right)
\end{align*}
then Eq. \eqref{eq:S-S_hat_bound_req} holds with probability at least $1 - \eps$, and thus the merging step in STDR selects the correct placeholder edge with high probability. 
\end{proof}

\begin{remark}\label{rem:one_step_to_many}
The guarantees in Theorems  \ref{thm:finite_split} and \ref{thm:finite_merging} are derived for a single partitioning and merging step. Since the algorithm is recursive, additional splitting and merging steps depend on submatrices of $\hat S$. If the bounds in Lemmas \ref{lem:davis_khan_for_trees_mc} and \ref{lem:guarantee_s} are satisfied for the full matrix $S$, they hold simultaneously for all submatrices of $S$ as well. Thus, the number of samples required in Theorems \ref{thm:finite_split} and \ref{thm:finite_merging} is sufficient to guarantee with high probability the success of STDR for multiple partitioning and merging steps.  
\end{remark}

\subsection{Comparison of sample complexity} \label{sec:finite_sample_comperison}
Combining Theorem \ref{thm:finite_split} and Theorem \ref{thm:finite_merging}, for a binary symmetric tree with a fixed similarity between adjacent nodes, the sample complexity of the partitioning and merging steps of STDR is 
\begin{equation} \label{eq:sample_complexity}
\wtilde{\OO}\left(m^3/\D_{min}^2+ m^{2+4\log_2(\frac{1}{\delta})}\right)  
\end{equation}
where $\D_{min}$ is the minimum value of $\D$ from Theorem \ref{thm:finite_merging} over all partitions of the tree. We compare this result to three other methods for full recovery of trees. For simplicity, we assume that the similarity between all adjacent nodes is $\delta$. Thus, the value of $D_{min} = \wtilde{\OO}(m^3/\delta \tau^2)$, where $\tau$ is the user given threshold. For a reasonalbe setting where $\tau = \wtilde{\OO}(m^3/\delta \tau^2)$, the sample complexity simplifies to $\wtilde \OO(m^{2+4\log_2(\frac{1}{\delta})})$. 
For NJ, the sample complexity given in Section 3.3 of \cite{atteson1999performance} is $\wtilde{\OO}(\exp(-4\min_{i,j}\ln(S(x_i,x_j)))$ or equivalently $\wtilde{\OO}\left(\delta^{-4 \text{diam}(\T)}\right)$, where $\text{diam}(\T)$ is the diameter of $\T$. For a binary symmetric tree $\text{diam}(\T)=2\log_2(m)$ and hence the complexity is $O(m^{8\log_2(1/\delta)})$, which is better than \eqref{eq:sample_complexity} for $\delta$ close to one, but worse for lower values of $\delta$. However, the diameter of the tree can be as large as $m$, in which case the sample complexity of NJ is exponential in $m$, rather than polynomial as in \eqref{eq:sample_complexity}. 

For SNJ, if $\delta^2>0.5$ the sample 
complexity is $\wtilde{\OO}(m^2)$ 
(by Theorem 4.3 in \cite{jaffe2020spectral}). This is similar 
to \eqref{eq:sample_complexity} for $\delta$ close to one, but improves upon \eqref{eq:sample_complexity} as $\delta$ decreases. 

For the Dyadic Close method \cite[Theorem 9]{erdHos1999few}, the sample complexity is $\wtilde{\OO}((1/\delta)^{\text{4h}(\T)})$, where recall that $h(\T)$ denotes the depth of a tree as in definition \ref{def:depth}. 
For a binary symmetric tree $h(\T) = \log_2(m)$ in which case the complexity is $\OO(m^{4\log_2(1/\delta)})$, which improves upon Eq. \eqref{eq:sample_complexity} by $m^2$. For highly imbalanced trees $\text{depth}(\T)=\OO(1)$ in which case the sample complexity is logarithmic in $m$. The improved sample complexity, however, comes at cost of a $\OO(m^5)$ computational complexity.
Thus, excluding the Dyadic Closure method, the sample complexity of STDR is similar to several other distance-based methods with theoretical guarantees.






\section{Simulation Results}\label{sec:experiments}

We illustrate the performance of  STDR in comparison to several other algorithms in a variety of simulated settings. To this end we generated trees according to the coalescent model (\ref{subsec:experiment_kingman}) and the birth-death model  (\ref{subsec:experiment_birthDeath}), which are common in phylogenetics. In addition, we also considered the challenging scenario of the caterpillar tree.
In all experiments, the sequences were generated according to the HKY substitution model \cite{hasegawa1985dating} with transition-transversion ratio of 2, a typical value in the human genome \cite{keller2007transition}. The mutation rate for the HKY model is specified for each simulation. 

We considered the following reconstruction methods:  (i) RAxML~\cite{stamatakis2014raxml}, a standard tool for maximum likelihood-based tree inference, (ii) neighbor joining (NJ), and (iii) spectral neighbor joining (SNJ). 
Recall that STDR requires as input a subroutine \textit{alg} for the reconstruction of the small trees. Thus, for comparison, we applied STDR with each of the aforementioned algorithms as the subroutine. We denote these three methods as
(iv) STDR + RAxML, (v) STR + NJ and (vi) STDR + SNJ. 
A second input to STDR is the threshold parameter $\tau$, which sets an upper bound for the size of the small trees. This parameter is specified in the description of each experiment.
The accuracy of the different algorithms is measured by the normalized Robinson-Foulds (RF) distance, defined  as the RF distance \cite{estabrook1985comparison} between the reconstructed and reference tree divided by $2m-6$. Each experiment was repeated $5$ times to obtain a mean and standard deviation of the performance and runtime of each method.

In addition to the above experiments,  we compare our merging procedure to TreeMerge \cite{molloy2019treemerge}. The results for the caterpillar tree and the comparison to TreeMerge are shown in Appendix \ref{appendix:experimental_results}. Finally, for a symmetric binary tree, we demonstrate how changes in the threshold $\tau$ affect the results of STDR.  

\paragraph{Implementation remarks}
To improve the results of STDR, we computed two possible partitions $C_1,C_2$: (i) A partition that corresponds to a threshold at $0$ in the Fiedler vector, and (ii) a partition that corresponds to the largest gap. In practice, the partition was chosen by method (i) or (ii), as the one that minimizes the second singular value of $S(C_1,C_2)$, see Lemma \ref{lem:affinity_spectral}.
To improve runtime, we apply randomized methods for computing leading singular values and vectors, see \cite{stewart2001matrix,halko2011finding,aizenbud2019matrix}.

\subsection{Kingman's coalescent model}\label{subsec:experiment_kingman}

We generated a random tree according to Kingman's coalescent model \cite{wakeley2009coalescent} with $m=2000$ terminal nodes (See example in Fig \ref{fig:kingman_tree}).
Figure \ref{fig:kingman_performance} shows the accuracy (left panel) and the runtime (right panel) of the different methods as functions of the sequence length. The threshold parameter $\tau$ was set to $128$ for all experiments. Here, STDR+RAxML performs similarly to RAxML in accuracy while achieving more than an order-of-magnitude reduction in runtime. Compared to NJ and SNJ, STDR+NJ and STDR+SNJ show improvement in both accuracy and runtime.

\begin{figure}[t]
    \centering
    \includegraphics[width = \textwidth]{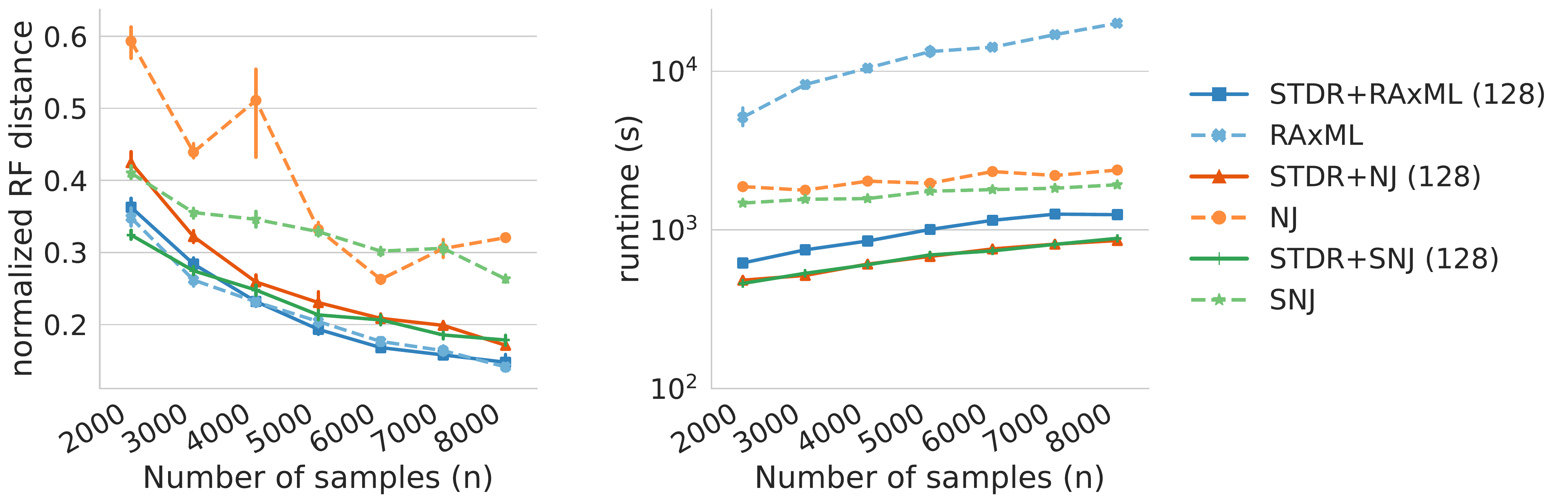}
    \caption{Trees generated according to  Kingman's coalescent model with $m=2000$ terminal nodes. The mean and standard deviation of the normalized RF distance (left) between the reconstructed tree and the input tree and of the runtime (right) are shown for each method over 5 independent runs. }
    \label{fig:kingman_performance}
\end{figure}

\subsection{Birth-death model}\label{subsec:experiment_birthDeath}
We generated random binary trees with $m=2048$ terminal nodes according to the birth-death model \cite{birthdeath} .
 The STDR threshold was set to $\tau = 256$ for all three methods.
Figure \ref{fig:birth_death_performance} shows the accuracy and runtime of the different methods as a function of the sequence length $n$. Using STDR with NJ clearly improves upon the performance of standard NJ both in terms of accuracy and runtime. 
Compared to SNJ and RAxML, STDR+SNJ and STDR+RAxML show similar accuracy but   with significantly faster runtimes


\begin{figure}[t]
    \centering
    \includegraphics[width = \textwidth]{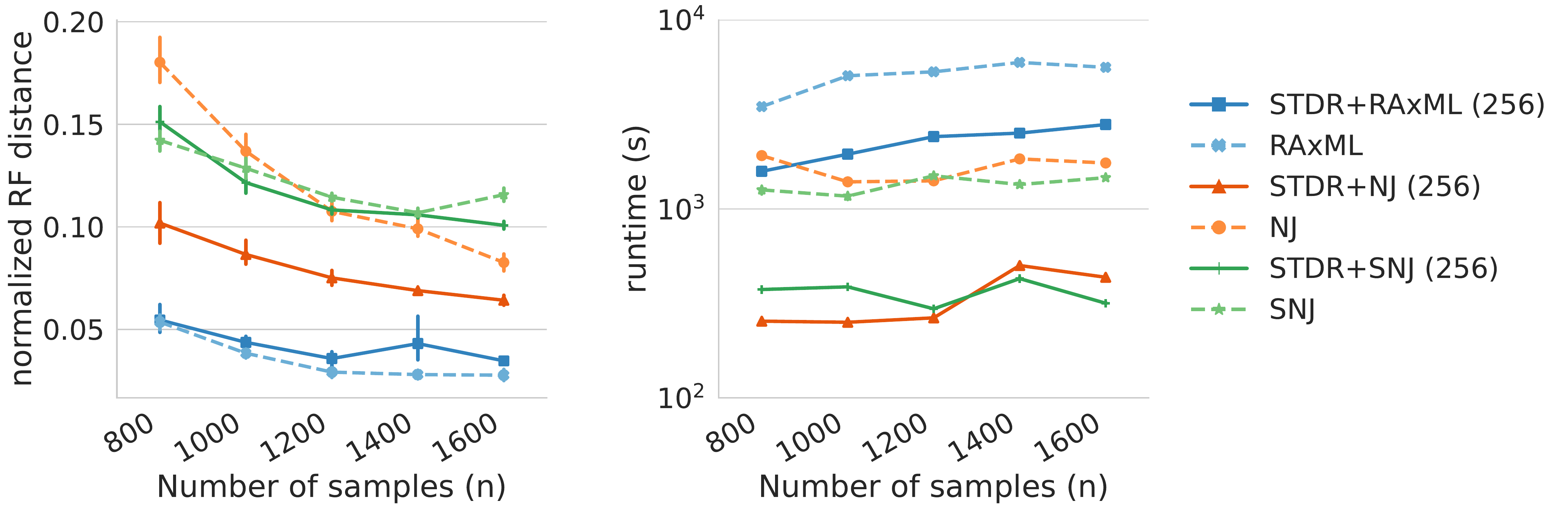}
    \caption{A birth-death tree with $m=2048$ terminal nodes. The mean and standard deviation of the normalized RF distance (left) between the reconstructed tree and the input tree and of the runtime (right) are shown for each method over 5 independent runs. }
    \label{fig:birth_death_performance}
\end{figure}

\subsection{Effect of threshold parameter}
Our aim in this experiment was to test the impact of the threshold parameter $\tau$ on the performance of STDR.
To that end, we created a binary symmetric tree with $m=2048$ terminal nodes and similarity between all adjacent nodes equal to $\delta =  0.65$. The number of samples was set to $n=1000$. We then reconstructed the tree via STDR with different subroutines and a range of threshold values.

Figure \ref{fig:threshold} shows the normalized RF distance between the recovered trees and the ground truth tree as a function of the threshold. For both RAxML and SNJ, accuracy slightly improves for higher values of the threshold. STDR + NJ is not shown in the plot because it is significantly less accurate in this setting. 
These results are in accordance with our analysis in Section \ref{sec:finite_sample}, where we show that the task of merging trees becomes challenging for small subsets of terminal nodes. 

\begin{figure}[t]
    \centering
    \includegraphics[width = \textwidth]{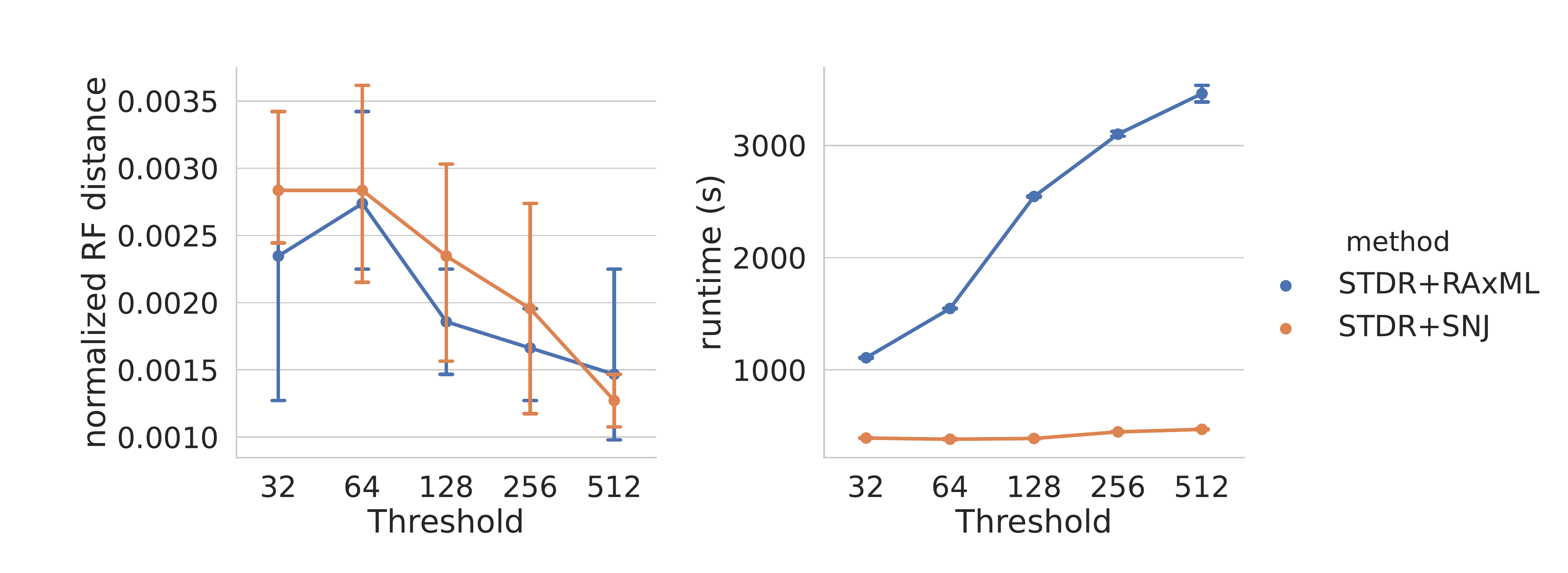}
    \caption{Effect of minimal tree size $\tau$ on runtime and accuracy of SDTR. Various values of threshold $\tau$ were chosen to test the performance of SDTR method in recovering a binary tree of size 2048 from sequences of length 1000. SNJ, and RAxML were used as the sub method of SDTR. }
    \label{fig:threshold}
\end{figure}
\section*{Acknowledgments}
The authors would like to thank Junhyong Kim, Stefan Steinerberger and Ronald Coifman for useful and insightful discussions. Y.K. and Y.A. acknowledge support by NIH grant R01GM131642, UM1DA051410 and R61DA047037. Y.K. and B.N. acknowledge support by NIH grant R01GM135928. Y.K. acknowledges support by NIH grant 2P50CA121974.


\begin{appendices}



\section{Example of Fiedler vector in a coalescent tree}

We generated a tree with $m=512$ nodes according to the coalescent model, see Figure \ref{fig:kingman_tree-tree}. 
The transition matrices  were set  according to the HKY model \cite{hasegawa1985dating}. We then generated a dataset of nucleotide sequences of length $n=2,000$. Figure \ref{fig:kingman_tree-eigs} shows the Fiedler vector of the similarity graph estimated from the dataset. Partitioning the terminal nodes according to the sign pattern of the Fiedler vector yields two clans. 
\begin{figure}[ht]
    \centering
    \begin{subfigure}[b]{0.4\textwidth}
        \includegraphics[width = 0.8\textwidth]{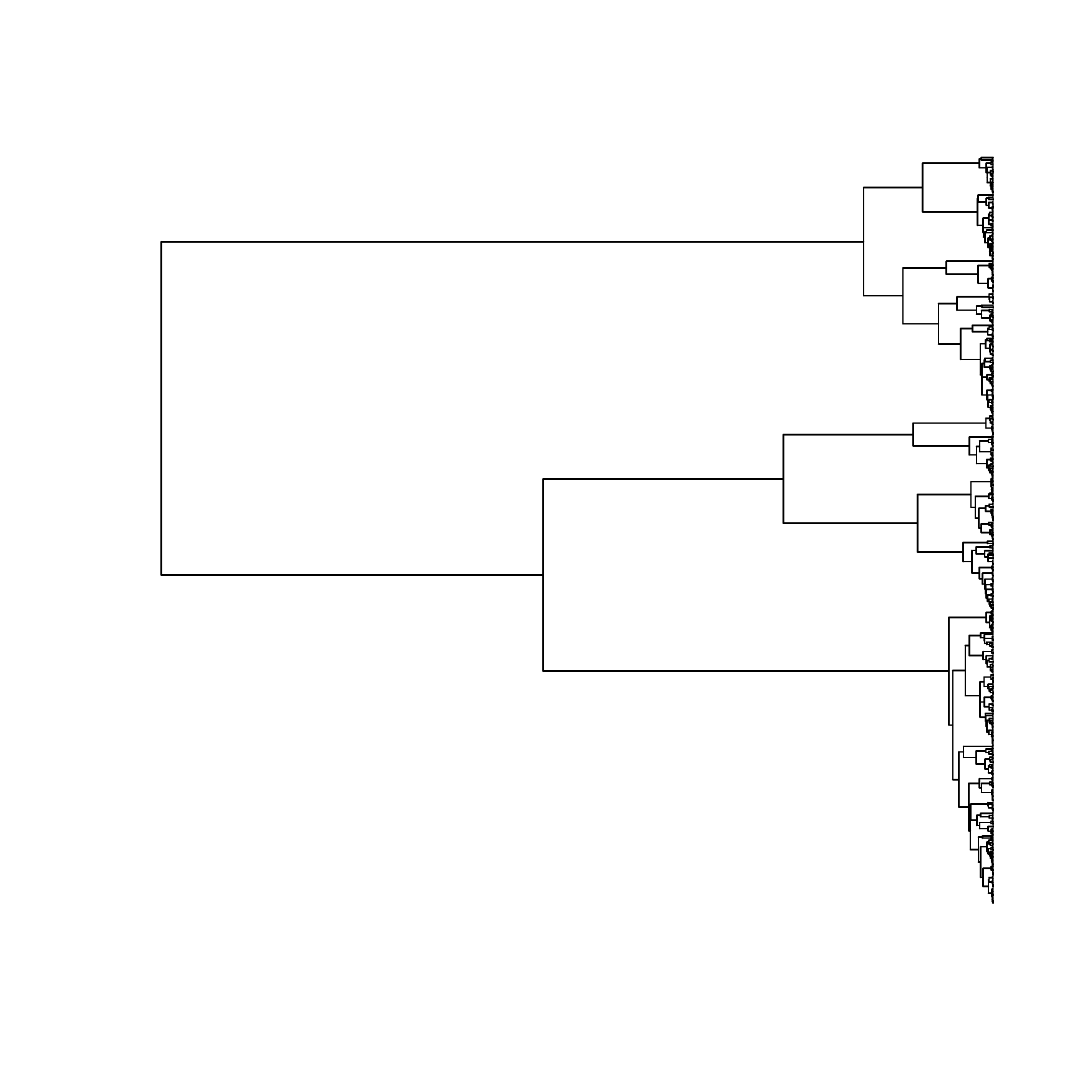}
        \caption{Generated coalescent tree}
        \label{fig:kingman_tree-tree}
 \end{subfigure}
    \begin{subfigure}[b]{0.53\textwidth}
        \includegraphics[width=\textwidth]{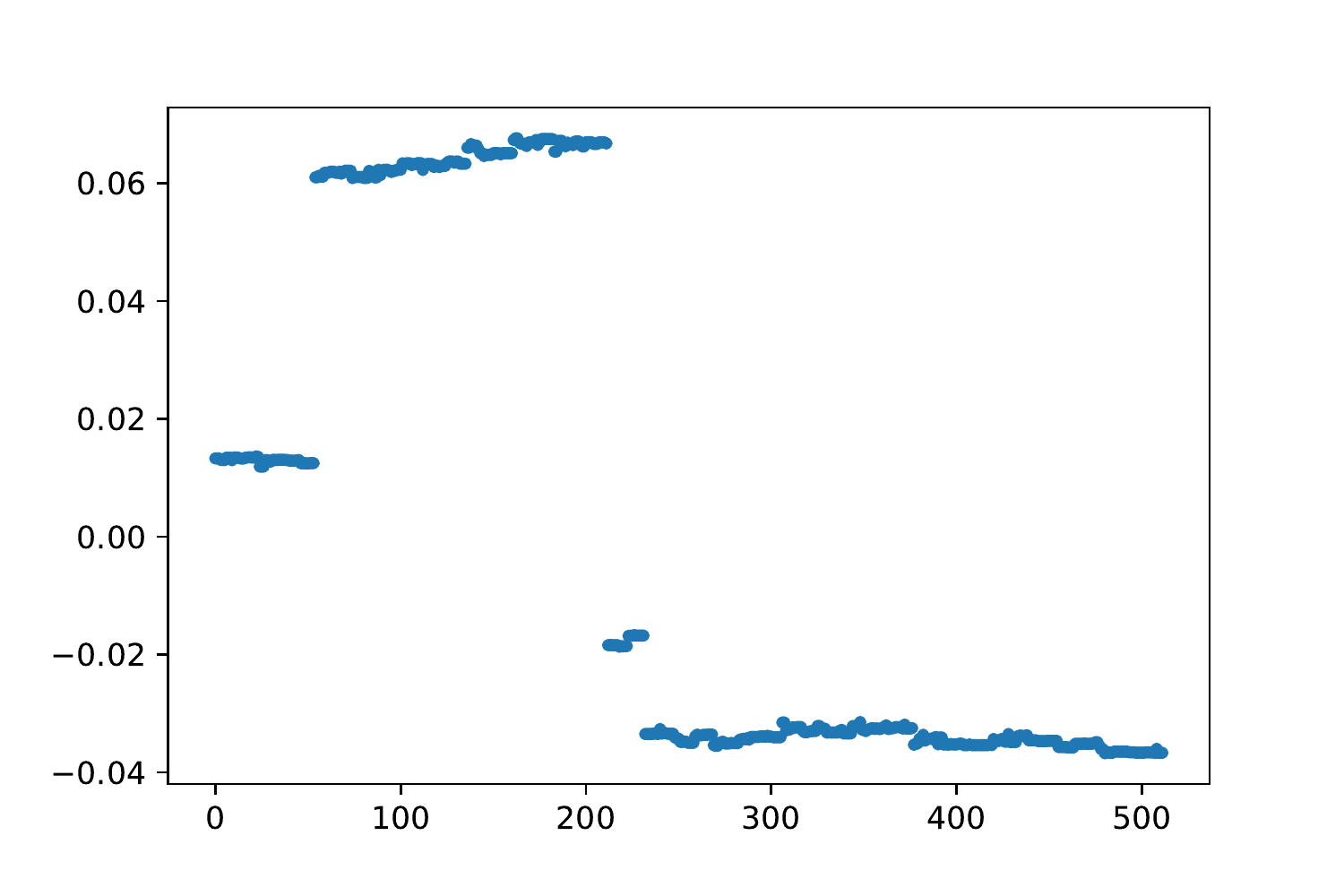}
        \caption{Fiedler vector of the coalescent tree}
        \label{fig:kingman_tree-eigs}
    \end{subfigure}
    \caption{Coalescent tree example with 512 terminal nodes}
    \label{fig:kingman_tree}
\end{figure}

\section{Relation between the partitioning step and the min-cut criterion}\label{appendix:mincut}
Let $\T$ be a binary tree and $G$ be its similarity graph, as defined in Section \ref{def:similarity_graph}.
The following lemma shows that partitioning the terminal nodes according to the min-cut criterion yields two clans of $\T$.
\begin{lemma}\label{lem:min_cut} 
Let $G$ be the similarity graph of a binary tree $\T$. Let $A^\ast$ and $B^\ast$ be a partition of the terminal nodes that minimizes the following min-cut criterion:
\begin{equation}\label{eq:min_cut}
(A^\ast,B^\ast) \in \argmin_{A,B} \mathrm{Cut}_G(A,B) = \argmin_{A,B} \sum_{i \in A,j \in B} S(x_i,y_j).
\end{equation}
  Then $A^\ast$ and $B^\ast$ are clans in $\T$.
\end{lemma}
\begin{proof}
Let $(x_1,x_2)$ be a pair of adjacent terminal nodes. Consider an arbitrary partition of the terminal nodes into two non-empty subsets, denoted $A$ and $B$. The two adjacent nodes $(x_1,x_2)$ can, respectively, be labeled $(A,B)$, $(A,A)$, $(B,A)$ or $(B,B)$.
We show that if $A$ and $B$ each contains nodes besides $x_1$ and $x_2$, then assigning $x_1$ and $x_2$ to the same subset decreases the value of the min-cut criterion.

Assume without loss of generality that $x_1 \in A, x_2\in B$. The cut between $A$ and $B$ is equal to
\[
\mathrm{Cut}(A,B) \equiv \sum_{x \in A,x' \in B} S(x,x') = S(x_1,x_2) + \sum_{x'\in B \setminus \{x_2\}} S(x_1,x') + \sum_{x\in A \setminus \{x_1\}} S(x,x_2) +S_0, 
\]
where 
\[
S_0 = \sum_{\substack{x \in A \setminus \{x_1\} \\ x' \in B \setminus \{x_2\}}} S(x,x')
\]
does not depend on the assignment of $x_1$ and $x_2$. Let $h$ be the unique node that is adjacent to both $x_1$ and $x_2$. From the multiplicative property of the similarity, we have \[
\mathrm{Cut}(A,B) = S(x_1,x_2) + S(x_1,h) \sum_{x' \in B \setminus \{x_2\}} S(h,x') + S(x_2,h) \sum_{x \in A \setminus \{x_1\}} S(x,h) +S_0.
\]
Without loss of generality, assume that
\begin{equation}\label{eq:partial_mincut_sums_order}
\sum_{x' \in B \setminus \{x_2\}} S(h,x') \geq \sum_{x \in A \setminus \{x_1\}} S(x,h).
\end{equation}
It follows that
\begin{align}\label{eq:mincut_score_improve}   
\mathrm{Cut}(A,B) &\geq  S(x_1,h) \sum_{x \in A \setminus \{x_1\}} S(x,h) + S(x_2,h) \sum_{x \in A \setminus \{x_1\}} S(x,h) +S_0  \\
&= \sum_{x \in A \setminus \{x_1\}} S(x,x_1) + \sum_{x \in A \setminus \{x_1\}} S(x,x_2) + \sum_{\substack{x \in A \setminus \{x_1\} \\ x' \in B \setminus \{x_2\}}} S(x,x') 
= \sum_{\substack{x \in A \setminus \{x_1\} \\ x' \in B \cup \{x_1\}}} S(x,x'). \notag 
\end{align}
Note that the right hand side of Eq. \eqref{eq:mincut_score_improve} equals the value of the cut of the same partition, but with $x_1$ moved from $A$ to $B$. 
Thus, the min-cut partition $\{A^*, B^*\}$ satisfies one of the following:
\begin{itemize}
    \item $x_1$ and $x_2$ are in the same subset.
    \item One of $A^*$ or $B^*$ equals exactly to $\{x_1\}$ or $\{x_2\}$.
\end{itemize}

Next, let $C_1$ and $C_2$ be two adjacent clans. Assume that the terminal nodes of each of the clans are homogeneous (i.e., they all belong to the same subset, $A$ or $B$).
The same argument for a pair of terminal nodes carries over to the case of two adjacent homogeneous clans, showing that the minimal cut partition $\{A^*, B^*\}$ satisfies one of the following:
\begin{itemize}
    \item $C_1$ and $C_2$ are in the same subset.
    \item One of $A^*$ or $B^*$ equals exactly $C_1$ or $C_2$.
\end{itemize}
Let $\{A, B\}$ be an arbitrary partition of the terminal nodes that does not correspond to two clans in the tree. 
Since $A$ and $B$ are not clans, there must be at least two disjoint pairs $C_1,C_2$ and $\tilde C_1,\tilde C_2$ of homogeneous adjacent subsets, where the nodes in $C_1$ are labeled by $A$ and the nodes in $C_2$ are labeled by $B$.  By our arguments $\text{Cut}(A,B)$ can be reduced by either changing the labels of $C_1$ to $B$ or $C_2$ to $A$ which implies that $\{A, B\}$ is not the min-cut partition. Thus, for any min-cut partition $\{A^*, B^*\}$, $A^*$ and $B^*$ are clans.
\end{proof}

\section{Supplementary proofs for Section  \ref{sec:algorithm}}\label{appendix:sec_algorithm}
We present here the proofs of Lemmas \ref{lem:correct_placeholder} and Lemma \ref{lem:merge_trees}  that are used in Section \ref{sec:algorithm}.

\begin{proof}[Proof of Lemma \ref{lem:correct_placeholder}]
Let $C_2$ be the clan of all the terminal nodes of $\T$ that are not in $C_1$. Consider an edge $e(h_A,h_B)$ in $\T_1$ that partitions $C_1$ into $A(e)$ and $B(e)$. 
First, assume that $e(h_A,h_B)$ is the correct placeholder edge of $\T_1$. Then there exists a node $h_1$ in the full tree $\T$ that is connected to $h_A,h_B$ and to the root node of $C_2$. Removing the edge $e(h_A,h_1)$ in $\T$ separates the subset $A(e)$ from the remaining nodes in $\T$, which implies that $A(e)$ is a clan in $\T$. By the same argument, $B(e)$ is also a clan in $\T$. 

Conversely, assume that $A(e)$, $B(e)$ and $C_2$ are disjoint clans that partition the terminal nodes of $\T$. 
Then, there exists a node $h_1$ that connects to the roots of $A(e),B(e)$ and $\T_2$. This proves that the edge $e(h_A, h_B)$ in $\T_1$ is the correct placeholder edge, since it is where the root $h_1$ is inserted.
\end{proof}

\begin{proof}[Proof of Lemma \ref{lem:merge_trees}]
Let $C_1 = A \cup B$ be the terminal nodes of the clan $\T_1$ and let  $h_1$ be its root. We denote by $C_2$ the terminal nodes in its adjacent clan. 
By the multiplicative property of the similarity function,
\[
S(C_1,C_2) = S(C_1,h_1) S(h_1,C_2). 
\]
Combining the above expression with Eq. \eqref{eq:clans_svd} implies that the left singular vector $u$ of $S(C_1,C_2)$ is proportional to 
the vector of similarities $S(C_1,h_1)$. That is, $\exists \beta \in \RR$ such that
$
    S(C_1,h_1)= \beta u.
$
Let $e$ be an edge in $\T_1$ that partitions the terminal nodes into $A(e),B(e)$. The vector 
$S(C_1,h_1)$ can be similarly partitioned into $S(A(e),h_1)$ and $S(B(e),h_1)$ such that
\begin{equation}\label{eq:split_sh_u}
 S(A(e),h_1)= \beta u_{A(e)}, \qquad 
 S(B(e),h_1)= \beta u_{B(e)}.
\end{equation}
We first prove that if $e$ is the correct placeholder edge of $\T_1$, then Eq. \eqref{eq:correct_placeholder} holds. 
By Lemma \ref{lem:correct_placeholder}, if $e$ is the correct placeholder edge then the root node $h_1$ separates $A(e)$ from $B(e)$. By Eq. \eqref{eq:split_sh_u} and the multiplicative property of the similarity measure, we have 
\[
    S(A(e),B(e)) = S(A(e),h_1)S(h_1,B(e)) =  u_{A(e)}\beta^2u_{B(e)}^T.
\]
Setting $\alpha = \beta^2$ proves Eq. \eqref{eq:correct_placeholder}.

Next, we assume that Eq. \eqref{eq:correct_placeholder} holds for some edge $e$ and prove that $e$ is the correct placeholder edge.
Consider the matrix $S(A(e),B(e) \cup C_2)$. 
Since $h_1$ is the root of $\T_1$,
\[
S(A(e),C_2) = S(A(e),h_1)S(h_1,C_2) \qquad \text{and} \qquad S(A(e),h_1)= \beta u_{A(e)}
\]
we have 
\[
S(A(e),C_2) = \beta u_{A(e)} S(h_1,C_2).
\]
Recall that by assumption $S(A(e),B(e)) = u_{A(e)}\alpha u_{B(e)}$. It follows that both matrices $S(A(e),B(e))$ and $S(A(e),C_2)$ are rank one with a left singular vector equal to $u_{A(e)}$. Thus, 
the concatenated matrix $S(A(e),B(e) \cup C_2)$ is rank-one.
By Lemma~\ref{lem:affinity_spectral}, this implies that $A(e)$ is a clan of the tree $\T$. A similar argument shows that $B(e)$ is also a clan in $\T$. Since $A(e)$ and $B(e)$ are both clans in $\T$, it follows from Lemma \ref{lem:correct_placeholder} that $e$ is the correct placeholder edge of $\T_1$.
\end{proof}

\section{Comparison to distance based tree partitioning}\label{appendix:griffing}
Let $D\in \RR^{m\times m}$ be a matrix 
whose elements are the pairwise phylogenetic distances between all terminal nodes.
Given the exact distance matrix, it was shown in \cite{griffing2012connections} that the terminal nodes of a tree can be partitioned into two clans according to the sign pattern of the leading eigenvector of the following matrix
\[
(I - \bm 1 \bm 1^T/m) D (I - \bm 1 \bm 1^T/m).
\]
Figure \ref{fig:similarity_distance_partition} shows the percentage of times the terminal nodes were  
correctly partitioned into clans by applying our similarity based approached vs. the distance-based approach derived in \cite{griffing2012connections}.
We generated $200$ random trees according to  Kingman's coalescent model with $m=128$ terminal nodes.  Figures \ref{fig:similarity_distance_vs_N} shows the ratio of times each method successfully partitioned the tree as a function of the number of samples with a fixed mutation rate between adjacent nodes of $\delta=0.9$. Similarly, Figure
 \ref{fig:similarity_distance_vs_delta} shows the performance of both methods  as a function of $\delta$ with a fixed number of samples $n=100$.  The advantage of using the similarity matrix over the distance matrix is clear. 
\begin{figure}
	\begin{subfigure}[b]{0.45\textwidth}
		\includegraphics[width = 0.8\textwidth]{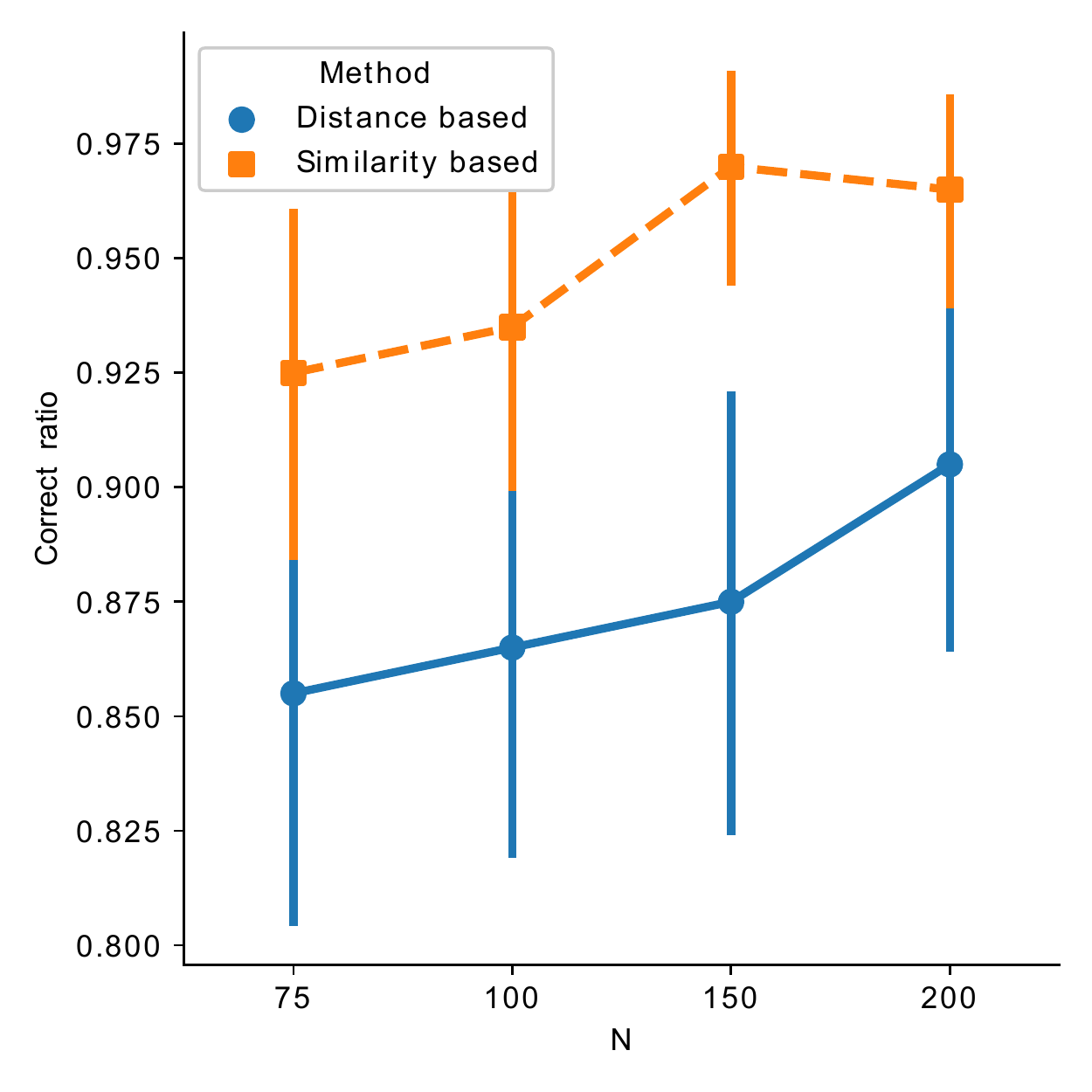}
		\caption{Partitioning accuracy vs. number of samples.}
		\label{fig:similarity_distance_vs_N}
	\end{subfigure}
	\begin{subfigure}[b]{0.45\textwidth}
	\includegraphics[width = 0.8\textwidth]{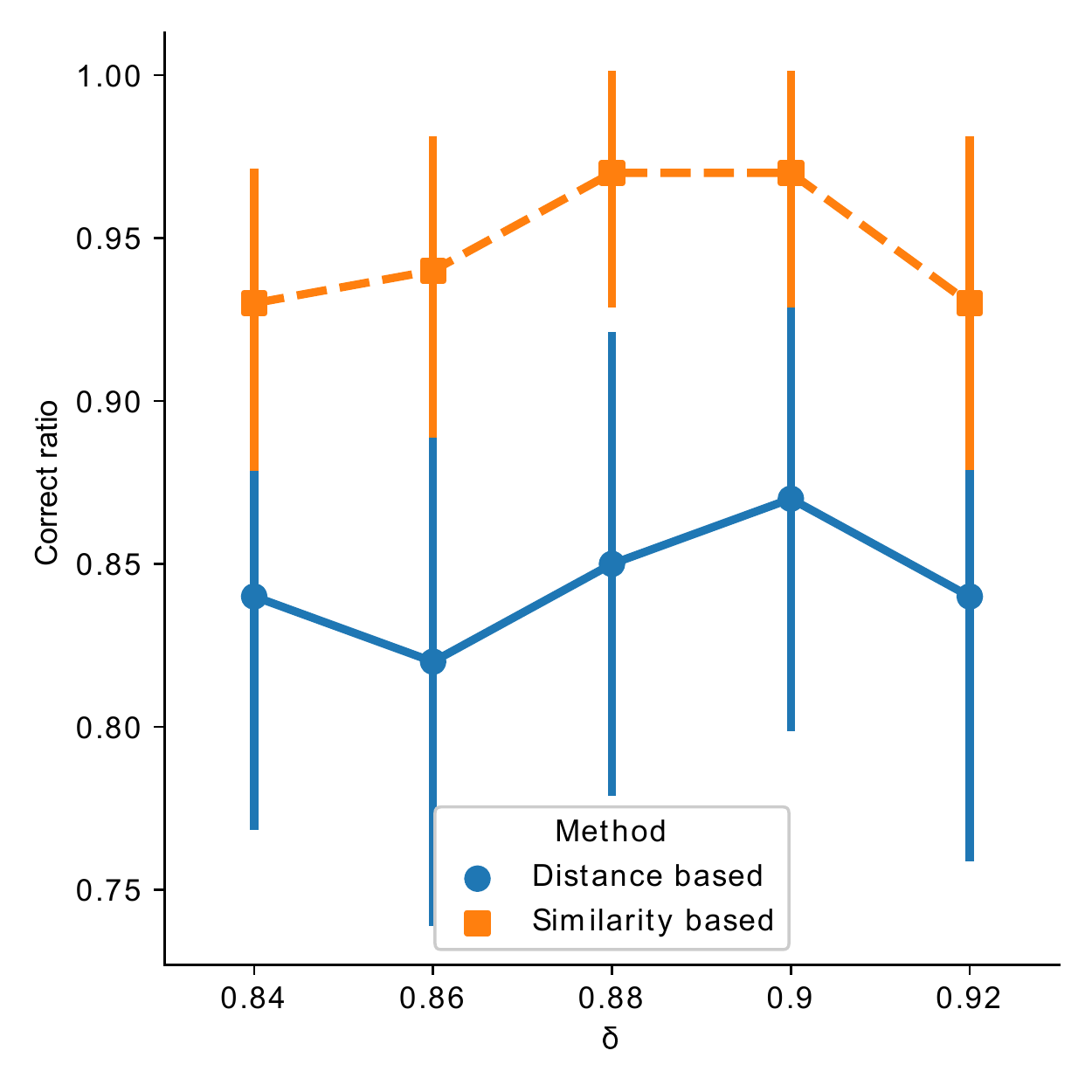}
	\caption{Partitioning accuracy vs. mutation rate.}
	\label{fig:similarity_distance_vs_delta}
	\end{subfigure}
	\caption{Comparison between distance based and similarity based spectral partitioning.}
	\label{fig:similarity_distance_partition}
\end{figure}

\section{Proof of Lemma \ref{lem:graph_construction}}\label{appendix:proof_of_consistensy_lemma}
We begin with several definitions and notations.
We denote by $G(v,w)$, $\T(v,w)$ the weight between nodes $v$ and  $w$ in a graph $G$ and tree $\T$, respectively. 
For a tree $\T$, we denote by $path_{\T}(v,w)$ the set of edges on the path between nodes $v$ and $w$,
\[
\mbox{path}_\T(v,w) = \{ (\tilde v,\tilde w) | \mbox{  $\tilde v$ and $\tilde w$ are adjacent nodes on the path between $v$ and $w$} \}.
\]
Next, we define the multiplicative weight between two nodes in a tree. 
\begin{definition}\label{def:product_path}
The multiplicative weight between $v$ and $w$ in a tree $\T$ is equal to,
\begin{equation}\label{eq:multi_weights}
\alpha_\T(v,w) = \prod_{(\tilde v,\tilde w)\in \mbox{path}_\T(v,w)} \T(\tilde v,\tilde w).
\end{equation}
\end{definition}
For example, let $\T$ be a tree whose edge weights are given by the similarity in Eq. \eqref{eq:similarity}, then the similarity between two terminal nodes $x_1,x_2$ is equal to the  multiplicative weight $\alpha_\T(x_1,x_2)$.
The next definition concerns a graph with nodes that correspond to a subset of nodes in $\T$, and weights computed according to \eqref{eq:multi_weights}.
\begin{definition}[Multiplicative subgraph]
Let $\T$ be a tree with a set of nodes $V$. We say that a graph $G$ is a \textit{multiplicative subgraph} with respect to $\T$ and a subset of nodes $\widetilde V \subset V$ if (i) the nodes of $G$ correspond to $\widetilde V$ and (ii) the weight assigned to an edge connecting $v,w$ in $G$ is equal to the multiplicative weight between $v$ and $w$ in $\T$, \[
G(v,w) = \alpha_{\T}(v,w).
\]
\end{definition}
For convenience, we will sometimes say that $G$ is a multiplicative subgraph of $\T$ without explicitly stating which nodes are included in $G$.
By definition, the similarity graph $G$ is a multiplicative subgraph with respect to the terminal nodes of $\T$.
Note that we use $v$ and $w$ as nodes both in $G$ and in $\T$ interchangeably, since by definition every node in $G$ corresponds to a node in $\T$. 

The proof of Lemma \ref{lem:graph_construction} is constructive. 
Given a tree $\T$ and its similarity graph $G$, we present an iterative procedure to build a second tree $\tilde{\T}$, with the same topology as $\T$, but with different weights such that
\[
L_G = L_{\tilde \T/R},
\]
where $R$ is the set of all internal nodes in $\T$.
Computing $\tilde \T$ consists of iterative and simultaneous updates of a graph and a tree: (i) a graph $G_i$ with nodes that correspond to a subset of the nodes in $\T$. The initial graph $G_0$ is set to $G$, with only the terminal nodes of $\T$. (ii) A tree $\T_i$, with the same topology as $\T$. The weights of the initial tree $\T_0$ are set such that $\T_0=\T$. 

At each iteration $i$, we add one of the non-terminal nodes  $h_i$ of $\T$ (that was not previously added) to $G_{i}$, 
creating $G_{i+1}$. The weights of the new graph $G_{i+1}$ are set such that
the Schur complement of its Laplacian matrix with respect to the added node $h_i$ is equal to the Laplacian of the previous graph $L_{G_{i}}$.
\begin{equation}\label{eq:update_graph}
 L_{G_{i}} = L_{G_{i+1}/h_i}.
\end{equation}
The steps for computing $G_{i+1}$ given $G_i$ and $\T_i$ are described in Algorithm \ref{alg:update_Gi}.
Next, we compute a new tree $\T_{i+1}$ with the same topology as $\T_{i}$. The weights of $\T_{i+1}$ are set such 
that $G_{i+1}$ becomes a multiplicative subgraph with respect to $\T_{i+1}$. The steps for computing $\T_{i+1}$ are described in Algorithm \ref{alg:update_Ti}.
At every iteration $i$, we maintain an \textit{active set} of nodes which we denote by $A_i$.
When updating $G_i$, changes are only made to edges connecting two nodes in $A_i \cup h_i$. When updating $\T_i$, changes are only made to edges on the path between two nodes in the active set. 
The initial active set $A_0$ is equal to all terminal nodes of $\T$. 

In our proof, we use the following two auxiliary lemmas, that show the correctness of the updating procedure of $G_i$ and $\T_i$. 
An implementation of Algorithms \ref{alg:update_Gi} and \ref{alg:update_Ti} is available on GitHub.
The first lemma proves the correctness of Algorithm \ref{alg:update_Gi}. 
The input to Algorithm \ref{alg:update_Gi} is the tree $\T_i$, a multiplicative subgraph $G_i$ and an active set $A_i$, all of which were computed in the previous iteration. 
The output of the algorithm is an updated graph $G_{i+1}$ that contains an additional node $h_i$. In addition, the algorithm updates the active set $A_i$ and creates $A_{i+1}$.
\begin{lemma}\label{lem:update_G}
The output of Algorithm \ref{alg:update_Gi}  is a graph $G_{i+1}$ whose nodes include $h_i$ as well as all the nodes in $G_i$ such that
\[
L_{G_{i+1}/h_i}  = L_{G_{i}}.
\]
\end{lemma}
The next lemma concerns the updating procedure of $\T_{i}$. The input to Algorithm \ref{alg:update_Ti} consists of the new active set  $A_{i+1}$, and the node $h_i$ added to $G_{i+1}$. Here, the only changes made are to edges on the path between $h_i$ and the nodes in the active set $A_{i+1}$.
\begin{lemma}\label{lem:update_T}
The tree $\T_{i+1}$ built according to Algoithm \ref{alg:update_Ti} is such that $G_{i+1}$ becomes a multiplicative subgraph of $\T_{i+1}$.
\end{lemma}
Figure \ref{fig:proof} shows two iterations of the aforementioned process for a tree $\T$ with four terminal and two non-terminal nodes. For simplicity, all the weights of the tree $\T$ are set to $1/2$. 

\begin{algorithm}[t]
	\caption{Updating $G_i$}
	\label{alg:update_Gi}
	\begin{algorithmic}[1]
		\Statex {\bfseries Input:}\begin{tabular}[t]{ll}
            $\T_i$ & - a tree graph\\
            $G_i$  & - a multiplicative subgraph of $\T_i$  \\
            $A_i$  & - active set of nodes
		\end{tabular}
		\Statex {\bfseries Output:}\begin{tabular}[t]{ll}
				$G_{i+1}$ & - updated graph such that $  L_{G_{i}} = L_{G_{i+1}/h_i}$ \\
				$A_{i+1}$  & - updated active set\\
				$h_i$ & - the node added to $G_i$\\
				$v_{i,1},v_{i,2}$ & - nodes removed form the active set
		\end{tabular}
	\State Initialize $G_{i+1} = G_i$ and $A_{i+1} = A_i$. 
	\State Choose a node $h_i$ in $\T_i$ that is not in $G_i$ and is adjacent to at least two nodes $v_{i,1},v_{i,2}$ in the active set $A_i$.  Add $h_i$ to $G_{i+1}$.
    \State Remove edges between the nodes $v_{i,1},v_{i,2}$ and the rest of the active set $A_i$.
    \State The weight between the new node $h_i$ and a node $x$ in the active set is computed by
    \begin{equation}\label{eq:G_ik_update}
        G_{i+1}(h_i,x) = d  \alpha_{\T_i}(h_i,x),
    \end{equation}
    where 
    \begin{equation}\label{eq:d_def}
    d = \sum_{x' \in A_i} \alpha_{\T_i}(h_i,x').
    \end{equation}
    \State The weights between two nodes $x,y$ in the active set (except $v_{i,1},v_{i,2}$) are updated by
    \begin{equation}\label{eq:G_jk_update}
        G_{i+1}(x,y) = G_{i}(x,y) -  \alpha_{\T_i}(h_i,x)
    \alpha_{\T_i}(h_i,y).
    \end{equation}
    
    \State Remove the nodes $v_{i,1}$ and $v_{i,2}$ from the active set $A_{i+1}$, and add $h_i$.
	\State \Return $G_{i+1}$, $A_{i+1}$, $h_i$,$v_{i,1}$,and $v_{i,2}$.
	\end{algorithmic}
\end{algorithm}

\begin{algorithm}[t]
	\caption{Updating $\T_i$}
	\label{alg:update_Ti}
	\begin{algorithmic}[1]
		\Statex {\bfseries Input:}\begin{tabular}[t]{ll}
            $\T_i$ & - a tree graph\\
            $A_{i+1}$  & - the active set\\
            $h_i$ & - the node last added to $G_{i+1}$\\
            $v_{i,1},v_{i,2}$ & - nodes that where removed from the active set in the last update
		\end{tabular}
		\Statex {\bfseries Output:}\begin{tabular}[t]{ll}
				$\T_{i+1}$ & - a tree with weights computed such\\
				           &~ that $G_{i+1}$ is a multiplicative subgraph of $\T_{i+1}$
		\end{tabular}
	\State Set $\T_{i+1}(h_i,v_{i,1}) =d \T_{i}(h_i,v_{i,1})$ and $\T_{i+1}(h_i,v_{i,2}) =d\T_{i}(h_i,v_{i,2})$
	\State For node $x \notin \{v_{i,1},v_{i,2}\}$ adjacent to $h_i$, set
    \[
    \T_{i+1}(h_i,x) = \frac{d\T_{i}(h_i,x)}{\sqrt{1-\alpha_{\T_{i}}(x,h_i)^2}},
    \]
    where $d$ is given by Eq.  \eqref{eq:d_def}.
    \State For two adjacent nodes $x,y \in \T$ where  $y$ is a node in the active set $A_{i+1}$ and $x$ is other path between $y$ and $h_i$, set
    \[
    \T_{i+1}(x,y) = \T_{i}(x,y) \sqrt{1-\alpha_{\T_{i}}(x,h_i)^2}  
    \]
    \State For two adjacent nodes $x,y \in \T$ that are not in the active set. If $\T_i(x,y)$ is on the path between a node in the active set and $h_i$, where $x$ is closer to $h_i$, set
    \[
    \T_{i+1}(x,y) = \T_{i}(x,y)\frac{\sqrt{1-\alpha_{\T_{i}}(x,h_i)^2}}{\sqrt{1-\alpha_{\T_{i}}(y,h_i)^2}}
    \]
	\State \Return $\T_{i+1}$
	\end{algorithmic}
\end{algorithm}

\begin{proof}[Proof of Lemma \ref{lem:graph_construction}]

We initialize the updating process with a tree $\T$ and its similarity matrix $G=G_0$. By definition, $G_0$ is a multiplicative subgraph of $\T$, and therefore satisfies the condition for Lemma \ref{lem:update_G}.
The lemma guarantees that after the first update, we obtain a graph $G_1$ with a Laplacian that satisfies,
\[
L_{G_{0}} = L_{G_{1}/h_0},
\]
where $h_0$ is the node added to $G_0$ at the first iteration. Lemma \ref{lem:update_T} guarantees that $G_1$ is a multiplicative subgraph of $\T_1$. Thus, we can re-apply Algorithm \ref{alg:update_Gi} with the pair $G_1,\T_1$. 
Thus, at each iteration $i$, we obtain a graph $G_{i+1}$ that satisfies,
\begin{equation}\label{schur_complement_i}
 L_{G_{i}} = L_{G_{i+1}/h_i}.
\end{equation}
Repeating the updating process for all $l$ non-terminal nodes of $\T$ yields the graph $G_l$, which by construction has the same topology as $\T$. In addition, due to the transitivity of the Schur's complement operation, Eq. \eqref{schur_complement_i} implies that
\[
 L_{\T_l/R}= L_{G_l/R} = L_{G_l/\{h_0,\ldots h_{l-1}\}} = L_{G_{l-1}/\{h_0,\ldots h_{l-2}\}} = \ldots = L_{G_1/h_0} = L_{ G_0} = L_G.
\]
Thus, $\T_l$ is a tree with the same topology as $\T$, but with different weights such that $L_{\T_l/R}=L_G$, which proves the lemma. 
\end{proof}


\begin{figure}[t]
  \centering
  
  \begin{subfigure}[b]{0.45\textwidth}
        \begin{tikzpicture}[-latex ,auto ,node distance =4 cm and 5cm ,on grid ,
            semithick ,scale=0.90,
            state/.style ={ circle ,top color =white , bottom color = blue!20 ,
            	draw,blue , text=blue , minimum width = 0.75 cm, inner sep=2pt, outer sep=0pt},
            	every edge quotes/.style = {auto=left, sloped,  inner sep=3pt}]	
            \node[state] (x1b) at (0,0) {$x_1$};
            \node[state] (x2b) at (0,3) {$x_2$};
            \node[state] (x3b) at (6,0) {$x_3$};
            \node[state] (x4b) at (6,3) {$x_4$};
            
        	\node[state, opacity =0.2] (hA) at (2,1.5) {$h_0$};
    	    \node[state, opacity =0.2] (hB) at (4,1.5) {$h_1$};
            
            \path[-] (x1b) edge ["$1/4$"] (x2b);	
            \path[-] (x1b) edge ["$1/8$"] (x3b);
            \path[-] (x1b) edge [pos=.7, "$1/8$"] (x4b);
            \path[-] (x2b) edge [pos=.3,"$1/8$"] (x3b);
            \path[-] (x2b) edge ["$1/8$"] (x4b);
            \path[-] (x3b) edge ["$1/4$"] (x4b);        
	    \end{tikzpicture}
        \caption{A graph $G_0$ that is a multiplicative sub-graph with respect to $\T$  and the the terminal nodes of $\T$.}
        \label{fig:proof_a}
    \end{subfigure}
    ~~~~~
  \begin{subfigure}[b]{0.45\textwidth}
    \begin{tikzpicture}[-latex ,auto ,node distance =4 cm and 5cm ,on grid ,
	semithick ,scale=0.90,
	state/.style ={ circle ,top color =white , bottom color = blue!20 ,
		draw,blue , text=blue , minimum width = 0.75 cm, inner sep=2pt, outer sep=0pt},
		every edge quotes/.style = {auto=left, sloped,  inner sep=3pt}]	
    	
    	
    	\node[state] (hA) at (2,1.5) {$h_0$};
    	\node[state] (hB) at (4,1.5) {$h_1$};
    	
    	\node[state] (x1) at (0,0) {$x_1$};
    	\node[state] (x2) at (0,3) {$x_2$};
    	\node[state] (x3) at (6,0) {$x_3$};
    	\node[state] (x4) at (6,3) {$x_4$};
    	
    	\path[-] (hA) edge ["$1/2$"] (x1);
    	\path[-] (hA) edge ["$1/2$"] (x2);
    	\path[-] (hB) edge ["$1/2$"] (x3);
    	\path[-] (hB) edge ["$1/2$"] (x4);
    	\path[-] (hA) edge ["$1/2$"] (hB);
        \end{tikzpicture}   
    \caption{A tree $\T_0$ with $4$ terminal nodes and $2$ internal nodes. The weight over all edges is equal to $1/2$.}
    \label{fig:proof_b}
    \end{subfigure}

    \vspace{1em}
  \begin{subfigure}[b]{0.45\textwidth}
        \begin{tikzpicture}[-latex ,auto ,node distance =4 cm and 5cm ,on grid ,
            semithick ,scale=0.90,
            state/.style ={ circle ,top color =white , bottom color = blue!20 ,
            	draw,blue , text=blue , minimum width = 0.75 cm, inner sep=2pt, outer sep=0pt},
            	every edge quotes/.style = {auto=left, sloped,  inner sep=3pt}]	
            \node[state] (x1) at (0,0) {$x_1$};
            \node[state] (x2) at (0,3) {$x_2$};
            \node[state] (x3) at (6,0) {$x_3$};
            \node[state] (x4) at (6,3) {$x_4$};
            \node[state] (h1) at (2,1.5) {$h_0$};
            \node[state, opacity =0.2] (hB) at (4,1.5) {$h_1$};
            
            \path[-] (x1) edge ["$3/4$"] (h1);	
            \path[-] (x2) edge [pos=.3, "$3/4$"] (h1);
            \path[-] (h1) edge [pos=.7,"$ 3/8$"] (x3);
            \path[-] (h1) edge ["$3/8$"] (x4);
            \path[-] (x3) edge ["$3/16$"] (x4);        
	    \end{tikzpicture}
        \caption{A graph $G_1$, created by adding the node $h_0$ to $G_0$. The weights of the graph are set such that $L_{G_0} = L_{G_1/h_0} $.}
        \label{fig:proof_c}
    \end{subfigure}
    ~~~~~
  \begin{subfigure}[b]{0.45\textwidth}
    \begin{tikzpicture}[-latex ,auto ,node distance =4 cm and 5cm ,on grid ,
	semithick ,scale=0.90,
	state/.style ={ circle ,top color =white , bottom color = blue!20 ,
		draw,blue , text=blue , minimum width = 0.75 cm, inner sep=2pt, outer sep=0pt},
		every edge quotes/.style = {auto=left, sloped,  inner sep=3pt}]	
    	
    	
    	\node[state] (hA) at (2,1.5) {$h_0$};
    	\node[state] (hB) at (4,1.5) {$h_1$};
    	
    	\node[state] (x1) at (0,0) {$x_1$};
    	\node[state] (x2) at (0,3) {$x_2$};
    	\node[state] (x3) at (6,0) {$x_3$};
    	\node[state] (x4) at (6,3) {$x_4$};
    	
    	\path[-] (hA) edge ["$3/4$"] (x1);
    	\path[-] (hA) edge ["$3/4$"] (x2);
    	\path[-] (hB) edge ["$\sqrt{3}/4$"] (x3);
    	\path[-] (hB) edge ["$\sqrt{3}/4$"] (x4);
    	\path[-] (hA) edge ["$\sqrt{3}/2$"] (hB);
        \end{tikzpicture}   
    \caption{A tree $\T_1$ with weights set such that $G_1$ is a multiplicative subgraph of $\T_1$ with respect to $\{x_1,x_2,x_3,x_4,h_1\}$}
    \label{fig:proof_d}
    \end{subfigure}
    
    \vspace{1em}
  \begin{subfigure}[b]{0.45\textwidth}
    \begin{tikzpicture}[-latex ,auto ,node distance =4 cm and 5cm ,on grid ,
	semithick ,scale=0.90,
	state/.style ={ circle ,top color =white , bottom color = blue!20 ,
		draw,blue , text=blue , minimum width = 0.75 cm, inner sep=2pt, outer sep=0pt},
		every edge quotes/.style = {auto=left, sloped,  inner sep=3pt}]	
    	
    	
    	\node[state] (hA) at (2,1.5) {$h_0$};
    	\node[state] (hB) at (4,1.5) {$h_1$};
    	
    	\node[state] (x1) at (0,0) {$x_1$};
    	\node[state] (x2) at (0,3) {$x_2$};
    	\node[state] (x3) at (6,0) {$x_3$};
    	\node[state] (x4) at (6,3) {$x_4$};
    	
    	\path[-] (hA) edge ["$3/4$"] (x1);
    	\path[-] (hA) edge ["$3/4$"] (x2);
    	\path[-] (hB) edge ["$3/4$"] (x3);
    	\path[-] (hB) edge ["$3/4$"] (x4);
    	\path[-] (hA) edge ["$3/2$"] (hB);
        \end{tikzpicture}   
    \caption{A tree $\T_2$ with weights set such that $L_{\T_2/h_2} = L_{G_1}$.}
    \label{fig:proof_e}
    \end{subfigure}
   
	\caption{Constructing a tree $\T_2$ such that the Schur complement of its Laplacian with respect to the internal nodes is equal to $L_G$.}
	\label{fig:proof}
\end{figure}
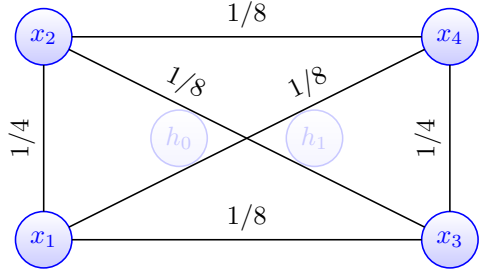
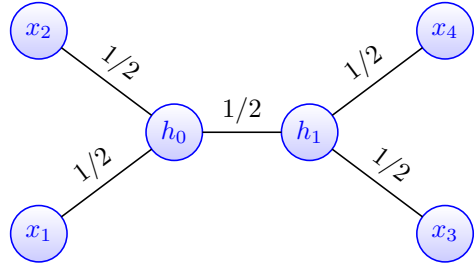
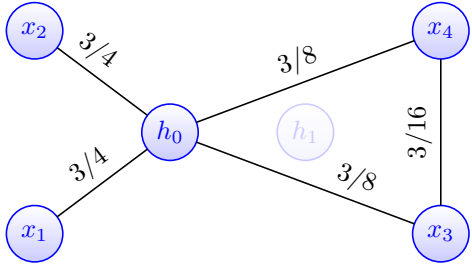
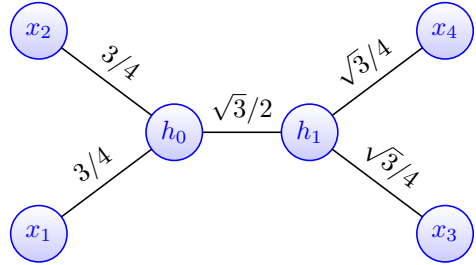
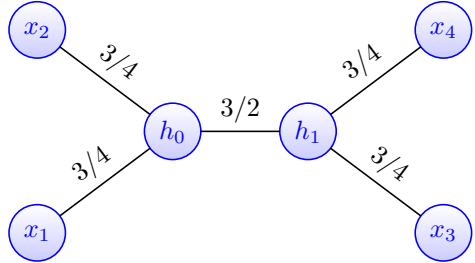

\begin{proof}[Proof of Lemma \ref{lem:update_G}]
Assume, for simplicity of notation, that the $j$th row/column of $L_{G}$ is the row/column that correspond to $h_j$ for any $j$ such that
\[
L_G(i,j) = -G(h_i,h_j) \qquad \forall h_i,h_j \in G \mbox{ with } i\neq j .
\]
We denote by $m_j$ the $j$-th column of $L_{G_{i+1}}$ after removing the $i$-th entry, and by $\bm 1$ the all one vector.
Since $h_i$ is a single node, the Schur complement $L_{G_{i+1}/h_i}$ defined in \eqref{def:schur} can be simplified to 
 
\begin{equation}\label{eq:Schur_L_i}
L_{G_{i+1}/h_i}(j,k) 
= L_{G_{i+1}}(j,k) - \frac{ (\bm 1^T m_j)(\bm 1^T m_k) }{\sum_{l\neq i} \bm 1^T m_l }.
\end{equation}
For a Laplacian matrix, the sum over any row is equal to zero. Since $m_j$ is equal to the row of $L_{G_{i+1}}$ after removing the $i$-th entry we have that $\bm 1^T m_j = -L_{G_{i+1}}(i,j)$. 
We rewrite Eq. \eqref{eq:Schur_L_i} as,
\begin{equation}\label{eq:schur_single_node}
L_{G_{i+1}/h_i}(j,k) = L_{G_{i+1}}(j,k) +  \frac{L_{G_{i+1}}(j,i)L_{G_{i+1}}(k,i)}{\sum_{l\neq i} L_{G_{i+1}}(i,l)  }.
\end{equation}
The only edges changed between $G_i$ and $G_{i+1}$  are edges between nodes in the active set $A_i$. Thus, if either $h_k$ or $h_j$ are not in the active set then
$L_{G_{i+1}}(j,k)=L_{G_{i}}(j,k)$. In addition, by step 4 of Algorithm \ref{alg:update_Gi}, the added node $h_i$ is only connected to nodes in the active set $A_{i}$.
Thus, if either node $h_k$ or $h_j$ are not part of $A_{i}$ we have
$L_{G_{i+1}}(j,i)L_{G_{i+1}}(k,i)=0$.
It follows that in this case
$L_{G_{i+1}/h_i}(j,k) = L_{G_{i}}(j,k)$ as required.
 
 Next, we assume that both $h_j$ and $h_k$ are part of the active set $A_{i}$. Eqs. \eqref{eq:G_ik_update} and \eqref{eq:G_jk_update} give
\begin{equation}\label{eq:j_k_active}
L_{G_{i+1}}(j,k) = L_{G_i}(j,k)+\alpha_{\T_{i}}(h_i,h_j)\alpha_{\T_{i}}(h_i,h_k), \qquad L_{G_{i+1}}(k,i) = -d\alpha_{\T_i}(h_i,h_k).
\end{equation}
By step 4 of Algorithm \ref{alg:update_Gi}, $h_i$ is only connected to nodes in the active set $A_i$. Inserting Eq. \eqref{eq:j_k_active} to Eq. \eqref{eq:schur_single_node} gives
\begin{equation}\label{eq:jk_in_active_set}
 L_{G_{i+1}/h_i}(j,k) = L_{G_i}(j,k)+\alpha_{\T_{i}}(h_i,h_j)\alpha_{\T_{i}}(h_i,h_k) -  \frac{d^2 \alpha_{\T_{i}}(h_i,h_j)\alpha_{\T_{i}}(h_i,h_k)}{\sum_{x \in A_{i}} d\alpha_{\T_{i}}(h_i,x)  },
 \end{equation}
 The denominator in the last term on the r.h.s of Eq. \eqref{eq:jk_in_active_set} is 
 equal to $d^2$ and hence,
 \[
  L_{G_{i+1}/h_i}(j,k) = L_{G_i}(j,k)+\alpha_{\T_{i}}(h_i,h_j)\alpha_{\T_{i}}(h_i,h_k) - d^2 \alpha_{\T_{i}}(h_i,h_j)\alpha_{\T_{i}}(h_i,h_k) \frac{1}{d^2} = L_{G_i}(j,k).
 \]
 We conclude that for any element $j,k$ we have 
 $L_{G_{i+1}/h_i}(j,k) = L_{G_i}(j,k)$.
 \end{proof}


\begin{proof}[Proof of Lemma \ref{lem:update_T}]
Here, our task is to prove that the weight assigned to any edge  $G_{i+1}(x,y)$ is equal to the multiplicative path $\alpha_{\T_{i+1}}(x,y)$. We address three cases: (i) the node $x$ is in the active set $A_{i+1}$ and $y$ is equal to the node $h_i$ added to the graph in iteration $i$. (ii) Both $x$ and $y$ are in $A_{i+1}$, and are not equal to $h_i$, and (iii) $x=h_i$ and $y$ is either $v_{i,1}$ or $v_{i,2}$.
For a pair of nodes $(x,y)$ that is not in $(i)-(iii)$ the edges in $G_i$ and $\T_i$ were not changed in the updating steps.

For case (i) we assume that $x$ is in $A_{i+1}$ and $y=h_i$ and hence by Eq. \eqref{eq:G_ik_update} in Algorithm \ref{alg:update_Gi}
\[
G_{i+1}(x,h_i) = d\alpha_{\T_i}(x,h_i).
\]
We denote the nodes on the path between $x$ and $h_i$ in $\T_i$ by
\[
path(h_i,x) = \{z_1 = h_i, z_2, \ldots, z_K = x\}.
\]
The edge between $h_i$ and $z_2$ is updated according to step 2 of Algorithm \ref{alg:update_Ti}. The edge between $z_{K-1}$ and $z_K$ is updated by step 3. The remaining edges are updated by step 4. 
The multiplicative weight $\alpha_{\T_{i+1}}(x,h_i)$ in the updated tree $\T_{i+1}$ according to Algorithm \ref{alg:update_Ti} is equal to
\begin{align}
\alpha_{\T_{i+1}}(x,h_i) &= 
\prod_{j=1}^{K-1} \T_{i+1}(z_j,z_{j+1}) \notag  \\&=
\frac{d\T_i(h_i,z_2)}{\sqrt{1-\alpha_{\T_{i}}(z_2,h_i)^2}} \times
\notag \\
&~~~~\prod_{j=2}^{K-2} \T_i(z_j,z_{j+1}) \frac{\sqrt{1-\alpha_{\T_{i}}(z_j,h_i)^2}}{\sqrt{1-\alpha_{\T_{i}}(z_{j+1},h_{i})^2}}    \sqrt{1-\alpha_{\T_{i}}(z_{K-1},h_i)^2}\T_i(z_{K-1},x)  \notag \\    
&=d \prod_{j=1}^{K-1} \T_i(z_j,z_{j+1})= d\alpha_{\T_i}(x,h_i).
\end{align}
Thus, the weight $G_{i+1}(x,h_i)=\alpha_{\T_{i+1}}(x,h_i)$ for any $x$ in the active set. 

In case (ii) $x,y$ are two nodes in the active set not equal to $h_i$. According to Eq. \eqref{eq:G_jk_update}  in Algorithm \ref{alg:update_Gi}
\[
G_{i+1}(x,y) = G_{i}(x,y) -  \alpha_{\T_i}(h_i,x)
    \alpha_{\T_i}(h_i,y).
\]
Denote by $u$ the unique node that connects between the nodes $x,y$ and $h_i$. Then,
\begin{equation}\label{eq:path_x_y}
\alpha_{\T_i}(h_i,x)
\alpha_{\T_i}(h_i,y) = \alpha_{\T_i}(h_i,u)^2   \alpha_{\T_i}(u,x)\alpha_{\T_i}(u,y) =\alpha_{\T_i}(h_i,u)^2   \alpha_{\T_i}(x,y).
\end{equation}
By assumption on the input to Alg. \ref{alg:update_Gi} of the previous iteration, the graph $G_i$ is a multiplicative subgraph of $\T_i$ and hence $G_i(x,y) = \alpha_{\T_i}(x,y)$. Thus, Eqs. \eqref{eq:G_jk_update} and \eqref{eq:path_x_y} imply
\begin{equation*}
    G_{i+1}(x,y) = G_i(x,y) - \alpha_{\T_i}(h_i,u)^2   \alpha_{\T_i}(x,y) 
    = G_i(x,y) - \alpha_{\T_i}(h_i,u)^2   G_i(x,y) = G_i(x,y)(1-\alpha_{\T_i}(h_i,u)^2).
\end{equation*}
Next, we show that $G_{i+1}(x,y)$ is equal to the multiplicative weight $\alpha_{\T_{i+1}}(x,y)$.
Let $z_1=x,\ldots,z_\kappa=u ,\ldots,z_K=y$ be the nodes on the path between $x$ and $y$.
By steps 2 and 3 in Algorithm \ref{alg:update_Ti}, the multiplicative weight $\alpha_{\T_{i+1}}(x,y)$ is equal to 
\begin{align}
    \alpha_{\T_{i+1}}(x,y) &=\prod_{j=1}^{\kappa-1} \T_{i+1}(z_j,z_{j+1}) \prod_{j=\kappa}^{K-1} \T_{i+1}(z_j,z_{j+1})  \notag \\
    &=
    \T_{i}(x,z_2)\sqrt{1-\alpha_{\T_{i}}(z_2,h_i)^2}
    \prod_{j=2}^{\kappa-1}
    \T_{i}(z_j,z_{z+1})
    \frac{\sqrt{1-\alpha_{\T_{i}}(z_{j+1},h_i)^2}}{\sqrt{1-\alpha_{\T_{i}}(z_j,h_i)^2}} \notag \\
    & \times 
    \T_{i}(y,z_{K-1})\sqrt{1-\alpha_{\T_{i}}(z_{K-1},h_i)^2}
    \prod_{j=\kappa}^{K-2}
    \T_{i}(z_j,z_{z+1})
    \frac{\sqrt{1-\alpha_{\T_{i}}(z_{j},h_i)^2}}{\sqrt{1-\alpha_{\T_{i}}(z_{j+1},h_i)^2}}.
\end{align}
Note that 
\begin{multline*}
\T_{i}(x,z_2)\sqrt{1-\alpha_{\T_{i}}(z_2,h_i)^2}
    \prod_{j=2}^{\kappa-1}
    \T_{i}(z_j,z_{z+1})
    \frac{\sqrt{1-\alpha_{\T_{i}}(z_{j+1},h_i)^2}}{\sqrt{1-\alpha_{\T_{i}}(z_j,h_i)^2}} = \sqrt{1-\alpha_{\T_{i}}(z_{\kappa},h_i)^2} \prod_{j=1}^{\kappa-1} \T_{i}(z_j,z_{z+1})     
\end{multline*}
and 
\begin{multline*}
\T_{i}(y,z_{K-1})\sqrt{1-\alpha_{\T_{i}}(z_{K-1},h_i)^2} \prod_{j=\kappa}^{K-2}
    \T_{i}(z_j,z_{z+1})
    \frac{\sqrt{1-\alpha_{\T_{i}}(z_{j},h_i)^2}}{\sqrt{1-\alpha_{\T_{i}}(z_{j+1},h_i)^2}}
    \\= \sqrt{1-\alpha_{\T_{i}}(z_{\kappa},h_i)^2}\prod_{j=\kappa}^{K-1} \T_{i}(z_j,z_{z+1}) 
\end{multline*}
and thus, 
\[
\alpha_{\T_{i+1}}(x,y) = \prod_{j=1}^{K-1} \T_{i}(z_j,z_{z+1}) (1-\alpha_{\T_{i}}(z_{\kappa},h_i)^2) = G_{i+1}(x,y).
\]

Lastly, we consider case (iii),  where 
$x=h_i$ and $y=v_{i,1}$ or $y=v_{i,2}$. Recall that $v_{i,1}, v_{i,2}$ are adjacent to $h_i$ in $\T$ and were removed from the active set. By step~4 of Algorithm \ref{alg:update_Gi} and step~1 of Algorithm \ref{alg:update_Ti} the edge $G_{i}(x,y)$ and its corresponding edge $\T_i(x,y)$ have both been updated such that $\T_{i+1}(x,y) = G_{i+1}(x,y) = d \T_i(x,y)$.
\end{proof}

\section{Auxiliary Lemmas for Section \ref{sec:finite_sample}}\label{appendix:proof_of_incorrect_merge_lemma}

\begin{proof}[Proof of Lemma \ref{lem:symmetric_spectrum}]
We begin by characterizing all the eigenvectors of $L\in \RR^{m\times m}$.
For any non-terminal node $h$ in the binary symmetric tree $\T$, we denote the set of descendent terminal nodes to the ``left" of $h$ by $A$, the set of descendant terminal nodes to the ``right" of $h$ by $B$, and the rest of the terminal nodes by $C$.
Let $v_h \in \RR^m$ be a vector with 
\[
(v_h)_i = \begin{cases}
1 & i \in A\\
-1 & i \in B\\
0 & i \in C.
\end{cases}
\]
We show that for any choice of non-terminal node $h$, $v_h$ is an eigenvector of $L$. Since there are $m-1$ non-terminal nodes, this set of eigenvectors, together with the vector of all-ones, forms the full set of all eigenvectors of $L$. 


First, we show that $v_h$ is an eigenvector of the similarity matrix $S$, and compute the corresponding eigenvalue.
For $i \in A$, 
\[
(Sv_h)_i = \sum_{j \in A} S(i,j) - \sum_{k \in B} S(i,k).
\]
Due to the symmetry of the tree $\T$, every terminal node has a similarity of $\delta^2$ to one other terminal node, $\delta^4$ to two other terminal nodes, etc. Thus,
\[
\sum_{j \in A} S(i,j) = 1+\delta^2+2\delta^4 +\ldots,+ \ldots,|A| \delta^{2\log_2 |A|}= \delta^2 \left(\frac{1-(2\delta^2)^{\log_2 |A|}}{1-2\delta^2}\right) + 1.
\]
The similarity between a node $i \in A$ and all nodes $k \in B$ is equal to $\delta^{2(\log|A|+1)}$. Thus,
\begin{align}\label{eq:eig_S}
\sum_{j \in A} S(i,j) - \sum_{k \in B} S(i,k) &= \delta^2 \left(\frac{1-(2\delta^2)^{\log_2 |A|}}{1-2\delta^2}\right) + 1 - |A|\delta^{2(\log|A|+1)} 
\notag\\ &= 
1 + \delta^2\left(\frac{1-(2\delta^2)^{\log |A|}(2-2\delta^2)}{1-2\delta^2}\right).
\end{align}
The same result with a negative sign holds for $i \in B$. If $i \in C$ 
then by symmetry $(Sv_h)_i=0$. Thus $v_h$ is an eigenvector of $S$ with eigenvalue equal to the right side of \eqref{eq:eig_S}.
The sum of every row in $S$ is equal to,
\begin{equation}\label{eq:D_entries}
d_i = \sum_j S(i,j) = 1 + \delta^2 + 2\delta^4 + \ldots + 2^{\log_2 m} \delta^{2\log_2 m} = \delta^2 \left(\frac{1-(2\delta^2)^{\log_2 m}}{1-2\delta^2}\right) + 1.    
\end{equation}
Let $D$ be the scalar matrix with diagonal elements equal to Eq. \eqref{eq:D_entries}.
Combining Eq. \eqref{eq:D_entries} and Eq. \eqref{eq:eig_S}, we get that $v_h$ is an eigenvector of $L = D-S$ with eigenvalue:
\begin{equation}\label{eq:symmetric_eigenvalue}
\lambda(h) = \delta^2 \left(\frac{(2\delta^2)^{\log_2 |A|} (2-2\delta^2) - (2\delta^2)^{\log_2 m}}{1-2\delta^2}\right).
\end{equation}
For any Laplacian matrix $0$ is an eigenvalue that correspond to the vector of all-ones.
Since the eigenvalue $e(h)$
decreases as $|A|$ grows, the two smallest non-zero eigenvalues correspond to  $|A|=m/2$ and $|A|=m/4$. 
  The three smallest eigenvalues are thus equal to,
\[
\lambda_1 = 0, 
\qquad 
\lambda_2 = m^{2 \log_2(\delta) + 1}, \qquad
\lambda_3 = m^{2 \log_2(\delta) + 1} \left(\frac12 + \frac1{2\delta^2} \right).
\]
\end{proof}

In the following proof, we use similar notations as in the proof of Lemma \ref{thm:merge_incorrect_lower_bound}.

\begin{proof}[Proof of Lemma \ref{lem:merge_lower_bound_single_pair} ]
For simplicity, let $x = \|S(A_i,B)\|^2_F$ and $y = \|S(A_{i+k},B)\|^2_F$. To compute the numerator of Eq.  \eqref{eq:merge_lemma_two_terms}, we set the partial derivative w.r.t. $\beta$ to 0, which gives 
\[
    \beta^\ast = \argmin_{\beta} \Big( (1-\beta R_i)^2 x +(1-\beta R_{i+k})^2 y\Big) = \frac{R_ix + R_{i+k}y}{R_i^2x + R_{i+k}^2y}.
\]
Plugging $\beta^\ast$ back into the numerator of Eq.  \eqref{eq:merge_lemma_two_terms} gives
\[
    \min_\beta \Big( (1-\beta R_i)^2 x + (1-\beta R_{i+k})^2 y\Big) = \frac{xy(R_{i}-R_{i+k})^2}{R_i^2x+R_{i+k}^2y}.
\]
Observe that $R_{i+k} = R_i S(h_i,h_{i+k})^2$. Thus, the above expression further simplifies to
\[
\frac{xy(R_{i}-R_{i+k})^2}{R_i^2x+R_{i+k}^2y} = 
\frac{xy R_i^2(1-S(h_i,h_{i+k})^2)^2}{R_i^2(x+S(h_i,h_{i+k})^4y)} = \frac{xy (1-S(h_i,h_{i+k})^2)^2}{x+S(h_i,h_{i+k})^4y}.
\]
Since $\|S(A_{i},B)\|^2_F + \|S(A_{i+k},B)\|^2_F = x+y$, the LHS of \eqref{eq:merge_lemma_two_terms} is equal to
\begin{equation}
\label{eq:merge_equality}
 \frac{xy (1-S(h_i,h_{i+k})^2)^2}{(x+y)(x+S(h_i,h_{i+k})^4y)}.
\end{equation}
Recall from Eqs. $\eqref{eq:assumption_1}$ and $\eqref{eq:adjacent_similarity}$ that for any $1 \leq i \leq N-1$, $S(h_i,h_{i+k}) < \xi < 1$. It follows that
\begin{equation}\label{eq:merge_max_lower_bound}
    \frac{xy (1-S(h_i,h_{i+k})^2)^2}{(x+y)(x+S(h_i,h_{i+k})^4y)} \geq 
    \frac{xy (1-\xi^2)^2}{(x+y)^2} \geq 
    \frac{xy (1-\xi^2)^2}{(2 \max (x,y))^2} 
    = \frac{(1-\xi^2)^2\min(x,y)}{4\max(x,y)}.
\end{equation}
Next, we simplify the term $\frac{\min(x,y)}{\max(x,y)}$ in Eq.  \eqref{eq:merge_max_lower_bound}. Note that $h_{i+k}$ separates $A_i$ and $A_{i+k}$ from $B$, see ilustration in Figure \ref{fig:incorrect_placeholder_edge}. Thus, we can rewrite $\min(x,y)$ as
\begin{align*}
    \min(x,y) &= \min(\|S(A_i,B)\|^2_F, \|S(A_{i+k},B)\|^2_F)\\
    &= \min(\|S(A_i, h_{i+k}) S(h_{i+k}, B)\|^2_F, \|S(A_{i+k}, h_{i+k}) S(h_{i+k}, B)\|^2_F)\\
    &= \min(\|S(A_i, h_{i+k})\|^2 \|S(h_{i+k}, B)\|^2, \|S(A_{i+k}, h_{i+k})\|^2 \|S(h_{i+k}, B)\|_F^2)\\
    &= \min(\|S(A_i, h_{i+k})\|^2, \|S(A_{i+k}, h_{i+k})\|^2) \cdot \|S(h_{i+k}, B)\|^2.
\end{align*}
Similarly, $
\max(x,y) = \max(\|S(A_i, h_{i+k})\|^2, \|S(A_{i+k}, h_{i+k})\|^2) \cdot \|S(h_{i+k}, B)\|^2$.
Thus,
\[
\frac{\min(x,y)}{\max(x,y)} = \frac{\min(\|S(A_i, h_{i+k})\|^2, \|S(A_{i+k}, h_{i+k})\|^2)}{\max(\|S(A_i, h_{i+k})\|^2, \|S(A_{k+1}, h_{i+k})\|^2)}.
\]
Next, we provide lower and upper bounds on the terms $\|S(A_i, h_{i+k})\|^2$ and $\|S(A_{i+1}, h_{i+k})\|^2$.
By Eq. \eqref{eq:assumption_1}, the similarity between the nodes in $A_i,A_{i+k}$ and $h_{i+k}$ is bounded by $\xi$. It follows that
\begin{equation}\label{eq:bound_denominator}
    \max(\|S(A_i,h_{i+k})\|^2,\|S(A_{i+k},h_{i+k})\|^2) \leq \max(|A_i|, |A_{i+k}|) \xi^2 \leq m \xi^2.
\end{equation}
For a lower bound, we apply  \cite[Lemma 4.5]{jaffe2020spectral}. Given the terminal nodes of a clan $A$, and the root of a clan $h$, the lemma bounds the norm of $S(A,h)$ by,
\begin{equation*}
    \|S(A,h)\|_F^2 \geq 
    \begin{cases}
    (2\delta^2)^{\log |A|} & \delta^2 \leq 0.5 \\
    2\delta^2 & \delta^2 > 0.5
    \end{cases}
    \geq
    \begin{cases}
    (2\delta^2)^{\log m} & \delta^2 \leq 0.5 \\
    2\delta^2 & \delta^2 > 0.5.
    \end{cases}
\end{equation*}
There are $k+1$ edges between the root of $A_i$ and $h_{i+k}$, and one edge between the root of $A_{i+k}$ and $h_{i+k}$. Thus,
\begin{equation}\label{eq:bound_numerator}
\begin{aligned}
    \min(\|S(A_i,h_{i+k})\|^2,\|S(A_{i+k},h_{i+k})\|^2) &\geq 
    \begin{cases}
    (2\delta^2)^{\log m}\delta^{2(k+1)} & \delta^2 \leq 0.5 \\
    2\delta^{(2k+2)} & \delta^2 > 0.5.
    \end{cases}
\end{aligned}
\end{equation}
Plugging Eqs. \eqref{eq:bound_denominator}, \eqref{eq:bound_numerator} into \eqref{eq:merge_max_lower_bound} concludes the proof.   
\end{proof}

\begin{proof}[Proof of Lemma \ref{lem:ratio_sum_inequality} ]
The lemma is a small variation over the known lower bound for ratio of sums, $\frac{\sum_i a_i}{\sum_i b_i}\geq \min_i \frac{a_i}{b_i}$. 
For an even number of elements, we can merge non overlapping pairs of consecutive elements such that $\tilde a_i = a_{2i}+a_{2i+1}$ and 
$\tilde b_i = b_{2i}+b_{2i+1}$. Applying the standard bound for ratio of sums for $\tilde a_i$ and $\tilde b_i$ gives,
\[
\frac{\sum_i \tilde a_i}{\sum_i \tilde b_i}\geq \min_i \frac{\tilde a_i}{\tilde b_i}=\min_i \frac{a_{2i}+ a_{2i+1}}{b_{2i}+b_{2i+1}}\geq \min_{i\neq j;|i-j|\leq 2}\frac{a_i+a_j}{b_i+b_j}.
\]
For an odd number of elements, we can merge the first three elements $i=0,1,2$. The rest will be merged into consecutive pairs. 
\[
\frac{\sum_i a_i}{\sum_i b_i} \geq \min \left\{\frac{a_0+a_1+a_2}{b_0+b_1+b_2},\frac{\sum_{i\geq 2}(a_{2i}+a_{2i+1})}{\sum_{i\geq 2}(b_{2i}+b_{2i+1})}\right\} 
\]
The ratio for elements $i=0,1,2$ can be bounded by the minimum ratio over all pairs $i,j \in \{0,1,2\}$. Thus,
\[
\frac{\sum_i a_i}{\sum_i b_i} \geq \min \left\{\min_{i \neq j \in \{0,1,2\}} \frac{a_i+a_j}{b_i+b_j},\frac{\sum_{i\geq 2}(a_{2i}+a_{2i+1})}{\sum_{i\geq 2}(b_{2i}+b_{2i+1})}\right\} \geq \min_{i\neq j;|i-j|\leq 2}\frac{a_i+a_j}{b_i+b_j}
\]

\end{proof}

\begin{lemma}\label{lem:aux_ratio_bound}
Let $X,X'\in \RR^{m\times n}$
and let $y,y' >0$. Assume that
$\|X'\|_F \leq y'$, then
\begin{equation}
    \Big \|\frac{X}{y} - \frac{X'}{y'} \Big \|_F
    \leq\frac{1}{y}(\|X - \hat X\|_F+|y-\hat y|).
    \end{equation}
\end{lemma}
\begin{proof}
By definition, 
\[
    \Big\|\frac{X}{y} - \frac{\hat X}{\hat y}\Big\|_F = \Big \| \frac{Xy' - X'y}{yy'}  \Big \|_F
    =\Big \|\frac{y'(X-X')}{y y'}  + \frac{X'(y'-y)}{yy'} \Big \|_F
    \]
    By the triangle inequality
    \[
\Big \|\frac{X}{y} - \frac{\hat X}{\hat y}\Big \|_F
\leq \frac1y \|X-X'\|_F + \frac{|y'-y|}y\cdot \frac{\|X'\|_F}{y'}
    \]
    Since $\|X'\|_F \leq y'$ the lemma follows.
\end{proof}


\begin{lemma}\label{lem:aux_matrix_diff_bound}
Let $X$ and $Y$ be two matrices and let $\hat X$ and $\hat Y$ be their corresponding noisy estimates. Then,
\[
\big| \|X-Y\|_F -\|\hat X-\hat Y\|_F \big| \leq \|X-\hat X\|_F + \|Y-\hat Y\|_F.
\]
\end{lemma}
\begin{proof}
Assume that $\|X-Y\|_F\geq \|\hat X-\hat Y\|_F$. In this case
\begin{multline*}
\big| \|X-Y\|_F -\|\hat X-\hat Y\|_F \big| 
=
\|X-Y\|_F -\|\hat X-\hat Y\|_F
\leq 
\|X-Y-\hat X+\hat Y\|_F
\leq
\|X-\hat X\|_F +\|Y-\hat Y\|_F.
\end{multline*}
Alternatively, if $\|X-Y\|_F \leq \|\hat X-\hat Y\|_F$ we have
\begin{multline*}
\big| \|X-Y\|_F -\|\hat X-\hat Y\|_F \big| 
=
 \|\hat X-\hat Y\|_F-\|X-Y\|_F
\leq 
\|\hat X-\hat Y-X+Y\|_F
\leq
\|\hat X-X\|_F +\|\hat Y-Y\|_F.
\end{multline*}
\end{proof}
\begin{lemma}\label{lem:aux_product_bound}
Let $X \in \RR^{n_1 \times n_2},Y\in \RR^{n_2 \times n_3},Z \in \RR^{n_3 \times n_4}$ be three matrices and let $\hat X,\hat Y, \hat Z$ be there corresponding estimates. Then
\[
\|XYZ - \hat X\hat Y\hat Z\|_F \leq \|X\|_F \|Y\|_F\|Z-\hat Z\|_F
+
\|\hat Z\|_F\|Y\|_F\|X-\hat X\|_F + \|\hat Z\|_F\|\hat X\|_F\|Y-\hat Y\|_F
\]
\end{lemma}
\begin{proof}
\begin{equation}
\|XYZ - \hat X\hat Y\hat Z\|_F = 
\|XYZ -  X Y \hat Z +  X Y \hat Z - \hat X\hat Y\hat Z\|_F 
\leq 
\|X\|_F\|Y\|_F\|Z - \hat Z\|_F + \|\hat Z\|_F\|XY - \hat X\hat Y\|_F    
\end{equation}
Focusing on $\|XY - \hat X\hat Y\|_F$ we have that
\[
\|XY - \hat X\hat Y\|_F = 
\|XY - \hat X Y  + \hat X Y - \hat X\hat Y\|_F
\leq 
\|X-\hat X\|_F\|Y\|_F + \|\hat X\|\|Y-\hat Y\|_F
\]
Combining the two bounds gives,
\[
\|XYZ - \hat X\hat Y\hat Z\|_F 
\leq 
\|X\|_F \|Y\|_F\|Z-\hat Z\|_F
+
\|\hat Z\|_F\|Y\|_F\|X-\hat X\|_F + \|\hat Z\|_F\|\hat X\|_F\|Y-\hat Y\|_F
\]
\end{proof}

\begin{lemma}\label{lem:bound_eps_A}
Let $S$ denote a rank one matrix and $\hat S$ its noisy estimate. We denote by $u,\hat u$ their respective leading left singular vectors. If $\|S - \hat S\|_F \leq 0.5\|S\|_F$
 then
\[
\|u u^T - \hat u\hat u^T\|_F^2 \leq \frac{50\|S-\hat S\|_F^2}{\|S\|_F^2}.
\]
\end{lemma}
\begin{proof}
\begin{align}\label{eq:bound_outer_product_diff_b}
\|uu^T - \hat u\hat u^T\|_F^2 
&= 
\sum_{ij} (uu^T-\hat u\hat u^T)_{ij}^2 = 
\sum_{ij} (uu^T)_{ij}^2 + \sum_{ij}(\hat u\hat u^T)_{ij}^2 - 
2\sum_{ij} (uu^T)_{ij}(\hat u\hat u^T)_{ij}
\notag \\
&=\|u\|^4+\|\hat u\|^4 - 2\sum_{i} u_i \hat u_{i}
\sum_j
u_j \hat u_{j} =  
2(1-(u^T\hat u)^2) = 2\sin^2(u,\hat u).
\end{align}
We apply a variant of the Davis-Kahan theorem for non square matrices \cite[Theorem 3]{yu2015useful}. The perturbation of the leading singular vector is bounded by
\[
\sin(u,\hat u) \leq \frac{2(2\sigma_1(S) + \|S-\hat S\|)\|S-\hat S\|}{\sigma_1^2(S)-\sigma_2^2(S)},
\]
where $\sigma_1(S)$ and $\sigma_2(S)$ are the two leading singular values of $S$. Since 
$S$ is rank one, $\sigma_1(S) = \|S\| = \|S\|_F$ and $\sigma_2(S)=0$. In addition, we assumed that $\|S-\hat S\|_F \leq 0.5 \|S\|_F$ and hence
\begin{equation}\label{eq:sin_u_bound}
\sin(u, \hat u) \leq \frac{5\|S\|_F \|S - \hat S\|_F}{\|S\|_F^2} = \frac{5\|S - \hat S\|_F}{\|S\|_F}.
\end{equation}
Combining Eqs. \eqref{eq:bound_outer_product_diff_b}, \eqref{eq:sin_u_bound} concludes the proof.
\end{proof}

\begin{proof}[Proof of Lemma \ref{lem:d_err_bound} ]
Let $d(e)$ denote the score of the edge $e$ computed by the exact similarity matrix $S$ as defined in \eqref{eq:edge_score}. We denote by $\hat d(e)$ the score computed by the noisy estimate of the similarity $\hat S$.
The difference between $d(e)$ and $\hat d(e)$ is equal to
\begin{equation}\label{eq:d_diff}
|d(e) - \hat d(e)| = \left|\frac{\|S(A,B)-\alpha^\ast u_Au_B^T\|_F}{\|S(A,B)\|_F}-\frac{\|\hat S(A,B)-\beta^\ast \hat u_A\hat u_B^T\|_F}{\|\hat S(A,B)\|_F} \right|,    
\end{equation}
where,
\[
\alpha^\ast = \argmin_\alpha \|S(A,B)-\alpha u_Au_B^T\|_F \qquad 
\beta^\ast = \argmin_\beta \|\hat S(A,B)-\beta \hat  u_A\hat u_B^T\|_F.
\]
We apply Lemma \ref{lem:aux_ratio_bound} with 
\begin{align*}
X &= \|S(A,B)-\alpha^\ast u_Au_B^T\|_F, \quad y = \|S(A,B)\|_F, \quad 
\hat X = \|\hat S(A,B)-\beta^\ast \hat u_A\hat u_B^T\|_F, \quad \hat y = \|\hat S(A,B)\|_F,
\end{align*} 
where we note that here $X$ and $\hat X$ are scalars.    
Lemma \ref{lem:aux_ratio_bound} requires that $0<|\hat X| \leq \hat y$, which holds trivially.
Applying Lemma \ref{lem:aux_ratio_bound} to  \eqref{eq:d_diff}  yields,
\begin{align}
|d(e) - \hat d(e)| \leq \frac{1}{\|S(A,B)\|_F} 
\left( \Big| \|S(A,B)-\alpha^\ast u_Au_B^T\|_F \right.&-
\|\hat S(A,B)-\beta^\ast \hat u_A\hat u_B^T\|_F
\Big|
\notag \\
& 
\left.+ 
\big|\|S(A,B)\|_F - \|\hat S(A,B)\|_F \big|
\right).
\end{align}
Next, setting $X = S(A,B), \hat X = \hat S(A,B), Y = \alpha^\ast u_Au_B^T$ and $\hat Y = \beta^\ast \hat u_A \hat u_B^T$,  by Lemma \ref{lem:aux_matrix_diff_bound}, 
\begin{align}
|d(e) - \hat d(e)| 
&\leq 
\frac{1}{\|S(A,B)\|_F} \left(\| S(A,B)-\hat S(A,B)\|_F + \|\alpha^\ast u_Au_B^T -\beta^\ast \hat u_A\hat u_B^T\|_F\right. \notag \\
& \qquad \qquad \qquad \quad\left.+ \big|\|S(A,B)\|_F - \|\hat S(A,B)\|_F \big|\right) 
\notag \\
&\leq \frac{1}{\D} \left(2\| S(A,B)-\hat S(A,B)\|_F + \|\alpha^\ast u_Au_B^T -\beta^\ast \hat u_A\hat u_B^T\|_F\right).
\label{eq:d_est_diff}
\end{align}
where the second inequality is due to the reverse triangle inequality and the definition of $\D$.

We focus on the term $\|\alpha^\ast u_Au_B^T -\beta^\ast \hat u_A \hat u_B^T\|_F$ in Eq. \eqref{eq:d_est_diff}.
The values of $\alpha^\ast,\beta^\ast$ are obtained via least square between the elements of $S(A,B),\hat S(A,B)$ and $u_A u_B^T,\hat u_A \hat u_B^T$, respectively. For $\alpha^\ast$, the least squares solution is
\begin{equation}\label{eq:alpha_ast}
\alpha^\ast = \frac{1}{\|S(A,B)\|_F^2}\sum_{i,j} S(A,B)_{ij} (u_Au_B^T)_{ij} = \frac{1}{\|S(A,B)\|_F^2} u_A^T S(A,B)u_B,    
\end{equation}
where a similar expression holds for $\beta^\ast$.
Multiplying $\alpha^\ast$ and $\beta^\ast$ by $u_Au_B^T$ and $\hat u_A \hat u_B^T$ gives,
\begin{equation}\label{eq:alpha_uAuB-beta_uAuB}
    \alpha^\ast u_A u_B - \beta^\ast \hat u_A \hat u_B = \frac{1}{\|S(A,B)\|_F^2} u_A u_A^T S(A,B)u_B u_B^T - \frac{1}{\|\hat S(A,B)\|_F^2} \hat u_A \hat u_A^T \hat S(A,B)\hat u_B \hat u_B^T.    
\end{equation}
Next, we apply Lemma \ref{lem:aux_ratio_bound} with $X= u_A u_A^T S(A,B)u_B u_B^T$, $y = \| S(A,B)\|_F^2$, $\hat X= \hat u_A \hat u_A^T \hat S(A,B)\hat u_B \hat u_B^T$ and  $\hat y = \|\hat S(A,B)\|_F^2$. The condition for Lemma \ref{lem:aux_ratio_bound} is that $\|\hat X\|_F \leq \hat y$, which holds since
\[
\|\hat X\|_F = \|\hat u_A \hat u_A^T \hat S(A,B)\hat u_B \hat u_B^T\|_F \leq \|\hat u_A \hat u_A^T\|_F \|\hat S(A,B)\|_F \| \hat u_B \hat u_B^T\|_F \leq \|\hat S(A,B)\|_F = \hat y.
\]
Applying Lemma \ref{lem:aux_ratio_bound} to \eqref{eq:alpha_uAuB-beta_uAuB} gives
\begin{align} \label{eq:alpha_ua_ub-beta_ua_ub}
    \|\alpha^\ast u_A u_B - \beta^\ast \hat u_A \hat u_B\|_F \leq 
\frac{1}{\|S(A,B)\|^2_F}&
\left( \|u_A u_A^T S(A,B)u_B u_B^T  - \hat u_A \hat u_A^T \hat S(A,B)\hat u_B \hat u_B^T\|_F. \right.
\notag \\
&~~\left.+ \big|\|S(A,B)\|_F^2-\|\hat S(A,B)\|_F^2 \big|\right).
\end{align}
Denote 
\begin{align*}
\eps(A,B) = &\hat S(A,B) - S(A,B) && \eps(C_1,C_2) = \hat S(C_1,C_2) - S(C_1,C_2) \\
\eps_A = &\hat u_A \hat u_A^T - u_Au_A^T 
&&
\eps_B = \hat u_B \hat u_B^T - u_Bu_B^T.
\end{align*}
Equipped with the above notations, we bound the first term in the numerator of Eq. \eqref{eq:alpha_ua_ub-beta_ua_ub} using Lemma \ref{lem:aux_product_bound} where $X = u_A u_A^T$, $Y = S(A,B)$, and $Z = u_B u_B^T$,
\begin{align*}
\|u_A u_A^T &S(A,B)u_B u_B^T  - \hat u_A \hat u_A^T \hat S(A,B)\hat u_B \hat u_B^T\|_F 
\notag \\
&\leq \|u_A u_A^T\|_F \|S(A,B)\|_F\|\eps_B\|_F
+
\|\hat u_B \hat u_B^T\|_F\|S(A,B)\|_F\|\eps_A\|_F + \|\hat u_B\hat u_B^T\|_F\|\hat u_A \hat u_A^T\|_F\|\eps(A,B)\|_F
\end{align*}
Since $\|u_A u_A^T\|_F, \|\hat u_B \hat u_B^T\|_F \leq 1$ we get,
\begin{equation}\label{eq:numerator_expression}
\|u_A u_A^T S(A,B)u_B u_B^T  - \hat u_A \hat u_A^T \hat S(A,B)\hat u_B \hat u_B^T\|_F 
\leq 
\|S(A,B)\|_F(\|\eps_A\|_F+\|\eps_B\|_F)+\|\eps(A,B)\|_F.
\end{equation}

The matrices $\eps_A,\eps_B$ are submatrices of $\hat u \hat u^T-uu^T$ and hence $\|\eps_A\|_F,\|\eps_B\|_F \leq \|\hat u \hat u^T-uu^T\|_F$.
Applying Lemma \ref{lem:bound_eps_A} gives 
\begin{equation}\label{eq:lemma_c_2}
\|\eps_A\|_F + \|\eps_B\|_F \leq 2\|uu^T-\hat u \hat u^T\|_F \leq \frac{10\sqrt{2} \|\eps(C_1,C_2)\|_F}{\|S(C_1,C_2)\|_F} \leq 
\frac{10\sqrt{2} \|\eps(C_1,C_2)\|_F}{\D}.
\end{equation}
Combining Eqs. \eqref{eq:d_est_diff},   \eqref{eq:alpha_ua_ub-beta_ua_ub},\eqref{eq:numerator_expression} and \eqref{eq:lemma_c_2} yields 
\begin{multline}
    \label{eq:bound_diff_score}  
    |d(e) - \hat d(e)| \leq 
    \frac{1}{\D} \Bigg(2\|\eps(A,B)\|_F 
    + \frac{1}{\|S(A,B)\|_F^2} \\
    \left.\left(\big|\|S(A,B)\|_F^2-\|\hat S(A,B)\|_F^2 \big| +  \|\eps(A,B)\|_F + \frac{10\sqrt{2} \|S(A,B)\|_F \|\eps(C_1,C_2)\|_F}{\D}
\right)\right).
\end{multline} 
We have that
\begin{align}\label{eq:square_diff}
\big| \|S(A,B)\|_F^2-\|\hat S(A,B)\|_F^2 \big| &= \big| \|S(A,B)\|_F-\|\hat S(A,B)\|_F\big|(\|S(A,B)\|_F+\|\hat S(A,B)\|_F) 
\notag \\
&\leq 2.5 \|\eps(A,B)\|_F \|S(A,B)\|_F,
\end{align}
where the inequality is due to the reverse triangle inequality and our assumption $\|\eps(A,B)\|_F\leq 0.5\|S(A,B)\|_F$ which implies $\|\hat S(A,B)\|_F \leq 1.5 \|S(A,B)\|_F$.
Combining \eqref{eq:bound_diff_score} and \eqref{eq:square_diff}, we get
\begin{equation*}
\begin{aligned}
 |d(e)-\hat d(e)| &\leq 
 \frac{1}{\D} \left( 2\|\eps(A,B)\|_F 
 +\frac{1}{\|S(A,B)\|_F^2} \right. \times 
 \\
 & \qquad 
\left.\left(  \|\eps(A,B)\|_F\left(2.5 \|S(A,B)\|_F +1\right)  + \frac{10\sqrt{2}  \|\eps(C_1,C_2)\|_F\|S(A,B)\|_F}{\D} \right)\right) 
\\
&\leq  \|\eps(A,B)\|_F \left(\frac{2}{\D} + \frac{2.5}{\D^2} + \frac{1}{\D^3} \right) + \|\eps(C_1,C_2)\|_F\frac{10\sqrt{2}}{\D^3}
\\
& \leq  \|S-\hat S\|_F \left(\frac{2}{\D}+\frac{2.5}{\D^2} +  \frac{1+10\sqrt{2}}{\D^3}  \right) ,
\end{aligned}
\end{equation*}
which concludes the proof.
\end{proof}

\section{Additional Simulation Results}\label{appendix:experimental_results}
\subsection{Caterpillar tree}\label{subsec:experiment_caterpiller}

We generated a caterpillar tree with $m=512$ terminal nodes, where the non-terminal nodes form a path graph. 
The similarity between each pair of adjacent nodes was set to $\delta = 0.81$. As in Section \ref{sec:experiments}, we compare NJ, SNJ and RAxML, with STDR where the aforementioned methods are used as subroutines. The STDR threshold is set to $\tau = 64$ for all three STDR variants.
Figure \ref{fig:caterpillar_performance} shows the normalized RF distance (left) and runtime (right) of the different methods as functions of the sequence length $n$. Here, all three methods are significantly improved when combined with STDR in both runtime and accuracy. 

\begin{figure}[ht] 
    \centering
    \includegraphics[width = \textwidth]{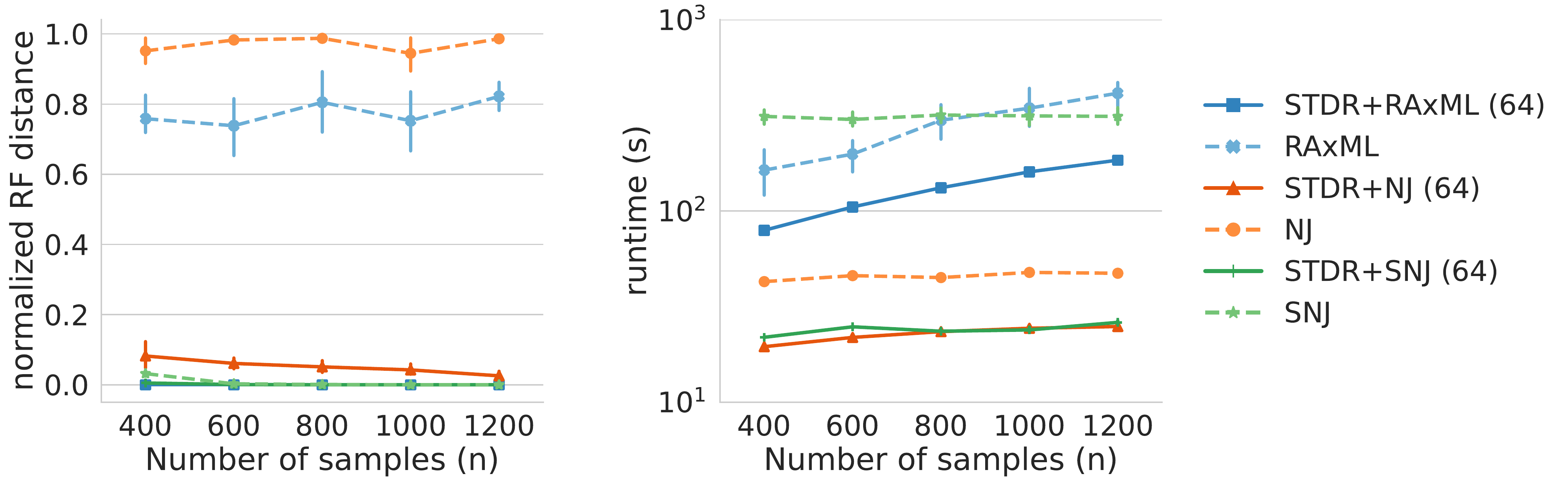}
    \caption{A caterpillar tree with $m=512$  terminal nodes. The mean and standard deviation of the runtime (right) and RF distance between the reconstructed tree and the input tree (left) are shown for each method over 5 independent runs. }
    \label{fig:caterpillar_performance}
\end{figure}

\subsection{Comparison to TreeMerge}
We generated random trees with $2000$ terminal nodes according to the coalescent model. The trees were recursively partitioned by STDR with a threshold of $\tau=128$. The structure of the different partitions was recovered by RAxML. We compared STDR's merging criteria with TreeMerge \cite{molloy2019treemerge} for various sequence lengths. The results are shown in Figure \ref{fig:coalesent_treemerge_performance}. The merging process of STDR achieved better accuracy than TreeMerge, with a significantly reduced runtime.

\begin{figure}[ht]
    \centering
    \includegraphics[width =\textwidth]{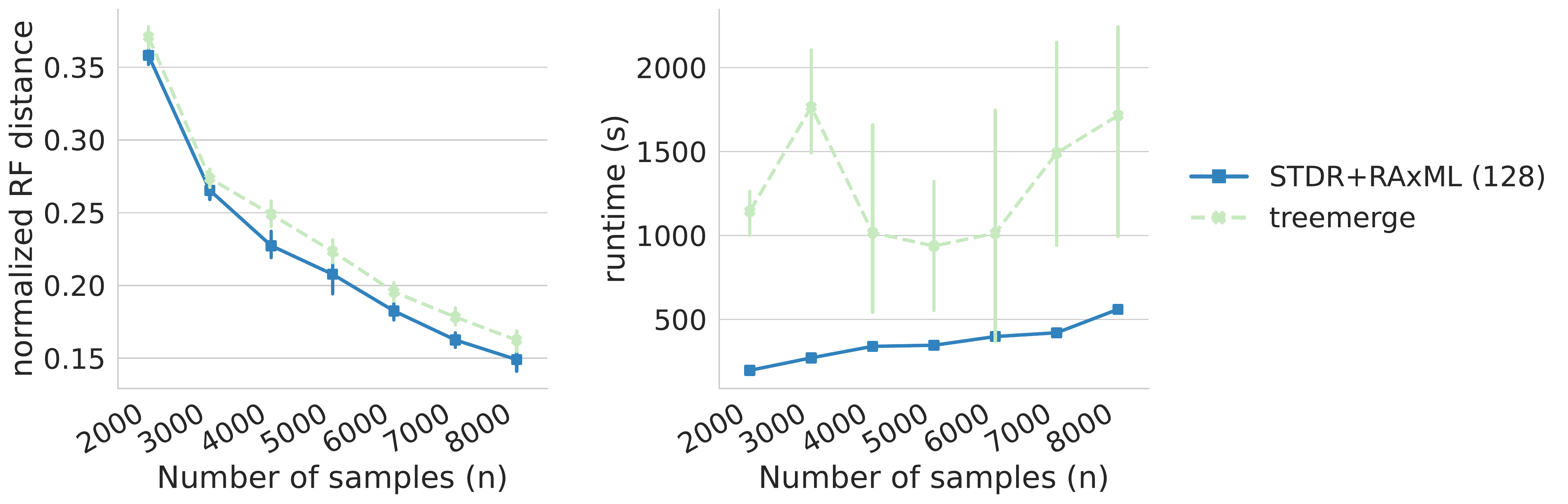}
    \caption{A coalesent tree with $m=2000$ terminal nodes. The mean and standard deviation of the normalized RF distance (left) between the reconstructed tree and the input tree and of the runtime (right) are shown for each method over 5 independent runs. }
    \label{fig:coalesent_treemerge_performance}
\end{figure}




\end{appendices}

\bibliography{references.bib}{}
\bibliographystyle{plain}

\end{document}